\documentclass[runningheads]{llncs}

\usepackage[T1]{fontenc}
\usepackage{amsmath, amssymb, amsfonts}
\usepackage{mathtools}
\usepackage{bm}

\usepackage{array}
\usepackage{multirow}
\usepackage{booktabs}
\usepackage{caption}
\usepackage{subcaption}

\usepackage[table]{xcolor}
\definecolor{lightblue}{RGB}{230, 240, 255}
\usepackage{tikz}

\usepackage{tcolorbox}

\usepackage{algorithm}
\usepackage{algpseudocode}      

\usepackage{textcomp}
\usepackage{stfloats}
\usepackage{url}
\usepackage{verbatim}
\usepackage{ltxtable}
\usepackage{adjustbox}
\usepackage{mdframed}
\usepackage{marvosym}
\usepackage{cite}

\usepackage{geometry}

\usepackage{hyperref}

\def\R{\mathbb{R}}
\def\E{\mathbb{E}}

\newtheorem{assumption}{Assumption}

\begin{document}
\title{AdaBFL: Multi-Layer Defensive Adaptive Aggregation for Bzantine-Robust Federated Learning}
%
%
\author{Zehui Tang\inst{1} \and
Yuchen Liu\inst{1} \and
Feihu Huang\inst{1,2}\textsuperscript{\Letter} }
\authorrunning{Tang et al.}
%
\institute{1. College of Computer Science and Technology, Nanjing University of Aeronautics and Astronautics, Nanjing, China; \\
2. MIIT Key Laboratory of Pattern Analysis and Machine Intelligence, Nanjing, China.	\\
\email{huangfeihu2018@gmail.com.}}
\maketitle              
\begin{abstract}
Federated learning (FL) is a popular distributed learning paradigm in machine learning, which enables multiple clients to collaboratively train models under the guidance of a server without exposing private client data. However, FL's decentralized nature makes it vulnerable to poisoning attacks, where malicious clients can submit corrupted models to manipulate the system. To counter such attacks, although various Byzantine-robust methods have been proposed, these methods struggle to provide balanced defense against multiple types of attacks or rely on  possessing the dataset in the server. To deal with these drawbacks, thus, we propose an effective multi-layer defensive adaptive aggregation for Bzantine-robust federated learning (AdaBFL) based on a novel three-layer  defensive mechanism, which can adaptively adjust the weights of defense algorithms to counter complex attacks. Moreover, we provide convergence properties of our AdaBFL method under the non-convex setting on non-iid data. Comprehensive experiments across multiple datasets validate the superiority of our AdaBFL over the comparable algorithms.

\keywords{ Federated learning \and Adaptive aggregation  \and Poisoning attack \and Byzantine-Robust.}
\end{abstract}
%
%
\section{Introduction}
\label{Introduction}
Federated Learning (FL)~\cite{mcmahan2017communication} effectively addresses the challenge of \textit{Data Silos}, which enables multiple clients to collaboratively train machine learning models under the coordination of a central server without sharing clients' private data. 
Recently, FL has been implemented across many  applications, including credit risk assessment \cite{webank2020utilization}, natural language processing \cite{liu2021federated,cai2023efficient}, speech recognition \cite{farahani2023toward}, and healthcare services \cite{rauniyar2023federated}, model compression \cite{wu2024auto} \cite{liao2023adaptive}, pretrained models \cite{he2025afl}, leakage attack \cite{zhao2023resource}\cite{zhang2023federated}\cite{xie2025dflmoe}\cite{zhou2025secure}, semantic segmentation \cite{miao2023fedseg} and multi-agent learning \cite{mohammadabadi2024communication}.
FedAvg~\cite{mcmahan2017communication} is the first formalized federated learning algorithm, which first performs multiple rounds of stochastic gradient descent (SGD) locally on each device, and then transmits its local model to the central server, where they are averaged. FedProx~\cite{li2020federated} is a variant of FedAvg \cite{mcmahan2017communication} that addresses heterogeneous federated environments by constraining local parameters with global parameters in the optimization objective. 
Scaffold~\cite{karimireddy2020scaffold} estimates the update directions of the server model and each client respectively, then defines difference between the update directions as the client’s drift estimate and uses it to correct the local update.

Although the above FL methods guarantees data privacy, it is not resistant to external malicious attacks. As shown in Fig. \ref{MeanAttack}, malicious attacks can alter the performance of the global model in FL.
The standard FL framework such as  FedAvg \cite{mcmahan2017communication}, a single malicious client is sufficient to manipulate the final aggregated global model (see Fig. \ref{MeanAttack}(b-c)). As attacks become more sophisticated \ref{MeanAttack}(d-f), they require more sophisticated defense strategies.
The decentralized nature of FL is often exploited by malicious adversaries to launch poisoning attacks that compromise integrity of the global model \cite{fang2025we}. Specifically,  malicious actors can manipulate locally trained data or fabricate toxic updates to the model and transmit them to the central server, thereby degrading accuracy of the global model. Based on the attacker's objectives, poisoning attacks are categorized as: \textit{1) Targeted attack} \cite{bagdasaryan2020backdoor,baruch2019little,cao2020fltrust,fang2020local} aims to induce model prediction errors by injecting backdoors (targeted attacks seek to misclassify only specific inputs while preserving predictions for others). \textit{2) Non-targeted attack} \cite{blanchard2017machine,fang2020local,shejwalkar2021manipulating} aims to degrade the overall performance of the global model. 

\begin{figure*}[h]
	\centering
	\includegraphics[width=0.78\textwidth, height=0.4\textwidth]{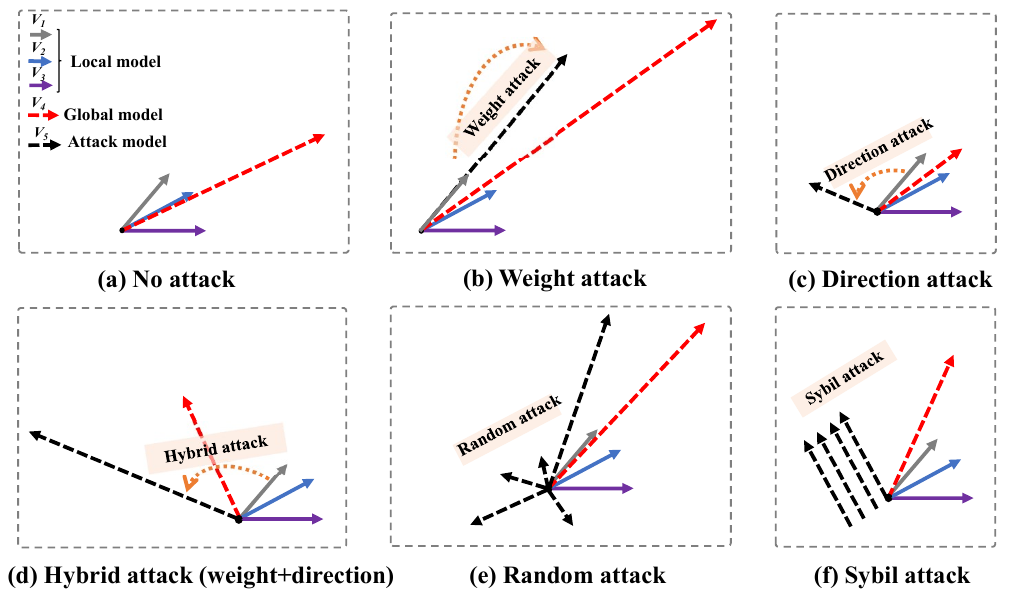}
	\caption{Common attacks in FL systems. A single malicious client can compromise the global model simply by altering the length and angle of parameters. When multiple malicious clients employ the same attack method (Sybil attack), they can effectively circumvent statistical defense methods. $V_1,V_2,V_3$ denote the local model, $V_4$ is the global model, $V_5$  is the attack models. }
	\label{MeanAttack}
\end{figure*}

\begin{figure*}[h]
	\centering
	\includegraphics[width=0.8\textwidth, height=0.43\textwidth]{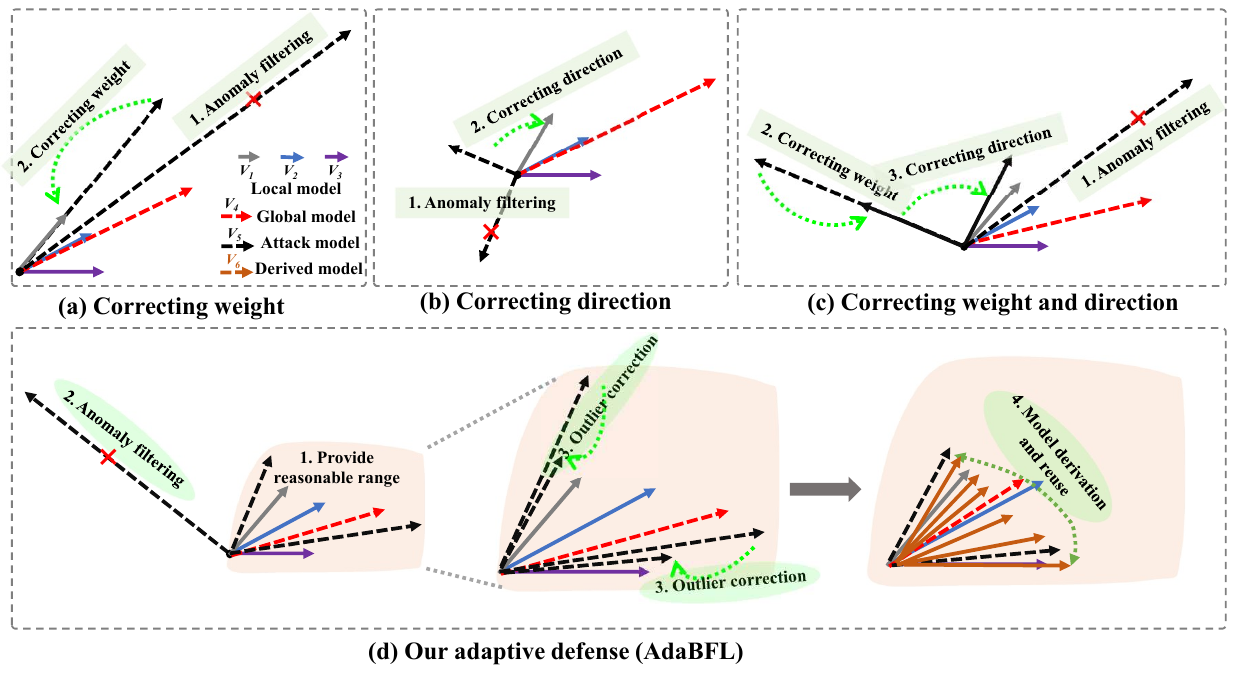}
	\caption{Summary of defense strategies. Traditional byzantine defense methods (a), (b), and (c) filter out significant outliers and then further refine the remaining data through aggregation based on commonality. Our AdaBFL (d) not only incorporates the capabilities of traditional defense methods but also leverages benign models to derive additional high-quality models. Finally, by dynamically adjusting the aggregation weights of different client models, it further enhances the FL system’s defense against Byzantine attacks. $V_6$  is the benign model derived from existing models. }
	\label{Defens}
\end{figure*}

To defend against these poisoning attacks, \textit{Byzantine-robust foundational aggregation (BRFA)} \cite{fang2022aflguard,munoz2019byzantine,cao2020fltrust,fang2020local} is proposed to enhance the resilience of FL. Unlike aggregation strategies based on the mean \cite{mcmahan2017communication}, BRFA aims to detect malicious clients and prune or discard the local update models they provide. Depending on the method used to identify malicious clients, these can be categorized as: \textit{1) Similarity-based defenses} \cite{fang2022aflguard,munoz2019byzantine,wang2022flare,xie2020zeno++,yin2018byzantine,fang2024byzantine}: These leverage the inherent consistency among benign updates, detecting and filtering malicious local updates through statistical similarity analysis. \textit{2) Performance-based defenses} \cite{cao2020fltrust,fang2020local,park2021sageflow,xie2019zeno}: Based on the assumption that benign updates enhance global model performance, these methods identify malicious updates by evaluating the aggregated loss or accuracy. \textit{3) Entropy-based defenses} \cite{park2021sageflow}: Leveraging the characteristic of benign models having high prediction confidence and low entropy values, isolating high-entropy local models to exclude malicious updates.	
However, we recognize that in certain scenarios, the non-identical and non-independent (non-iid) distribution of training data among clients makes it challenging for a single detection mechanism to identify malicious local update models effectively.

\textbf{Limitations of Existing Robust Aggregation Rules:} First, current Byzantine-resistant aggregation rules still exhibit security flaws when confronted with sophisticated poisoning attacks. For instance, commonly used Krum/Trim aggregation rules are vulnerable to Krum attacks/Trim attacks \cite {fang2020local}. Second, many robust aggregation rules \cite{cao2020fltrust,pan2020justinian, park2021sageflow,wang2022flare,xie2019zeno} rely on the strong assumption that servers possess verification data. Undoubtedly, such assumptions contradict the original intent of FL, as central servers struggle to determine the distribution of local client training data accurately. Furthermore, central servers possessing verification data infringe upon client privacy.
Recently, traditional aggregation defense methods can defend against single attacks but struggle to effectively counter multiple attacks  (see Fig. \ref{Defens}(a-c)). 
These limitations prompt us to pose a natural question:  
\begin{center}
	\begin{tcolorbox}
\textbf{\textit{Can we design novel multi-layer Byzantine-robust aggregation rules from different perspectives?}}
\end{tcolorbox}
\end{center}

In the paper, we provide an affirmative answer to the above question and propose an effective Multi-Layer Defensive \underline{Ad}aptive \underline{A}ggregation for \underline{B}zantine-robust \underline{F}ederated \underline{L}earning (AdaBFL). Our AdaBFL algorithm can prevent outliers from degrading the performance of the global model (see Fig. \ref{Defens}(d)). The \textbf{main contributions} of this paper are summarized as follows:
\begin{itemize}
	\item[(i)] We propose an effective Byzantine-robust federated learning (AdaBFL) method by using an adaptive aggregation rule. In particular, we also develop a new adaptive aggregation rule based on a three-layer defense mechanism. 
	\item[(ii)] We theoretically prove the convergence of AdaBFL under the nonconvex setting on non-iid data, which demonstrates its robust resilience against malicious attacks under the assumption of universal acceptability.
	\item[(iii)] We conduct extensive simulation experiments across various poisoning attack scenarios and different aggregation rules using diverse datasets, validating effectiveness of our AdaBFL framework.
\end{itemize}

\section{Related Work}
 In this section, we will review some methods on poisoning attacks to FL and 
 Byzantine-robust FL, respectively.

\subsection{Poisoning Attacks to FL}
FL is susceptible to data poisoning attacks\cite{biggio2012poisoning,munoz2017towards,tolpegin2020data} and model poisoning attacks\cite{blanchard2017machine,fang2020local,shejwalkar2021manipulating} due to its decentralized nature. In data poisoning attacks, malicious clients corrupt local training data by injecting poisoned samples into the dataset. For instance, in label flipping attacks \cite{tolpegin2020data}, attackers tamper with the correspondence between local training labels $y^i$ and samples $x^i$, thereby degrading the global model's performance without altering the features. In model poisoning attacks, attackers directly tamper with the trained local model parameters to construct anomalous or maliciously backdoored parameters. The ultimate goal of Poisoning Attacks is to degrade the performance of the global model. Data poisoning involves \textit{poisoning at the data source}, while model poisoning involves \textit{poisoning during model training}. The latter is generally considered more threatening because it is more direct and efficient. However, it may also be easier to detect.

\subsection{Byzantine-robust FL}
In standard FL, the central server updates the global model by computing the average of all client models~\cite{mcmahan2017communication}. However, recent research~\cite{fang2025we} demonstrates that averaging-based aggregation rules are vulnerable to Byzantine attacks. In such attacks, a single client can arbitrarily manipulate the final aggregated global model. The introduction of Byzantine-robust aggregation rules (Krum \cite{blanchard2017machine}, Trimmed-mean \cite{yin2018byzantine}, Median \cite{yin2018byzantine}, BALANCE \cite{fang2024byzantine}, FoundationFL~\cite{fang2025we}) effectively mitigates poisoning attacks against FL. For instance, in the Krum \cite{blanchard2017machine} aggregation rule, after receiving $n$ local updates, the server selects the update with the smallest sum of distances to its $n-\chi  -2$ nearest neighbors as the global model, where $n$ denotes the total number of clients and $\chi  $ represents the number of malicious clients. In the Trimmed-mean~\cite{yin2018byzantine} aggregation rule, the server operates independently on each model parameter dimension. It first removes the $\beta$ most extreme values (both maximum and minimum) in that dimension, then calculates the average of the remaining values, where $\beta$ is a predefined threshold for trimming outliers. In the Median\cite{yin2018byzantine} aggregation rule, the server operates independently on each model parameter dimension, directly selecting the median value of all clients' values on that dimension as the final value for that dimension in the global model. BALANCE~\cite{fang2024byzantine} leverages consistency among well-functioning models, aggregating neighbor node models with its own updates.  FoundationFL~\cite{fang2025we} generates synthetic updates when receiving local model updates from clients, then aggregates the new global model using Trimmed-mean or Median. \cite{cao2020fltrust,pan2020justinian,park2021sageflow,wang2022flare,xie2019zeno} share a common assumption: the server possesses a small, clean validation set drawn from the overall data distribution. Based on this validation set, the server can compute a representative baseline model update (or performance benchmark). It then identifies malicious updates by comparing the similarity or performance deviation between received local updates and this baseline.

\begin{figure}[htbp]
	\centering
	\includegraphics[width=0.83\textwidth,height=0.4\textwidth]{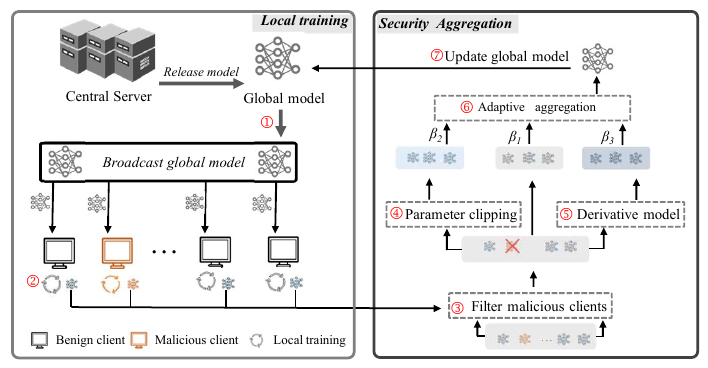}
	\caption {System model of AdaBFL.}
	\label{System_model}
\end{figure}

\section{Preliminaries}

Federated learning (FL)~\cite{mcmahan2017communication} is a popular distributed  learning paradigm in machine learning, and generally solve the following distributed optimization problem:
\begin{align} 	\label{P1}
\min_{\theta \in \mathbb{R}^d} \ F(\theta ) =\frac{1}{N} \sum_{i=1}^N f^i(\theta),
\end{align}
where for any $i\in [N]$, $f^i(\theta) = \E_{\xi^i}[f^i(\theta;\xi^i)]$ denotes the loss function at $i$-th client, and $\xi^i$ denotes a random variable followed fixed but unknown data distribution $\mathcal{D}^i$, and $\mathcal{D}^i$ may be not the same as $\mathcal{D}^j$ for all $i\neq j$. Here $\theta \in \mathbb{R}^d $ denotes the model parameter vector. 

Recently, many FL algorithms have been developed to solve the above problem (\ref{P1}), whose algorithmic framework generally is give as follows: 
\begin{itemize}
	\item \textbf{Step 1: Broadcast the global model}. The central server broadcasts the current global model $\theta_{t}$ to all clients.
	
	\item \textbf{Step 2: Local model training and updating}. Each client $i\in[n]$ trains a local model $\theta^i$ using the received global model $\theta_t$ and its private dataset $\mathcal{D}^i$, i.e.,  $\theta^i_{t+1} = \theta_t - \eta g^i_t$, where $g^{i}_{t}=\nabla f^i(\theta_t;\zeta^i)$ denotes the stochastic gradient at $i$-th client.  
	
	\item \textbf{Step 3: Global model updating.} After collecting updated models from clients, the server merges these updates using an aggregation rule (denoted as Agg) to obtain the global model. The globally updated model is commonly represented as $Agg\{\theta^i_t: i\in [N] \}$.
\end{itemize}

It is imperative to note that different FL training methods typically employ distinct aggregation protocols \cite{mcmahan2017communication, blanchard2017machine, yin2018byzantine}. When lower security in FL, the mean aggregation~\cite{mcmahan2017communication} is commonly used: $(Agg\{g^i_t, i\in [N] \}=\frac{1}{N}\sum_{i=1}^{N}g^i_t)$. 
When high security in FL,
 we assume attackers possess capabilities for data poisoning and model poisoning as in~\cite{shejwalkar2021manipulating,fang2024byzantine,fang2025we}. Specifically, attackers manipulate a subset of malicious clients (which may be either attacker-registered fake clients or benign clients deceived by the attacker) to tamper with the training dataset or inject backdoors into the trained model post-training. It is noteworthy that in practical attack scenarios, attackers possess knowledge of the aggregation rules employed by the central server and the global model.  We also assume the central server cannot discern the attackers' specific tactics or the distribution of client data.

\section{Our AdaBFL Method}
\label{Proposed_model}
In the section, we propose an effective AdaBFL method for Byzantine-robust FL based on a novel multi-layer defensive mechanism. Fig \ref{System_model} shows the system model of AdaBFL. Specifically, our multi-layer defensive mechanism includes three defensive strategies:
1) The server first filters out malicious clients, retaining only the benign clients' models; 2) The server then expands the benign clients' models by dimension and prunes anomalous parameters to obtain the trimmed models;
3) From among the benign clients, the optimal model is selected, subsequently, and multiple optimal models are derived based on this selected optimal model.
Finally, the server fuses the models obtained in steps 1), 2), and 3) using adaptive weights to derive the global model. AdaBFL defends against poisoning attacks from multiple angles.

\begin{algorithm}[H]
	\caption{AdaBFL Algorithm}
	\label{alg:dfl}
	\begin{algorithmic}[1]
		\footnotesize
		\State \textbf{Input} Total number of clients \( N \); number of clients participating in training per epoch \( n \); filtered benign clients \( \Lambda \); local training data \( D^i, i \in [N] \); number of global training rounds \( T \); learning rate \( \eta \); weights \( \beta_1,\beta_2,\beta_3 \); parameters \( \gamma, \kappa, \lambda_{t} \).
		\State \textbf{Output} Global model \( \theta_T \).
		
		\State Initialize global model \( \theta_0 \).
		\State \textit{// \textcolor{blue}{Global model synchronization} }
		\State Server broadcasts \( \theta_0 \) to all clients.
		\State \textit{// \textcolor{blue}{Local model updating} }
		\For{\( t =  1, \dots, T-1 \)}
		
		\For{each client \( i \in [n] \) in parallel}
		\State Drawn $b$ i.i.d. samples $\mathcal{B}=\{\xi_{t-1,j}^i\}_{j=1}^b$ from data distribution $\mathcal{D}^i$;
		\State $g^{i}_{t}=\frac{1}{b}\sum_{j=1}^b\nabla f^i(\theta_{t-1};\xi^i_{t-1,j})$; 
		\State \( \theta_{t}^{i} = \theta_{t-1} - \eta g_t^{i} \).
		\EndFor
		\State Clients send \(\theta_t^{i}\) to server.
		
		\State \textit{// \textcolor{blue}{ Adaptive aggregation}}
		\For{each client \( i \in [n] \)}
		\State \textit{// \textcolor{blue}{(1) Filter malicious clients}}
		\State {The server executes Eq.(\ref{EqP1})}
		
		\State \textit{// \textcolor{blue}{(2) Parameter clipping}}
		\State {The server executes Eq.(\ref{EqP2}) and Eq.(\ref{EqP2-2})}
		
		\State \textit{// \textcolor{blue}{(3) Derivative model
		}}
		\State {The server executes Eq.(\ref{EqP3}-\ref{EqP6}))}
		\EndFor
		\State {The server executes Weight Update Algorithm (\ref{WUA})}
	\EndFor
	\end{algorithmic}
\end{algorithm}

Algorithm~\ref{alg:dfl} shows the algorithmic framework of our AdaBFL method. 
In Algorithm~\ref{alg:dfl}, we implement \textbf{\textit{Filter malicious, Parameter clipping}}, and \textbf{\textit{ Derivative model}} to identify malicious model updates from different angles or amplify the weight of benign model updates. (1) The core idea of Filter malicious is that if $i$-th client is benign, then the local update model provided by $i$-th client should be close in both direction and magnitude to the local update models of the remaining benign $j$-th client. In this paper, based on this idea, we can effectively select a set of benign clients. (2) The core idea of Parameter clipping is to perform fine-grained clipping on the locally updated models provided by the selected benign clients to counteract the impact of outliers in abnormal dimensions on the global model. (3) The core idea of Derivative model is to derive synthetic data from the original data, then fuse the synthetic data with the original data to obtain new fused data. Based on this approach, we can further amplify the influence of benign data on the global model. Simultaneously, the introduction of synthetic data enhances the model's generalization capability. Finally, by dynamically monitoring the status of Filter malicious, Parameter clipping, and Derivative model, we can adaptively adjust their respective weights. Notably, we further propose a threshold-free adaptive adjustment mechanism and a momentum-based extension method.
The detailed description of the AdaBFL Algorithm \ref{alg:dfl} is as follows:

The server broadcasts the global model $\theta_0$ to the clients, and the clients then use local data $D^i$ to optimize their local models $\theta_{t}^i$ using SGD and send the updated local models back to the server. The services perform Filter malicious, Parameter clipping, and Derivative model. They filter out malicious models and aggregate the local models from benign clients to generate the global model $\theta_t$. The core design of the Filter Malicious, Parameter Clipping, and Derivative Model are described as follows: 

\textbf{\textit{1). Filter malicious clients}} \textit{(Algorithm 1,  Lines 15–16)}: 
As described in \cite{fang2020local,shejwalkar2021manipulating,fang2024byzantine}, attackers can launch poisoning attacks against FL systems by fabricating malicious clients or deceiving benign clients to manipulate the size or direction of local models. Our perspective is that if a client provides a model that differs significantly from those provided by other clients, it should be considered potentially malicious and ignored by the central server. Specifically, during the $t$-th training round, when the central server receives the local model $\theta_t^i$ from client $i \in N$, it computes the similarity between each $\theta_t^i$ and the mean of the remaining models $\theta_t^j \ (j \in N, j \ne i)$, thereby verifying whether the $i$-th client is malicious or benign. If $\theta_t^i$ resembles $\frac{1}{n-1}\sum_{j=1}^{n-1}\theta_t^{j}$ in both direction and magnitude, the central server classifies $\theta_t^i$ as a benign model. As the round number $t$ increases, the models converge, and $\theta_t^i$ approaches $\frac{1}{n-1}\sum_{j=1}^{n-1}\theta_t^{j}$ increasingly closely. Based on this insight, if condition Eq.(\ref{EqP1}) is satisfied, the central server classifies client $i$ as a benign client $\Lambda$,
\begin{equation}
	\label{EqP1}
	\begin{aligned}
		\left\|\theta_{t}^{i}-\frac{1}{n-1} \sum_{j=1}^{n-1} \theta_{t}^{j}\right\| \leq \frac{\gamma e^{-\kappa \lambda(t)}}{2} \cdot\left\|\frac{1}{n-1} \sum_{j=1}^{n-1} \theta_{t}^{j}+\theta_{t}^{i}\right\|.
	\end{aligned}
\end{equation}
Here the parameter $\gamma > 0$ represents the upper threshold for determining a client as benign. The value of the parameter $\kappa > 0$ determines the decay rate of the exponential function (a larger $\kappa$ results in faster decay, while a smaller $\kappa$ leads to slower decay). $\lambda(t)$ is a function dependent on the round number $t$, which must satisfy the conditions of monotonicity and non-negativity. Typically, $\lambda(t)$ can be set as $\lambda(t)=\frac{1}{T}$ or $\lambda(t)=\log_b(t)$ where $b>1$.

\textbf{\textit{2). Parameter clipping}} \textit{(Algorithm 1,  Lines 17–18)}: 
After the central server obtains the set of benign clients $\Lambda$ via Eq.(\ref{EqP1}), it applies the Trimmed-mean method \cite{yin2018byzantine} to trim $\left\{\theta_{t}^{i}\right\}_{\Lambda}$. Specifically, during the $t$-th communication round, the central server trims $\left | \Lambda  \right | $ local model updates according to Eq.(\ref{EqP2}). $\widetilde{\theta_t}(k)$ denotes the trimmed parameter value in dimension $k$ after application of Trimmed-mean. $\beta$ is the trimming ratio coefficient, representing the discard of the maximum and minimum $\beta$ values. Typically, it is assumed that a FL system has at most $\beta$ malicious clients. $\theta^{i_\Lambda}_{t}(k)$ denotes the sequence obtained by sorting the original set in ascending order (i.e., $\theta^{1}_{t}(k) \le \theta^{2}_{t}(k), \ldots, \le \theta^{i_\Lambda}_{t}(k)$),
\begin{equation}
	\label{EqP2}
	\begin{aligned}
		\widetilde{\theta_{t}}(k) =\frac{1}{\left | \Lambda  \right | -2 \beta} \sum_{i_\Lambda =\beta+1}^{\left | \Lambda  \right |-\beta} \theta ^{i_\Lambda}_{t}(k).
	\end{aligned}
\end{equation}
Calculate the change $p^{1}_{t}$ between the mean of the trimmed $\widetilde{\theta_{t}}$ and $\theta^i_{t}$. $p^{1}_{t}$ represents the magnitude of correction applied by the parameter trimming operation to the overall trend of the original parameters. The smaller the value of $p^{1}_{t}$, the less the impact of the trimming operation on the original trend, and the less effective it is at eliminating the influence of outliers. The smaller $p^{1}_{t}$, the less the impact of the trimming operation on the original trend, making it less effective at eliminating the influence of outliers,

\begin{equation}
	\label{EqP2-2}
	\begin{aligned}
		p^{1}_{t} \gets \frac{1}{d} \left\|\widetilde{\theta_{t}} - \text{mean}(\theta^i_{t}) \right\|
	\end{aligned}.
\end{equation}

\begin{algorithm}[H]
	\caption{ Weight Update Algorithm}
	\label{WUA}
	\begin{algorithmic}[1]
		\State \textit{// Weighting updating} 
		\State \( \beta_1, \beta_2, \beta_3 \gets \text{Updating}(p^1_t, p^2_t) \)
		\If{\(p^1_t \ge \rho_1\)}
		\State \( \beta_2 \gets \min(\beta_2 + \delta_{\text{high}}, \beta_{2\_\max}) \)
		\State \( \beta_1 \gets \max(\beta_1 - \delta_{\text{high}}, \beta_{1\_\min}) \)   
		\ElsIf{\(p^2_t < \rho_2\)}
		\State \( \beta_3 \gets \max(\beta_3 - \delta_{\text{low}}, \beta_{3\_\min}) \)
		\State \( \beta_2 \gets \min(\beta_2 + \delta_{\text{high}}, \beta_{2\_\max}) \)  
		\Else
		\State \( \beta_1 \gets \beta_{1\_\text{base}} + \kappa(1 - p^1_t) \)
		\EndIf
		
		\State \( \beta_{\text{sum}} \gets \beta_1 + \beta_2 + \beta_3 \)
		\State \( \beta_{1-3} \gets \beta_{1-3} / \beta_{\text{sum}} \)
		%
		\State \( \theta_t \gets \beta_1 \cdot \text{mean}(\{\theta_t^i\}_{i\in[\Lambda] }) + \beta_2 \cdot \widetilde{\theta_{t}} + \beta_3 \cdot \overline{\theta_t} \)
		
	\end{algorithmic}
	
\end{algorithm}

\vspace{-4mm}
\begin{algorithm}[H]
	\caption{ Weight updates without thresholds ($\rho_1,\rho_2$)}
	\label{WUWT}
	\begin{algorithmic}[1]
		\State \textit{// Weighting updating}
		\State \( \beta_1, \beta_2, \beta_3 \gets \text{Updating}(p^1_t, p^2_t) \)
		\State \textit{// Weighting update}
		\State $\text{base} \gets 1 + p^1_t + \frac{1}{p^2_t + \epsilon}$
		\State $\beta_1 \gets \frac{1}{\text{base}}$
		\State $\beta_2 \gets \frac{p^1_t}{\text{base}}$
		\State $\beta_3 \gets \frac{1/(p^2_t + \epsilon)}{\text{base}}$
		
		\State \( \theta_t \gets \beta_1 \cdot \text{mean}(\{\theta_{t}^{i}\}_{i\in[\Lambda] }) + \beta_2 \cdot \widetilde{\theta_t} + \beta_3 \cdot \overline{\theta_{t}} \)
		
	\end{algorithmic}
\end{algorithm}

\textbf{\textit{3).  Derivative model}} \textit{(Algorithm 1,  Lines 19–20)}: 
After obtaining the set of benign clients $\left | \Lambda  \right | $ via the central server's pass-through Eq.(\ref{EqP1}), the central server constructs a new synthetic update model using the existing $\theta ^{i_\Lambda}_{t}(k)$. As shown in Eq.(\ref{EqP3}), the parameters of each model are computed in the same dimension, and the maximum and minimum values are selected,

\begin{equation}
	\label{EqP3}
	\begin{aligned}
		\theta ^{\max }_{t}[k] & =\max \left\{\theta^{1}_{t}[k], \ldots, \theta^{\left | \Lambda  \right |  }_{t}[k]\right\}, k \in[d], \\
		\theta^{\min }_{t}[k] & =\min \left\{\theta^{1}_{t}[k], \ldots, \theta^{\left | \Lambda  \right |  }_{t}[k]\right\}, k \in[d].
	\end{aligned}
\end{equation}

After obtaining $\theta^{\max}_{t}[k]$ and $\theta^{\min}_{t}[k]$, the server assigns a trust score $S^{i}_{t}$ to each benign client using Eq.(\ref{EqP4}). This score quantifies the consistency between each benign client's local model $\theta_t^i$ and the extreme updates ($\theta^{\max}_{t}[k]$, $\theta^{\min}_{t}[k]$),

\begin{equation}
	\label{EqP4}
	\begin{aligned}
		S^{i}_{t}=\min \left\{\left\|{\theta }^{i}_{t}-{\theta }^{\max }_{t}\right\|,\left\|{\theta }^{i}_{t}-{\theta }^{\min }_{t}\right\|\right\}.
	\end{aligned}
\end{equation}

The larger value of $S^i_t$, the farther $\theta_t^i$ is from outliers, indicating a lower probability that $\theta_t^i$ belongs to a malicious model. To further optimize the accuracy of model aggregation, the derived model should be sourced from models distant from extreme values. Therefore, AdaBFL selects the model from the client with the highest score as the derived model. $i^* \in \left [ \Lambda  \right ] $ is defined as the client with the highest score, such that for any $i \ge \left [ \Lambda  \right ] $, $S_{i*}^ t\ge S^{i}_t$ holds. This selection criterion can be defined as Eq.(\ref{EqP5}):

\begin{equation}
	\label{EqP5}
	\begin{aligned}
		i_{*}=\underset{i \in[n]}{\operatorname{argmax}} \  S^{i}_{t}.
	\end{aligned}
\end{equation}

Define the number of synthetic updates as $m$. The central server assigns the synthetic update as $\hat{\theta}^j_{t}=\theta_{i_*}^{t}$, where $j\in \left [ m \right ] $. The newly synthesized $m$ models are fused with the previous $\left | \Lambda  \right | $ models. The fused model is then trimmed using the Trimmed-mean to obtain the model $\overline{\theta_{t}}$, with the update defined by Eq.(\ref{EqP6}).

\begin{equation}
	\label{EqP6}
	\begin{aligned}
		\hat{\theta}^{t}[k] &= \text{Trimmed-mean}\left\{\theta^{1}_{t}[k], \ldots, \theta^{n}_{t}[k], \overline{\theta}^1_{t}[k], \ldots, \overline{\theta}^m_{t}[k]\right\} \\ &
		\text{s.t.} \quad \overline{\theta}^1_{t}[k] = \cdots = \overline{\theta}^m_{t}[k] = \theta^{i_*}_{t}[k].
	\end{aligned}
\end{equation}

Weighting updating: 
After completing the operations of  Filter malicious, Parameter clipping and Derivative model, the central server fuse the obtained $\{\theta_t^i\}_{i\in \left [ \Lambda  \right ]  }, \widetilde{\theta^{i}_{t}}, \overline{\theta_{t}}$ according to Eq.(\ref{EqP7}) to derive the globally updated model $\theta_{t}$,

\begin{equation}
	\label{EqP7}
	\begin{aligned}
		\theta_{t} \leftarrow \beta_{1} \cdot \operatorname{mean}\left(\left\{\theta_{t}^{i}\right\}_{i\in \left [ \Lambda  \right ]  }\right)+\beta_{2} \cdot \widetilde{\theta_{t}}+\beta_{3} \cdot \overline{\theta_{t}},
	\end{aligned}
\end{equation}
where $\beta_1,\beta_2,\beta_3$ represent the respective weights assigned to $\{\theta_t^i\}_{\Lambda}$, $\widetilde{\theta^{i}_{t}}$, and $\overline{\theta_{t}}$ during aggregation, satisfying the constraint $\beta_1+\beta_2+\beta_3=1$.

If $p^1_{t} \ge \rho_1$, then update $\beta_1,\beta_2$ according to Eq.(\ref{EqP8}). Where $p^{1}_{t} \gets \frac{1}{d} \left\|(\widetilde{\theta_{t}} - \text{mean}(\theta^{i}_{t})) \right\|$, which measures the difference between the pruned model $\widetilde{\theta_{t}}$ and the mean of benign client models $mean(\theta_t^i)$. $\rho_1$ represents the tolerance level for this difference, $\delta_{\text{high}}$ denotes the maximum increment for weight adjustment, $\beta_{2\_max}$ is the upper bound for the maximum weight assigned to $\beta_2$, and $\beta_{1\_min}$ is the lower bound for the minimum weight assigned to $\beta_1$,

\begin{equation}
	\label{EqP8}
	\begin{aligned}
		\beta_2 \gets \min(\beta_2 + \delta_{\text{high}}, \beta_{2\_\max}),\\
		\beta_1 \gets \max(\beta_1 - \delta_{\text{high}}, \beta_{1\_\min}). 
	\end{aligned}
\end{equation}

If $p^2_{t} \ge \rho_2$, then update $\beta_3, beta_2$ according to Eq.(\ref{EqP9}), where $p^{2}_{t} \leftarrow \left \| \frac{1}{|\Lambda|} \sum_{i\in \left [ \Lambda  \right ] } \left(\theta_{t}^{i}, \overline{\theta_{t}}\right)  \right \| $, which measures the discrepancy between the fused model $\overline{\theta_{t}}$ and the mean of benign client models $mean(\theta_t^i)$.  $\rho_2$ represents the maximum tolerated difference, $\delta_{\text{low}}$ denotes the weight increment magnitude, and $\beta_{3\_min}$ is the minimum lower bound for $\beta_3$'s weight,

\begin{equation}
	\label{EqP9}
	\begin{aligned}
		\beta_3 \gets \max(\beta_3 - \delta_{\text{low}}, \beta_{3\_\min}),\\
		\beta_2 \gets \min(\beta_2 + \delta_{\text{high}}, \beta_{2\_\max}).
	\end{aligned}
\end{equation}

If $p^1_{t} < \rho_1$ and $p^2_{t} \ge \rho_2$ are not satisfied, then update $\beta_{1}$ directly according to Eq.(\ref{EqP10}). Where, $\beta_{1\_ base}$ denotes the minimum weight lower bound for $\beta_{1}$, and $\kappa$ is a parameter,

\begin{equation}
	\label{EqP10}
	\begin{aligned}
		\beta_1 \gets \beta_{1\_\text{base}} + \kappa(1 - p^1_{t}).
	\end{aligned}
\end{equation}

\textbf{Remark}: 
Algorithm\ref{alg:dfl} and Algorithm\ref{WUA} provide pseudocode for the AdaBFL framework. The steps involving Trimmed-Mean and mean in Algorithm\ref{alg:dfl} and Algorithm\ref{WUA} can be replaced with the \textit{Median}, as \textit{Median} can substitute mean in discarding outliers. Furthermore, to address the limitation in Algorithm\ref{WUA} where the \textit{Weighting update} step relies on thresholds $\rho_1$ and $\rho_2$, Algorithm\ref{WUWT} is proposed to overcome this constraint, enabling adaptive weight adjustment without thresholds. Finally, to optimize Algorithm\ref{alg:dfl}, we propose the momentum-based AdaBFL framework $AdaBFL_{Mo}$ (Algorithm\ref{alg:AdaM} see Appendix).

\section{Theoretical Analysis}
\label{Theoretical_analysis}
In this section, we provide convergence analysis for our AdaBFL method under the non-convex setting. We first give some standard assumptions on the above problem (\ref{P1}), which are commonly used in related FL works~\cite{mcmahan2017communication,li2020federated,karimireddy2020scaffold,yin2018byzantine,fang2025we,fang2024byzantine,el2021collaborative}. For detailed proof, please refer to the following Appendix.

\begin{assumption} 	\label{ass1}
	The objective function $F(\theta)$ is $L$-smooth. For any $\theta_1,\theta_2 \in \R^d$, we have 
	\begin{align}
		\left \| \nabla F(\theta_1) - \nabla F(\theta_2) \right \| \le L  \left \| \theta_1 - \theta_2 \right \| ,
		\\
		F(\theta_2) \le F(\theta_1) + \langle \nabla F(\theta_1), \theta_2 - \theta_1 \rangle + \frac{L}{2} \|\theta_2 - \theta_1\|^2 .
	\end{align} 
\end{assumption}

\begin{assumption} 	\label{ass2}
	The stochastic gradient $\nabla f^i(\theta;\xi^i)$ is an unbiased estimate of the true gradient and its variance is bounded, such that 
	\begin{align}
		\mathbb{E} \left [ \nabla f^i(\theta;\xi^i) \right ]  = \nabla f^i(  \theta), \  \mathbb{E} \left [ \left \|  \nabla f^i(\theta;\xi^i) - \nabla f^i(  \theta) \right \|  \right ]^2 \le \sigma ^2, \ \forall \ i\in [n].      
	\end{align}     
\end{assumption}
\begin{assumption} 	\label{ass3} 
	For any client $ i\in [n] $, the variance of its true local gradient $\nabla f^i(\theta)$ with respect to the global gradient $\nabla F(\theta)$ is bounded, such that
	\begin{align}
		\| \nabla f^i(\theta) - \nabla F(\theta)  \|^2 \le \tau ^2, \ \forall \ \theta \in \R^d .     
	\end{align}
\end{assumption}

\begin{assumption} 	\label{ass4} 
The function $F(\theta)$ has a lower bounded, i.e.,  $F^* = \inf_{\theta\in \R^d} F(w)>-\infty$.
\end{assumption}

Based on the above assumptions, theoretical results for the our AdaBFL framework are given.

\begin{theorem} \label{th1}
	Suppose the sequence $\{\theta_t\}_{t=0}^T$ is generated from our Algorithm~\ref{alg:dfl}. Let Assumptions \ref{ass1},\ref{ass2},\ref{ass3},\ref{ass4} hold, and given $\eta \in \left ( 0, \frac{1}{8L}  \right ] $ and $\|e_t\|\leq K\eta^2$, we have 
	\[
	\frac{1}{T} \sum_{t=0}^{T-1} \mathbb{E}\left[\left\|\nabla F\left(\theta_{t}\right)\right\|^{2}\right] \leq \frac{2\left ( F\left(\theta_{0}\right)-F^* \right ) }{\eta T} + \frac{7K^2\eta^2}{4}   + 2L\eta D,
	\]
	where $D = \frac{2\sigma ^2}{bN} +\tau ^2$, $K=\sqrt{\sigma + \tau} $, and $e_t = \theta_{t}- \left ( \theta_{t-1} - \frac{1}{N} \sum_{i \in [N]} g^{i}_{t}  \right )$.
\end{theorem}

\begin{remark}
	Based on the above Theorem~\ref{th1}, gave $\eta=O(\frac{1}{\sqrt{T}})$, $b=O(1)$,  $L=O(1)$, $\sigma=O(1)$ and $\tau=O(1)$, 
	we can obtain 
	\begin{align}
		\frac{1}{T} \sum_{t=0}^{T-1} \mathbb{E}\left\|\nabla F\left(\theta_{t}\right)\right\|^2 \leq O(\frac{1}{\sqrt{T}}).
	\end{align}
	Thus, our AdaBFL algorithm has the same convergence rate of $O(\frac{1}{\sqrt{T}})$ as the standard FL algorithms~\cite{li2019convergence,li2020federated,karimireddy2020scaffold} under the nonconvex setting on non-iid data. 
\end{remark}

\section{Numerical Experiments}
In this section, we provide extensive numerical experiments to 
demonstrate effectiveness of our AdaBFL method.

\subsection{Experimental Setup}
In the subsection, we give the experimental setup on numerical experiments. 
\textit{1) Datasets:} In the experiment, seven datasets (Tiny ImageNet \cite{cs231n_tinyimagenet}, MNIST \cite{lecun2002gradient}, Fashion-MNIST \cite{xiao2017fashion}, Human Activity Recog-
nition (HAR) \cite{anguita2013public}, Shakespeare \cite{Shakespeare}, PetImage \cite{dogs-vs-cats-redux-kernels-edition} and CIFAR-10 \cite{krizhevsky2009learning}.) from different domains were utilized, and each dataset employed a distinct training framework. 

\textbf{a)  Tiny ImageNet \cite{cs231n_tinyimagenet}:} Tiny ImageNet dataset is commonly used for training and validation in image classification tasks. This dataset comprises images across 200 categories, with each category containing 500 training images and 50 validation images, totaling 100,000 images. 

\textbf{b)  MNIST \cite{lecun2002gradient}:} MNIST dataset is a handwritten digit recognition database comprising 60,000 training images and 10,000 test images, with a total of 10 image categories.  

\textbf{c) Fashion-MNIST \cite{xiao2017fashion}: }  Fashion-MNIST contains 60,000 training images and 10,000 test images in total. Each image is a 28x28 grayscale image, with a total of 10 image categories.  

\textbf{d) Human Activity Recognition (HAR) \cite{anguita2013public}: }HAR is a dataset (comprising six categories) collected from actual human activities, derived from the smartphone usage records of 30 users aged 19 to 48. Similar to \cite{fang2025we}, in the experimental simulation, 75\% of each user's data is randomly allocated as training samples, while 25\% is designated as test samples.

\textbf{e) Shakespeare \cite{Shakespeare}: }The Shakespeare dataset is a benchmark dataset specifically designed for federated learning, containing over 4,226,15 samples, each with a string length of 80 characters. This dataset is commonly used for character prediction tasks, which involve predicting the next character following a given string.

\textbf{f) PetImage \cite{dogs-vs-cats-redux-kernels-edition}: }PetImage is a cat and dog image classification dataset provided by Kaggle, containing a total of 25,000 images, with 12,500 images each of cats and dogs.  

\textbf{g) CIFAR-10 \cite{krizhevsky2009learning}:} The CIFAR-10 dataset was constructed through extensive sampling of real-world images, covering 10 distinct categories. Each category contains 6,000 color images at 32x32 pixels, totaling 60,000 images. The dataset is randomly split into 50,000 training images and 10,000 test images, ensuring a balanced representation during both the training and evaluation phases.

\textit{2) Poisoning Attacks: }
Similar to related state-of-the-art approaches \cite{fang2025we}, our experiments explored seven attacks: Label Flipping attack \cite{tolpegin2020data}, Gaussian attack \cite{blanchard2017machine}, Trim attack \cite{fang2020local}, Krum attack \cite{fang2020local}, Min-Max attack \cite{shejwalkar2021manipulating}, Min-Sum attack \cite{shejwalkar2021manipulating}, Scaling attack \cite{cao2020fltrust}) to evaluate our approach. These attacks are currently well-established and represent the primary threats posing substantial risks to FL systems.

\textbf{a) Label flipping attack (LFA) \cite{tolpegin2020data}:} LFA attack manifest as attackers modifying the labels of client training data. Specifically, for labels initially marked as $y$, attackers alter them to $M-y-1$, where $M$ is the total number of labels.

\textbf{b) Gaussian attack \cite{blanchard2017machine}:}
In this attack, malicious clients send randomly generated vectors to the central server. These vectors must follow a Gaussian distribution, typically sampled from a Gaussian distribution with a mean of 0 and a variance of 200.

\textbf{c) Trim attack \cite{fang2020local}:} Trim attack is a strategy exploiting vulnerabilities in Trim aggregation policies to target mean and median trimming in aggregation. Attackers manipulate the global model away from its intended target by crafting malicious local update models for client devices.

\textbf{d) Krum attack \cite{fang2020local}:} Krum attack aim to exploit Krum aggregation rules by altering model pruning precision or launching attacks based on proportions.

\textbf{e) Min-Max attack \cite{shejwalkar2021manipulating}:} 
The core objective of the Min-Max attack lies in performing fine-tuning operations on models to achieve a specific outcome: the distance between a malicious local model and any benign client model must be smaller than the maximum distance between any two benign models after their respective updates.

\textbf{f) Scaling attack \cite{cao2020fltrust}:} Scaling attacks are targeted attack where attackers first implant backdoor triggers in duplicate data copies to inflate malicious clients' local training datasets. These clients then use the inflated data to train models, amplifying model updates before transmitting them to the server.

\textbf{g) Sybil attack \cite{XiongTII}:} Sybil attack are a classic example of collusive attacks, where attackers control a large number of fake clients to coordinate the upload of identical malicious model updates, thereby manipulating or disrupting the collaborative training of global models.

\textit{3) Compared Aggregation Rules:} We compared AdaBFL with the following nine aggregation rules.

\textbf{FedAvg \cite{mcmahan2017communication}:} The lack of sufficient robustness in this aggregation method manifests as follows: after obtaining local model updates from all participating nodes, the server aggregates these updates using only a single method—averaging.

\textbf{Trimmed-Mean(Trim-mean) \cite{yin2018byzantine}:} It is a coordinate-based aggregation method whose operational logic is as follows: For each dimension, the server first removes the largest $k$ elements and the smallest $k$ elements from that dimension, then calculates the average of the remaining elements.

\textbf{GAS+TRIM-Mean \cite{liu2023byzantine}: }The aggregation process proceeds as follows: After receiving local model updates from all clients, the server first decomposes each update into multiple segments. Next, it applies the Trimmed-Mean \cite{yin2018byzantine} aggregation rule to each segment to obtain segment-level aggregated results. Subsequently, the server calculates an identification score for each client and selects the $n-f$ local model updates with the lowest identification scores (where \(f\) represents the total number of malicious clients). The final aggregation is completed by computing the average of these \(n-f\) updates.

\textbf{Gaussian+Trim-mean (Gau+Trim-mean):} After receiving $n$ local model updates from the client, the server generates $m$ synthetic updates. The parameters of the $k$-th dimension in these composite models follow a Gaussian distribution $\mathbin{N}(\mu, \sigma^2 )$, where $\mu$ and $\sigma^2$ represent the mean and standard deviation of $\left \{ g^{i}_{t}\left [ k \right ], i \in \left [ n \right ]  \right \} $, respectively. Subsequently, the server merges these $m$ synthetic updates with the $n$ received updates using the Trim-mean aggregation rule.

\textbf{Median \cite{yin2018byzantine}:} As another aggregation strategy implemented by coordinate, the specific approach is: for each dimension, the server calculates the median of that dimension from all received model updates.

\textbf{GAS+Median \cite{liu2023byzantine}: }The aggregation logic of GAS+Median is similar to that of GAS+Trim-Mean. The server divides each local model update into multiple parts, but in this method, each part uses the Median\cite{yin2018byzantine} aggregation rule. Subsequently, the server selects the $n-f$ local model updates with the lowest identification scores and calculates the final aggregated update result by computing the average of these updates.

\textbf{Gaussian + Median(Gau+Median):} The server generates $m$ synthetic updates following the same steps as the Gaussian + Trim-mean method, with the sole difference being the aggregation method. Here, the server employs the Median\cite{yin2018byzantine} aggregation rule to combine the $m$ synthetic updates with the client's $n$ local model updates.

\textbf{FoundationFL \cite{fang2025we}:} FoundationFL is a specialized aggregation rule where, after receiving $n$ local model updates from clients, the server computes a score $S^{i}_{t}$ for each client $(S^{i}_{t} = \min \left\{\left\|\boldsymbol{g}^{i}_{t} - \boldsymbol{g}^{\max}_{t}\right\|,\left\|\boldsymbol{g}^{i}_{t}-\boldsymbol{g}^{\min }_{t}\right\|\right\})$ for each client. It then selects the $m$ local models with the highest scores to derive $m$ synthetic updates. Finally, use Trimmed-mean or Median \cite{yin2018byzantine} to aggregate the $n$ local model updates from clients with the $m$ synthetic model updates.

\textit{Non-iid Setting:}
In the field of Federated Learning (FL), the non-iid (non-independent and identically distributed) property of client-side local training data represents a key distinction from traditional centralized learning. This study constructs a non-iid Setting experimental environment based on the methods outlined in \cite{fang2020local}\cite{fang2025we}: For a dataset with $M$ categories, all clients are first randomly partitioned into $M$ clusters; During sample allocation, a training sample labeled $y$ has a probability $h$ of being assigned to clients in cluster $y$ and a probability $\frac{(1-h)}{(M-1)}$ of being assigned to clients in other clusters. A higher value of $h$ indicates a more pronounced non-iid setting property in the client training data. For the Tiny Imagenet, MNIST, Fashion-MNIST, Petimage, and Cifar-10 datasets, we set $h=0.5$. Since the Shakespeare and Har datasets inherently possess heterogeneous characteristics, we did not simulate a non-iid setting for them. By default, results are displayed using parallel MNIST data. The evaluation metric in this paper is the \textit{Testing error rate}, which represents the percentage of test instances misclassified by the global model. A lower testing error rate indicates stronger adversarial resilience.

\subsection{Experimental Results}
\label{Evaluation}

\begin{table*}[htbp]
	\centering
	\caption{Validation results (\textbf{testing error rate}) of different aggregation methods on the Tiny Imagenet, MNIST, Fashion-MNIST, HAR, Cifar10, Petimage, and Shakespeare datasets. The proportion of malicious clients is 0.3.}
	\label{AdaBFL_effective}
	
	
	\centering
	
	\textbf{(a) Tiny ImageNet dataset}
	
	\vspace{0.2cm}
	\begin{tabular}{lccccccccc}
		\toprule
		Method & No  & LF  & Gaussian  & Trim  & Krum  & Min-Max  & Scaling  & Sybil \\
		\midrule
		
		FedAvg     & 0.279	& \textbf{0.280}	& 0.888	& 0.909	& 0.995	& 0.995	& 0.995 & 	0.995 \\
		
		Trim-mean      & \textbf{0.272}	 & 0.293	 & 0.803 & 	0.848	 & 0.995 & 	0.995 & 	0.995 & 	0.995 \\
		
		GAS+Trim-mean      & 0.283 & 	0.289 & 	0.885 & 	0.876 & 	0.995 & 	0.995 & 	0.995 & 	0.995\\
		
		FoundationFL+Median  &  0.273 & 0.281 &	0.295 & 0.331 & 0.328 & 0.372 & 0.305 & 0.323 \\
		
		Gau+Trim-mean   &  0.341 & 	0.306 &	0.914 & 0.863 & 0.995 & 0.995 &	0.995 & 0.995  \\
		
		Median  & \textbf{0.272}	 & 0.283 & 	0.288 &	0.313 & 0.307 & 0.337 & 0.294 & 0.283 \\
		
		GAS+Median   &  0.274 & 0.282 & 0.683 & 0.630 & 0.995 & 	0.995 & 0.995 & 0.995 \\
		
		Gau+Median    &  0.276 & \textbf{0.280} & 0.303 & 0.305 &	0.335  & 	0.384 & 	0.312 & 0.321  \\
		
		FoundationFL+Mean   & 0.290	 & 0.290 & 0.760 & 	0.401 & 0.674 & 0.995 & 0.995 & 0.995 \\
		
		\rowcolor{lightblue}
		AdaBFL(ours)    & \textbf{0.272}	& 0.282	& \textbf{0.270}	& \textbf{0.291}	& \textbf{0.282}	& \textbf{0.292}	& \textbf{0.287}	& \textbf{0.278} \\
		\bottomrule
	\end{tabular}
	
	\vspace{0.2cm} 
	\centering
	\textbf{(b) Cifar-10 dataset}
	
	\vspace{0.2cm}
	\begin{tabular}{lccccccccc}
		\toprule
		Method & No  & LF  & Gaussian  & Trim  & Krum  & Min-Max  & Scaling  & Sybil \\
		\midrule
		
		FedAvg     & \textbf{0.268} & 0.272 & 0.900 & 0.900 & 0.900 & 0.900 & 0.900  & 0.900 \\
		
		Trim-mean     & 0.272 & 0.269 & 0.900 & 0.900 & 0.900 & 0.850 & 0.900  & 0.900 \\
		
		GAS+Trim-mean     & 0.563 & 0.577 & 0.864 & 0.500 & 0.900 & 0.589 & 0.865  & 0.900 \\
		
		FoundationFL+Median     & 0.275 & 0.275 & 0.314 & 0.340 & 0.273 & 0.363 & 0.351  & 0.388 \\
		
		Gau+Trim-mean     & 0.271 & 0.273 & 0.557 & 0.676 & 0.900 & 0.671 & 0.900  & 0.900 \\
		
		Median     & 0.280 & 0.278 & 0.293 & 0.427 & 0.322 & 0.336 & 0.449  & 0.432 \\
		
		GAS+Median     & 0.555 & 0.572 & 0.570 & 0.594 & 0.583 & 0.579 & 0.550  & 0.485 \\
		
		Gau+Median     & 0.272 & 0.273 & 0.314 & 0.518 & 0.900 & 0.389 & 0.900  & 0.900 \\
		
		FoundationFL+Mean     & 0.272 & 0.272 & 0.900 & 0.412 & 0.737 & 0.795 & 0.900  & 0.900 \\
		
		\rowcolor{lightblue}
		AdaBFL(ours)    & \textbf{0.268} & \textbf{0.257} & \textbf{0.286} & \textbf{0.276} & \textbf{0.274} & \textbf{0.290} & \textbf{0.278}  & \textbf{0.278} \\
		\bottomrule
	\end{tabular}
	
	\vspace{0.2cm}
	\textbf{(c) MNIST dataset}
	
	\vspace{0.2cm}
	\begin{tabular}{lccccccccc}
		\toprule
		Method & No  & LF  & Gaussian  & Trim  & Krum  & Min-Max  & Scaling  & Sybil \\
		\midrule
		
		FedAvg     & \textbf{0.010} & \textbf{0.010} & 0.897 & 0.900 & 0.91 & 0.878 & 0.900 & 0.900 \\
		
		Trim-mean     & \textbf{0.010} & \textbf{0.010} & 0.900 & 0.900 & 0.900 & 0.473 & 0.900  & 0.900 \\
		
		GAS+Trim-mean     & 0.031 & 0.030 & 0.544 & 0.027 & 0.886 & 0.236 & 0.900  & 0.900 \\
		
		FoundationFL+Median     & \textbf{0.010} & \textbf{0.010} & 0.019 & 0.011 & \textbf{0.010} & 0.016 & \textbf{0.010}  & 0.010 \\
		
		Gau+Trim-mean     & \textbf{0.010} & \textbf{0.010} & 0.900 & 0.900 & 0.900 & 0.743 & 0.900  & 0.900 \\
		
		Median     & \textbf{0.010} & \textbf{0.010} & 0.010 & 0.013 & \textbf{0.010} & 0.012 & 0.012  & 0.013 \\
		
		GAS+Median     & 0.032 & 0.032 & 0.031 & 0.028 & 0.032 & 0.034 & 0.024  & 0.025 \\
		
		Gau+Median     & \textbf{0.010} & \textbf{0.010} & 0.010 & 0.015 & 0.538 & 0.015 & 0.900  & 0.900 \\
		
		FoundationFL+Mean     & \textbf{0.010} & \textbf{0.010} & 0.883 & 0.105 & 0.325 & 0.681 & 0.900  & 0.900 \\
		
		\rowcolor{lightblue}
		AdaBFL(ours)    & \textbf{0.010} & \textbf{0.010} & \textbf{0.009} & \textbf{0.010} & \textbf{0.010} & \textbf{0.009} & \textbf{0.010}  & \textbf{0.008} \\
		\bottomrule
	\end{tabular}
	
	\vspace{0.2cm} 
	
	\centering
	
\end{table*}

\textbf{Our AdaBFL is effective:} 
Table \ref{AdaBFL_effective} shows the performance of various federated aggregation algorithms under different poisoning attacks on the Tiny ImageNet, MNIST, Fashion-MNIST, HAR, Cifar10, Petimage, and Shakespeare datasets. Results for the Fashion-MNIST, Petimage, HAR and Shakespeare datasets are presented in Table \ref{AdaBFL_effective2} in the appendix. (Note that the evaluation metric for the Shakespeare dataset is calculated perplexity, not test error rate. A lower calculated perplexity indicates better model performance.) No attack indicates that all clients in the FL system are benign and do not launch poisoning attacks. Under the No attack setting, AdaBFL achieves a test error rate comparable to that of the baseline FedAvg, demonstrating that AdaBFL can match FedAvg's effectiveness in a benign environment. However, when malicious clients exist in the FL system, AdaBFL demonstrates robust anti-attack properties. For instance, on the Tiny Imagenet, MNIST, Fashion-MNIST, HAR, and Cifar-10 datasets, AdaBFL is the only method that maintains equivalence to FedAvg under multiple attack scenarios while preserving the same utility as in the no-attack setting. Experiments also reveal inherent flaws in existing aggregation schemes under poisoning attacks. For instance, the test error rate of Gau+Trim-mean on the MNIST dataset surged from 0.010 without attacks to 0.900 under Gaussian attacks. Additionally, while the Median aggregation algorithm exhibits weaker resistance than AdaBFL on MNIST, Fashion-MNIST, and HAR datasets, it still maintains satisfactory utility. However, the Median aggregation algorithm exhibits reduced resistance on Cifar10. For instance, its test error rate increases from 0.280 without attacks to 0.449 under a Scaling attack. To further highlight AdaBFL's advantages, we also evaluated the resistance of Median and AdaBFL against various attacks when over 50\% of clients are malicious (see Appendix Table \ref{tab_Malicious_1}). When the proportion of malicious clients exceeds 50\%, Median only resists LF attacks, while AdaBFL resists LF attacks, Gaussian attacks, Trim attacks, and min-max attacks. The above analysis demonstrates that AdaBFL possesses significant advantages among Byzantine-robust aggregation algorithms.

\begin{figure*}[htbp]
	\centering
	
	\includegraphics[width=1\textwidth]{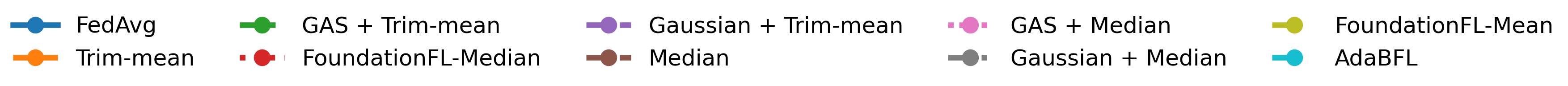}
	
	\vspace{0.2cm} 
	
	\begin{minipage}[b]{0.24\textwidth}
		\centering
		\includegraphics[width=\textwidth]{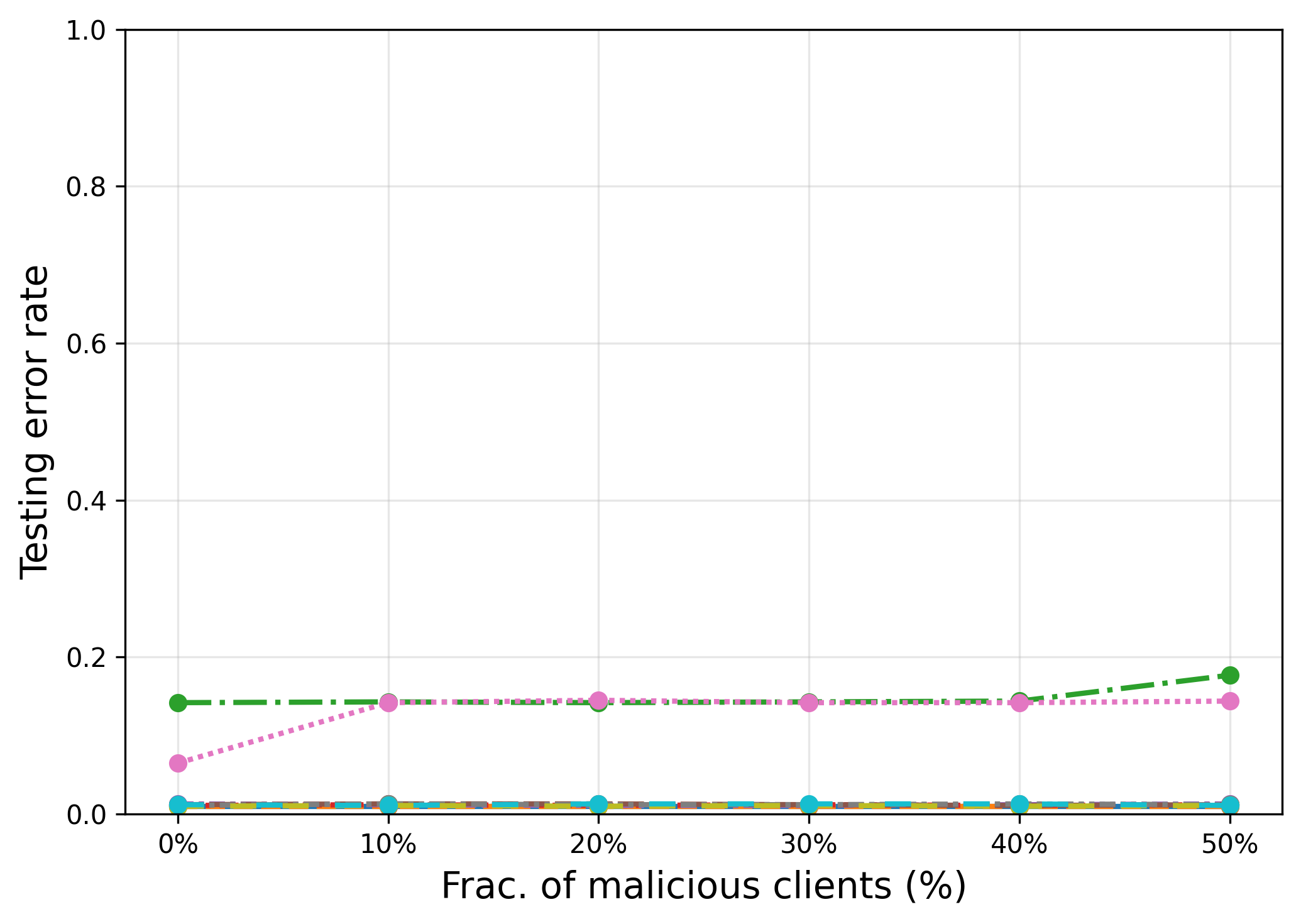}
		\small\captionof*{figure}{(a) No attack} 
	\end{minipage}
	\hfill 
	\begin{minipage}[b]{0.24\textwidth}
		\centering
		\includegraphics[width=\textwidth]{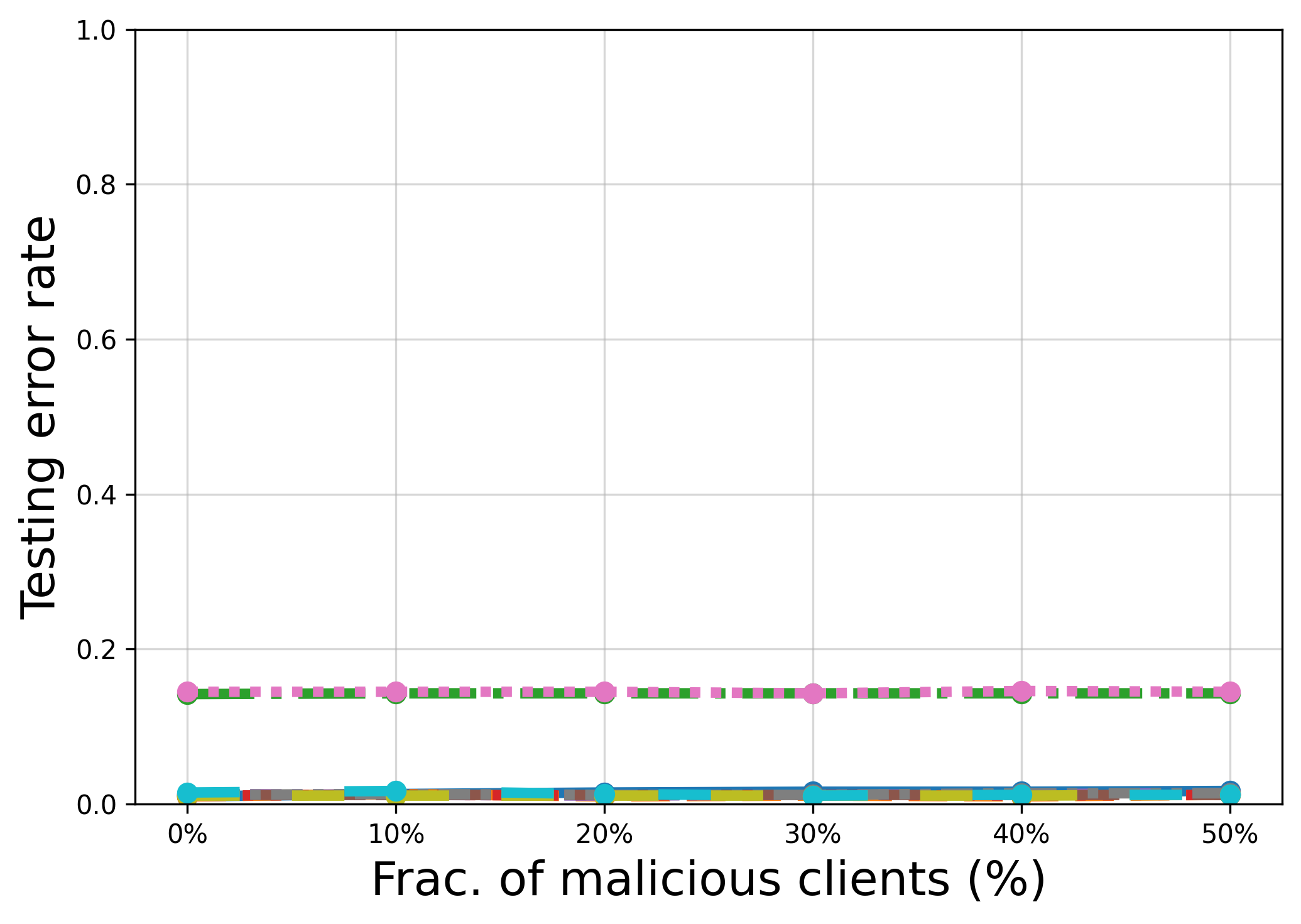}
		\small\captionof*{figure}{(b) LF attack}
	\end{minipage}
	\hfill
	\begin{minipage}[b]{0.24\textwidth}
		\centering
		\includegraphics[width=\textwidth]{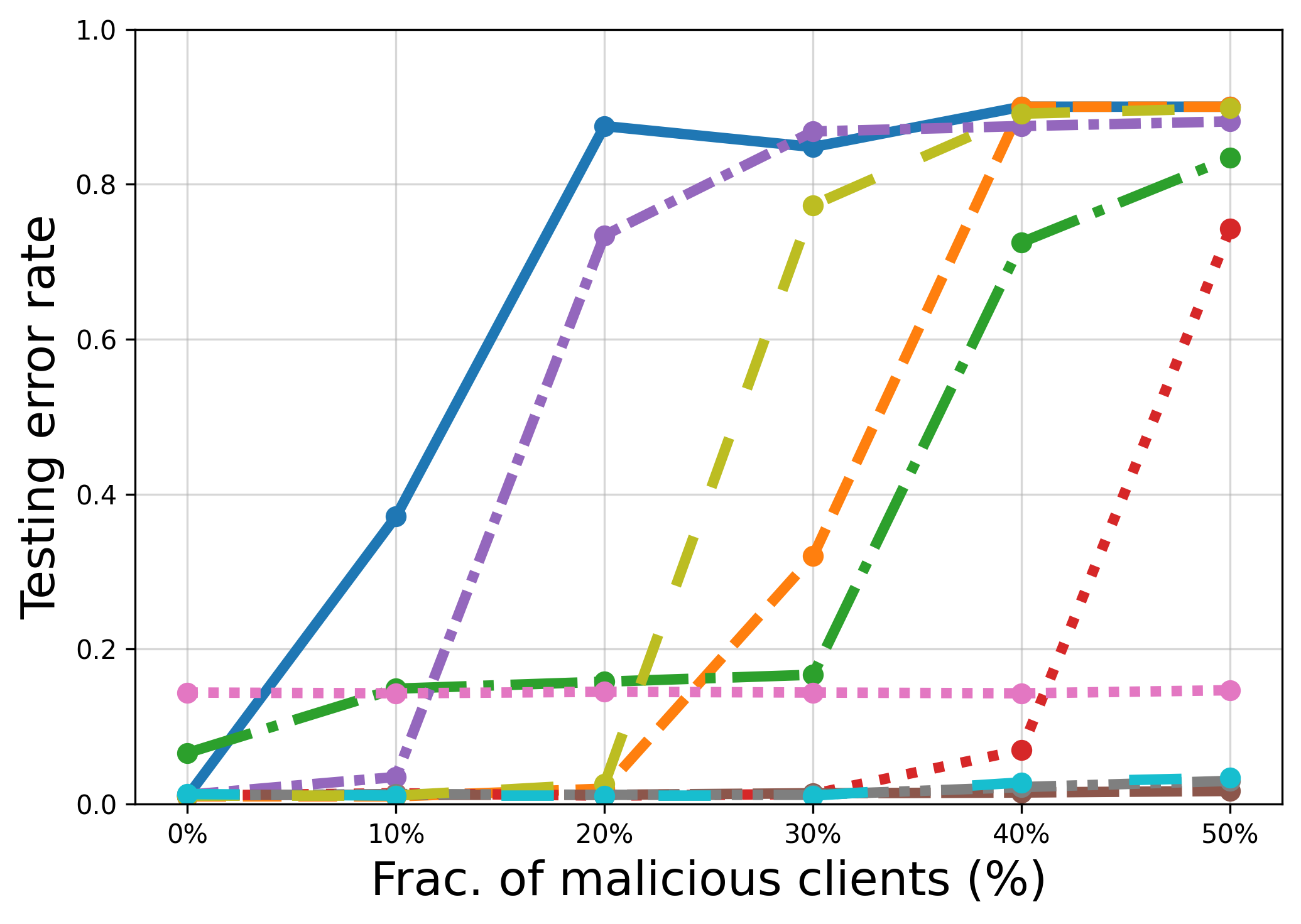}
		\small\captionof*{figure}{(c) Gaussian attack}
	\end{minipage}
	\hfill
	\begin{minipage}[b]{0.24\textwidth}
		\centering
		\includegraphics[width=\textwidth]{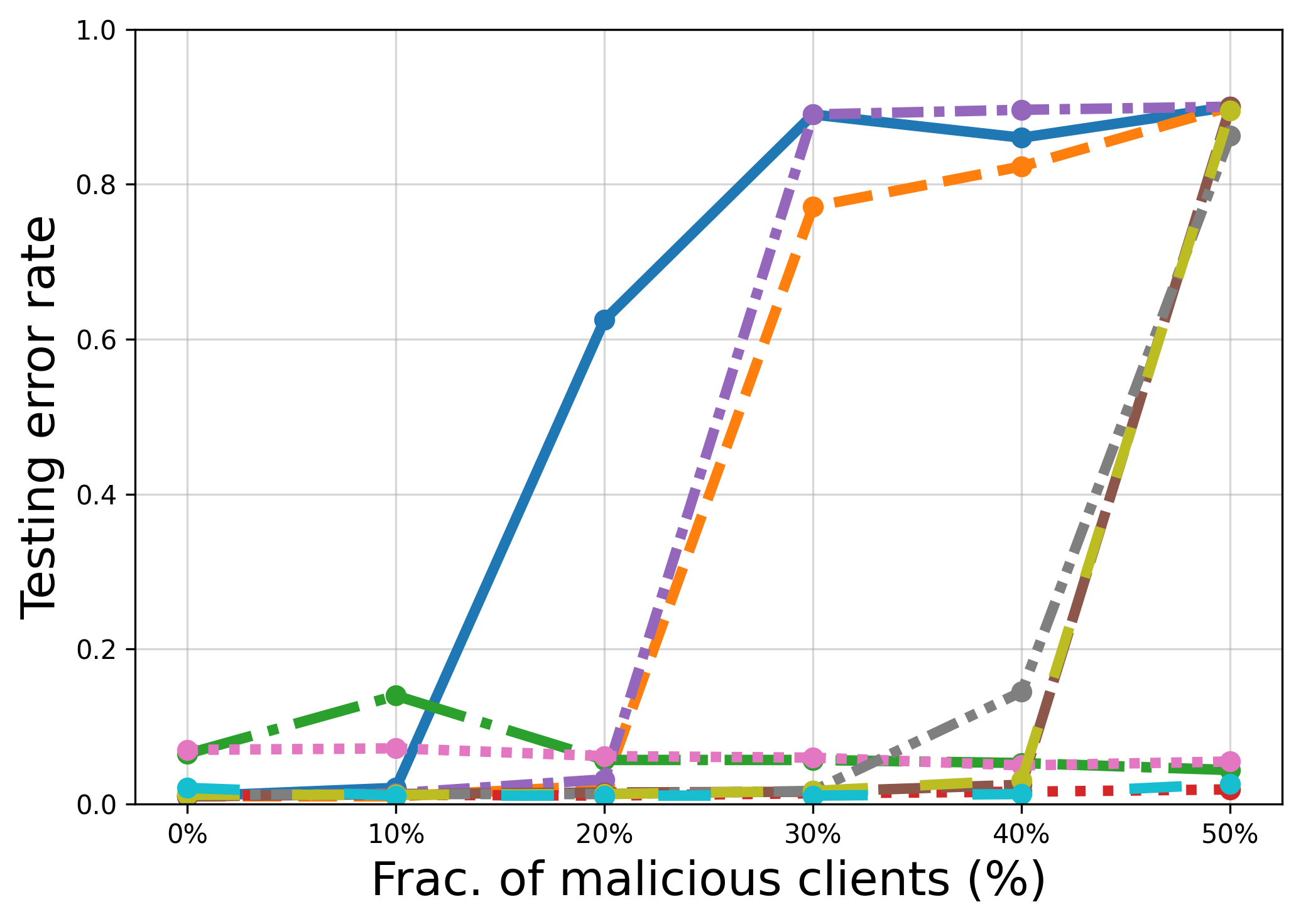}
		\small\captionof*{figure}{(d) Trim attack}
	\end{minipage}
	
	\vspace{0.15cm} 
	
	\begin{minipage}[b]{0.24\textwidth}
		\centering
		\includegraphics[width=\textwidth]{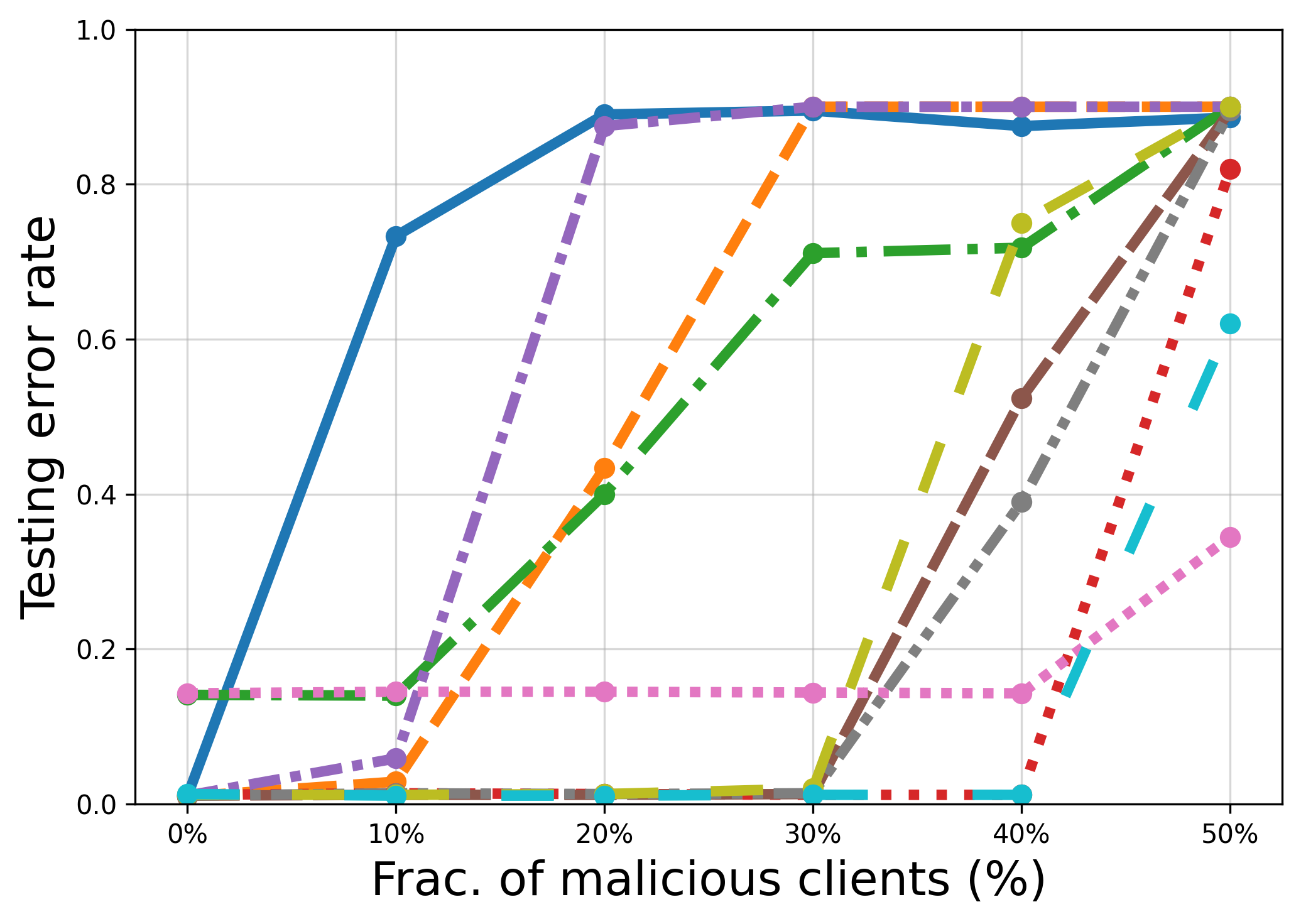}
		\small\captionof*{figure}{(e) Krum attack}
	\end{minipage}
	\hfill
	\begin{minipage}[b]{0.24\textwidth}
		\centering
		\includegraphics[width=\textwidth]{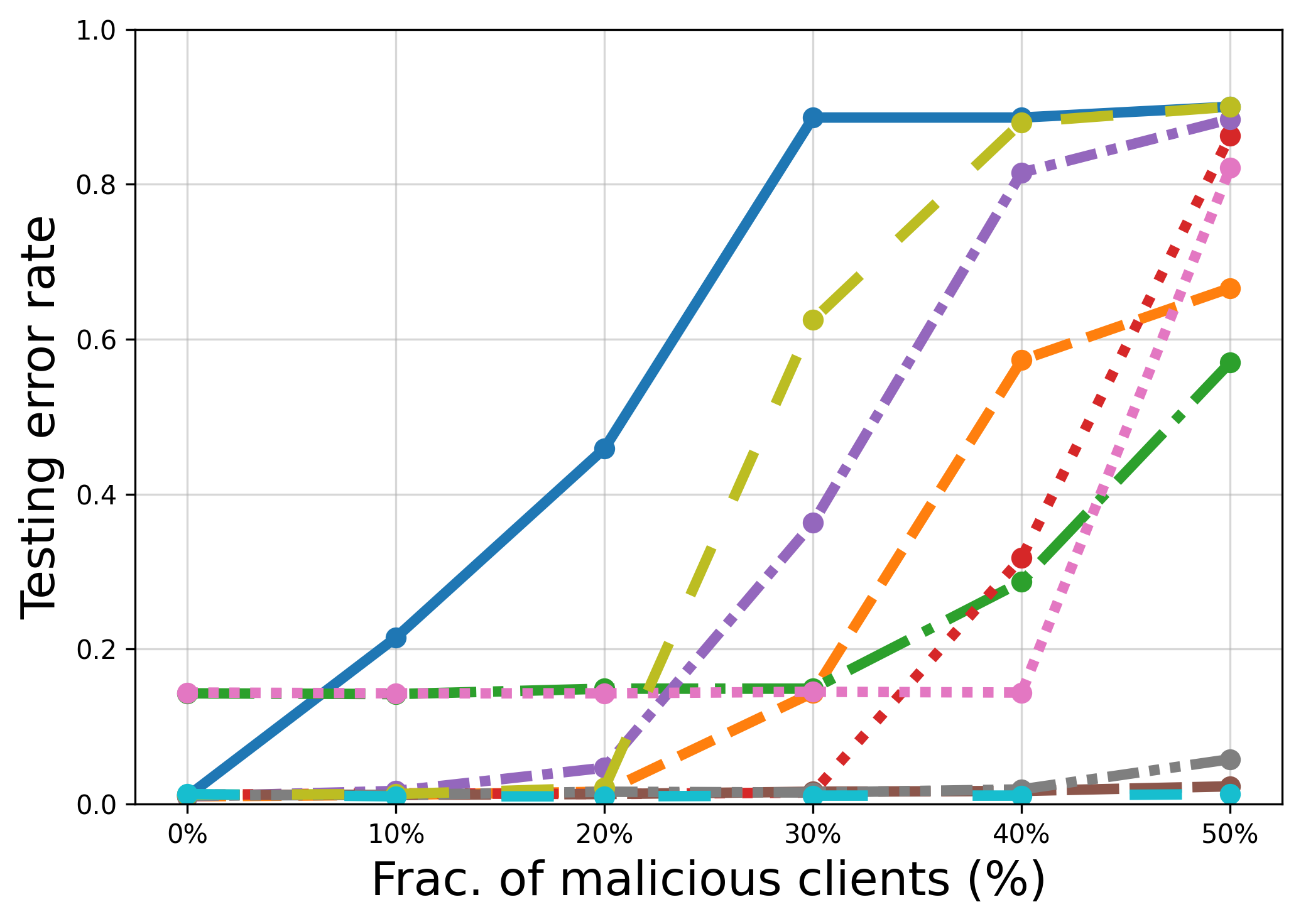}
		\small\captionof*{figure}{(f) Min-Max attack}
	\end{minipage}
	\hfill
	\begin{minipage}[b]{0.24\textwidth}
		\centering
		\includegraphics[width=\textwidth]{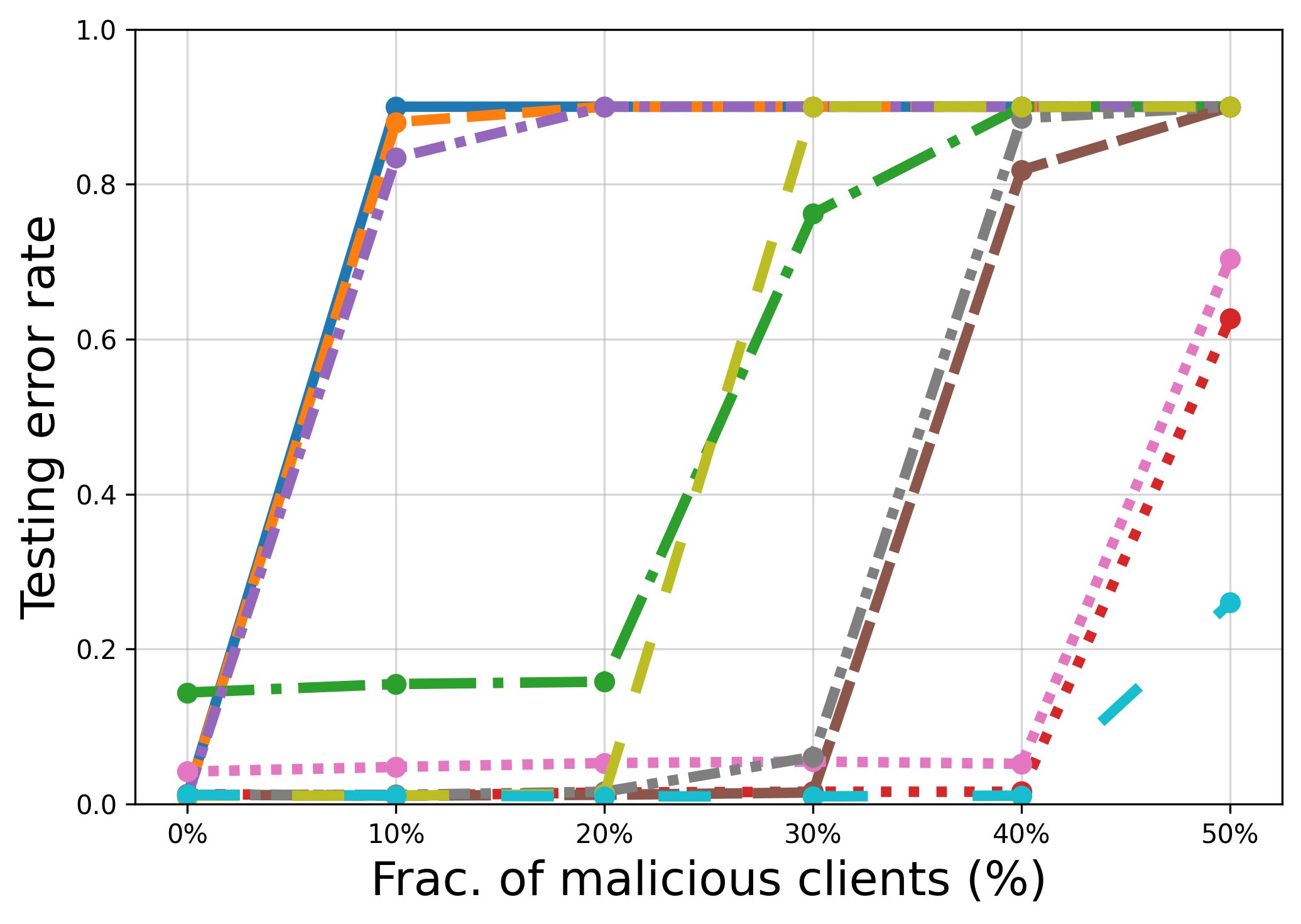}
		\small\captionof*{figure}{(g) Scaling attack}
	\end{minipage}
	\hfill
	\begin{minipage}[b]{0.24\textwidth}
		\centering
		\includegraphics[width=\textwidth]{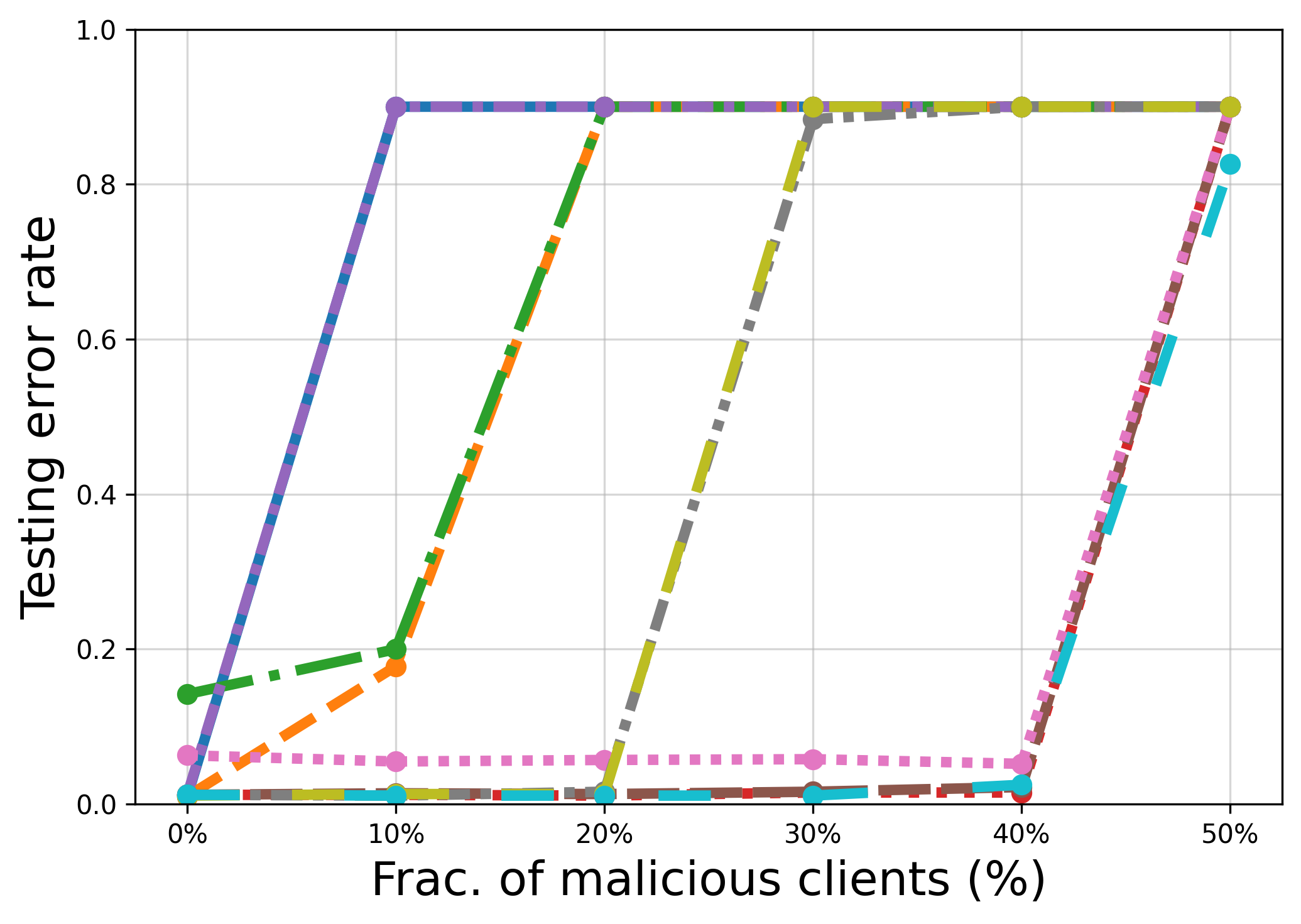}
		\small\captionof*{figure}{(h) Sybil attack}
	\end{minipage}
	
	\caption{Impact of fraction of malicious clients on MNIST (the total number of clients is 100).}
	\label{Malicious_num}
\end{figure*}

\begin{figure*}[htbp]
	\centering
	
	\includegraphics[width=1\textwidth]{Images/Experimental_Results/legend_only.png}
	
	\vspace{0.2cm} 
	
	\begin{minipage}[b]{0.24\textwidth}
		\centering
		\includegraphics[width=\textwidth]{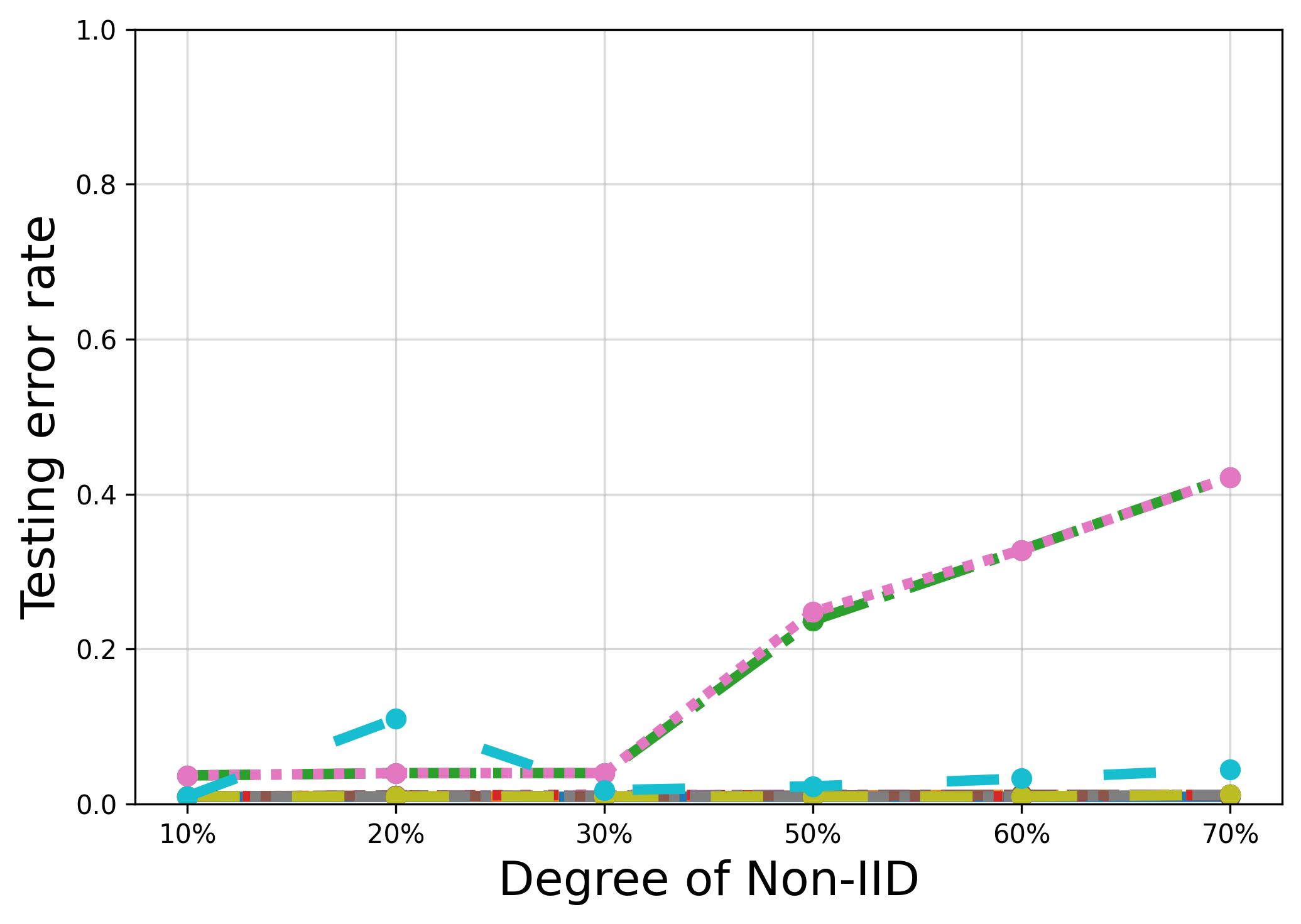}
		\small\captionof*{figure}{(a) No attack} 
	\end{minipage}
	\hfill 
	\begin{minipage}[b]{0.24\textwidth}
		\centering
		\includegraphics[width=\textwidth]{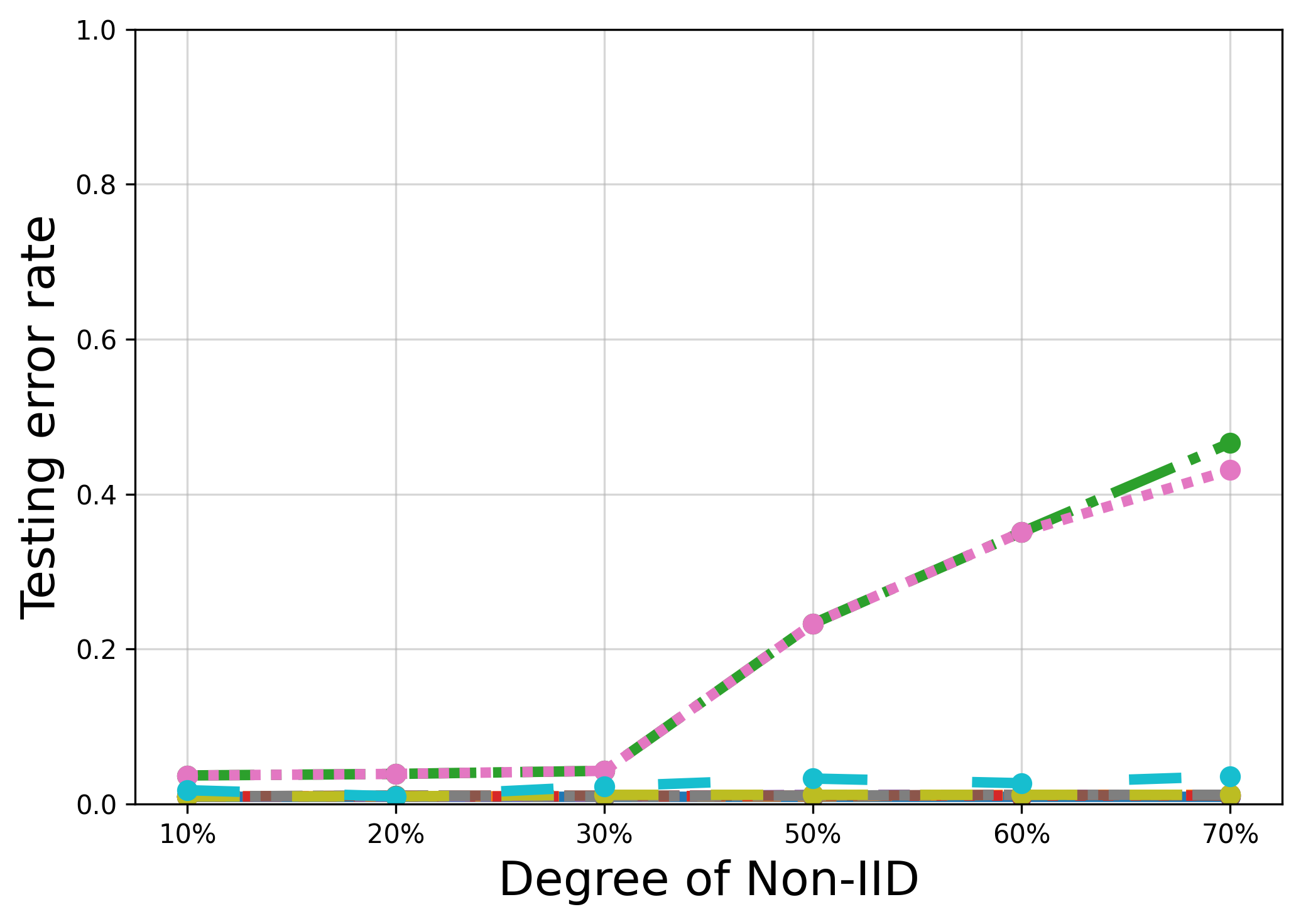}
		\small\captionof*{figure}{(b) LF attack}
	\end{minipage}
	\hfill
	\begin{minipage}[b]{0.24\textwidth}
		\centering
		\includegraphics[width=\textwidth]{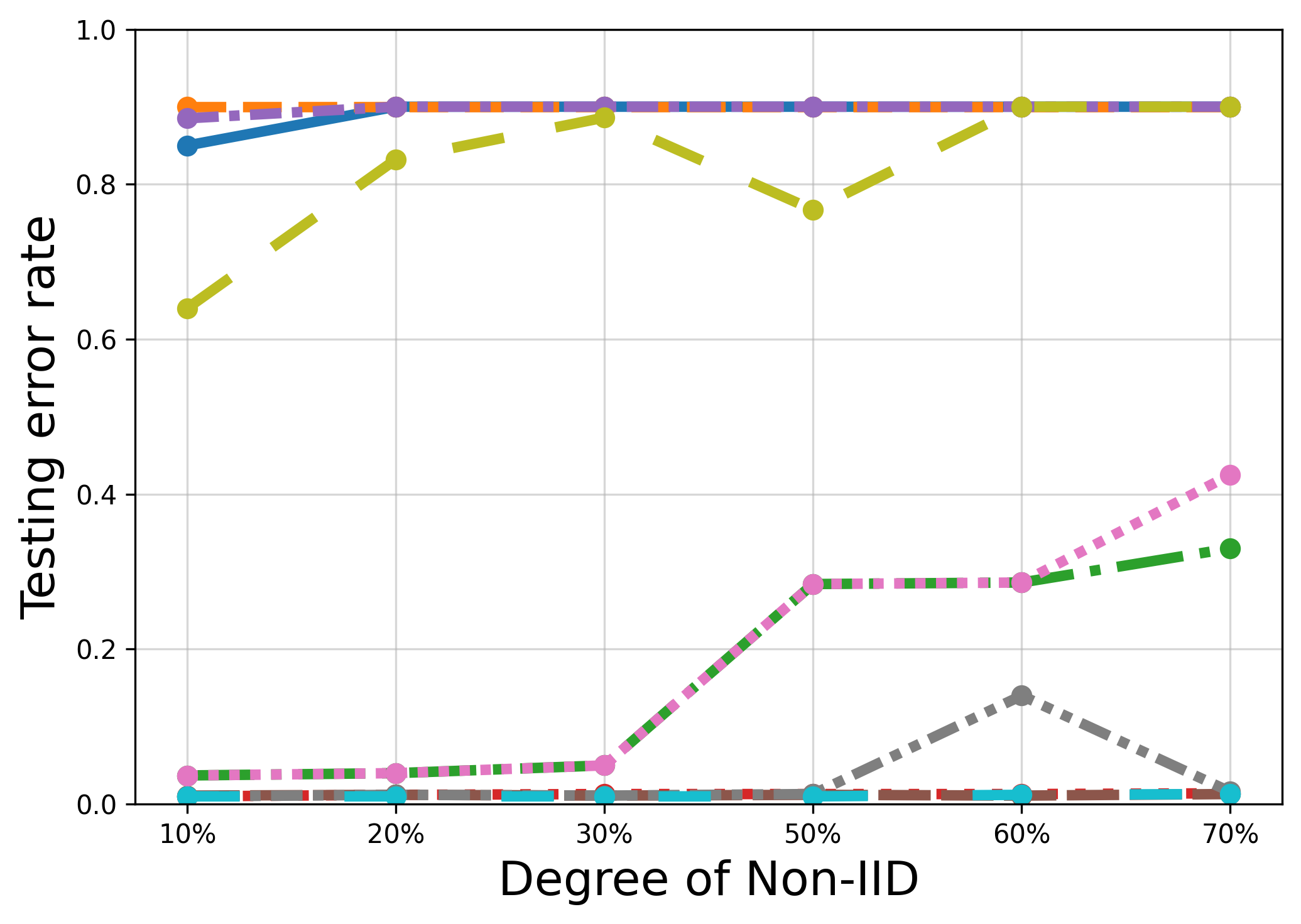}
		\small\captionof*{figure}{(c) Gaussian attack}
	\end{minipage}
	\hfill
	\begin{minipage}[b]{0.24\textwidth}
		\centering
		\includegraphics[width=\textwidth]{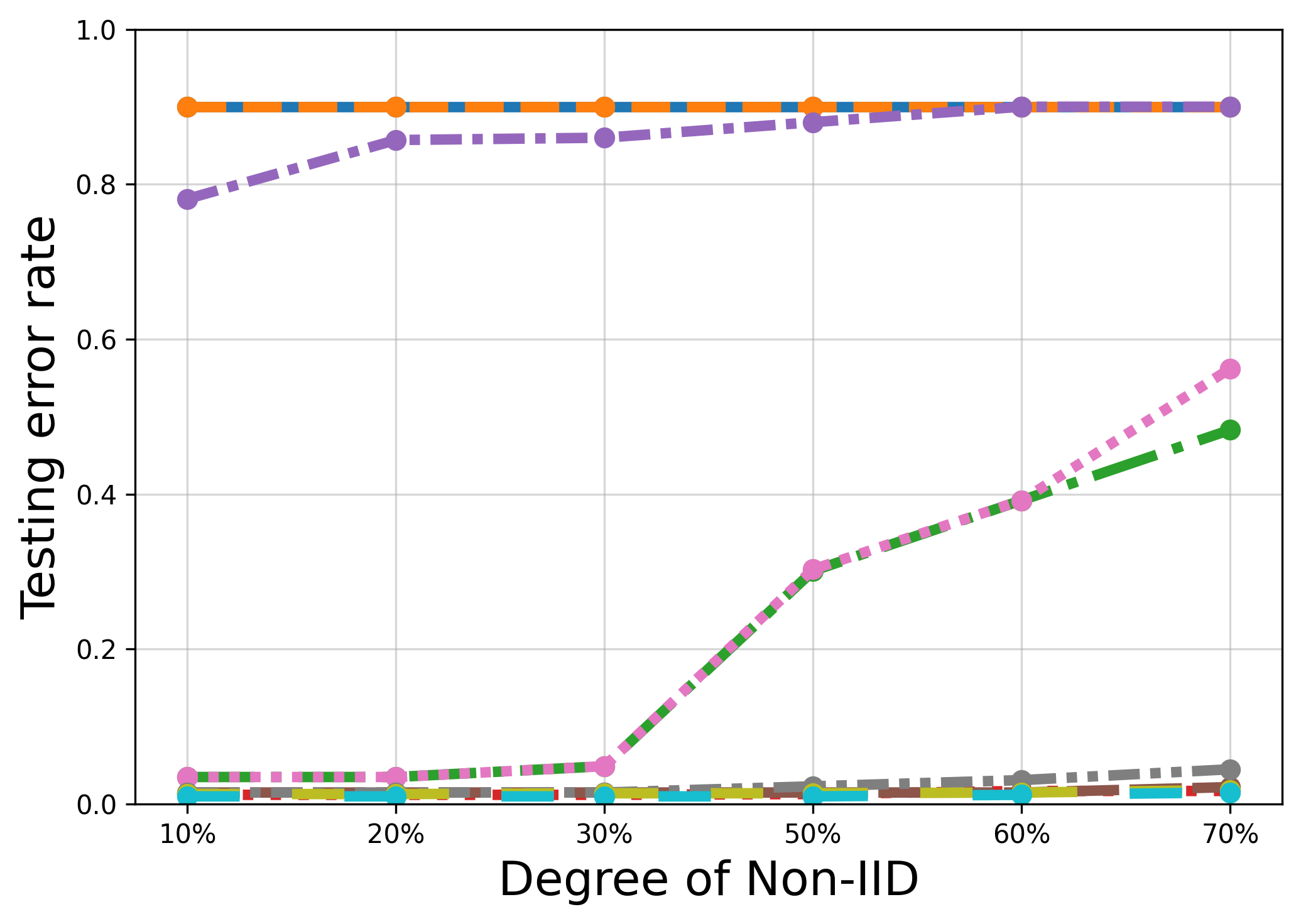}
		\small\captionof*{figure}{(d) Trim attack}
	\end{minipage}
	
	\vspace{0.15cm} 
	
	\begin{minipage}[b]{0.24\textwidth}
		\centering
		\includegraphics[width=\textwidth]{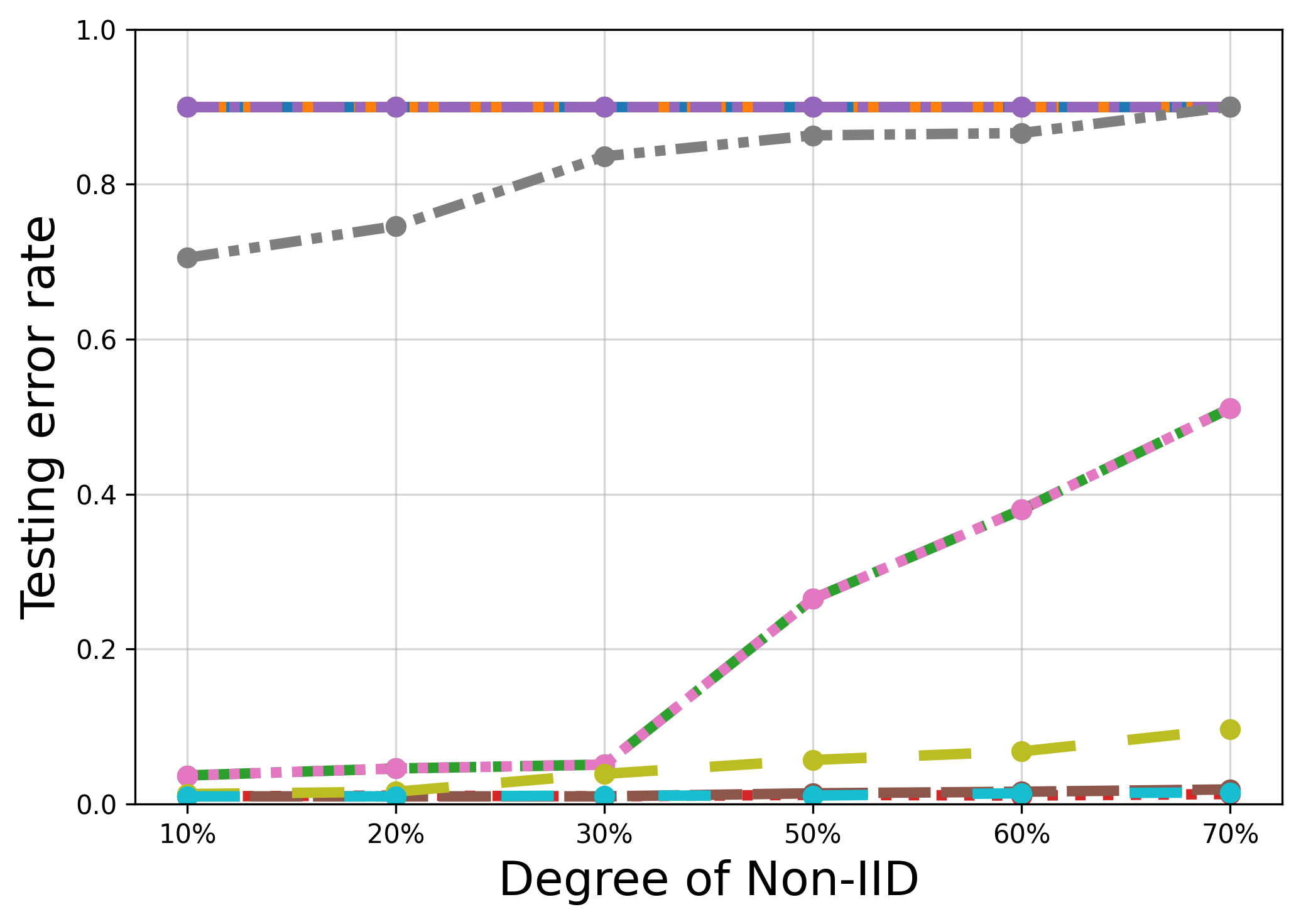}
		\small\captionof*{figure}{(e) Krum attack}
	\end{minipage}
	\hfill
	\begin{minipage}[b]{0.24\textwidth}
		\centering
		\includegraphics[width=\textwidth]{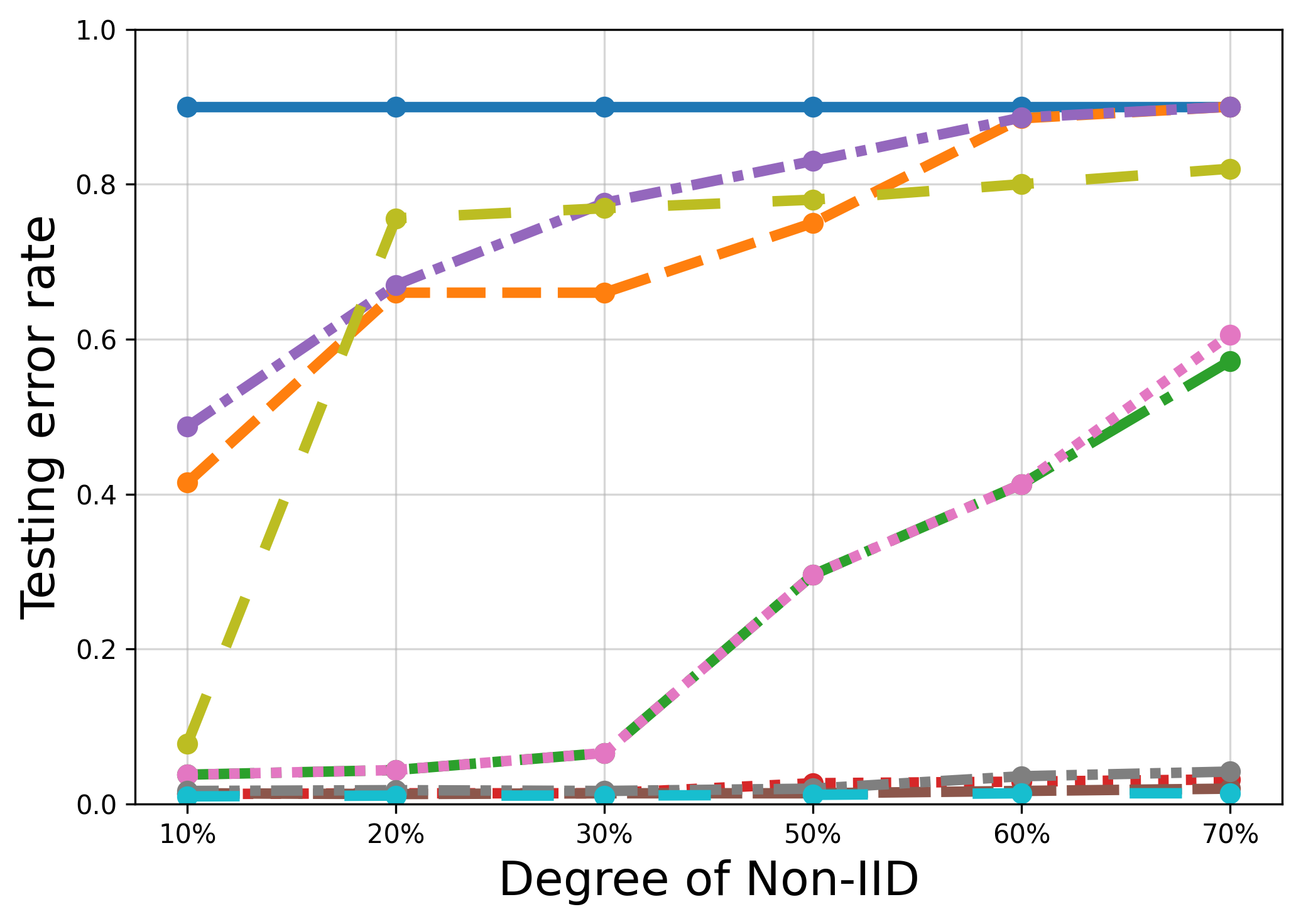}
		\small\captionof*{figure}{(f) Min-Max attack}
	\end{minipage}
	\hfill
	\begin{minipage}[b]{0.24\textwidth}
		\centering
		\includegraphics[width=\textwidth]{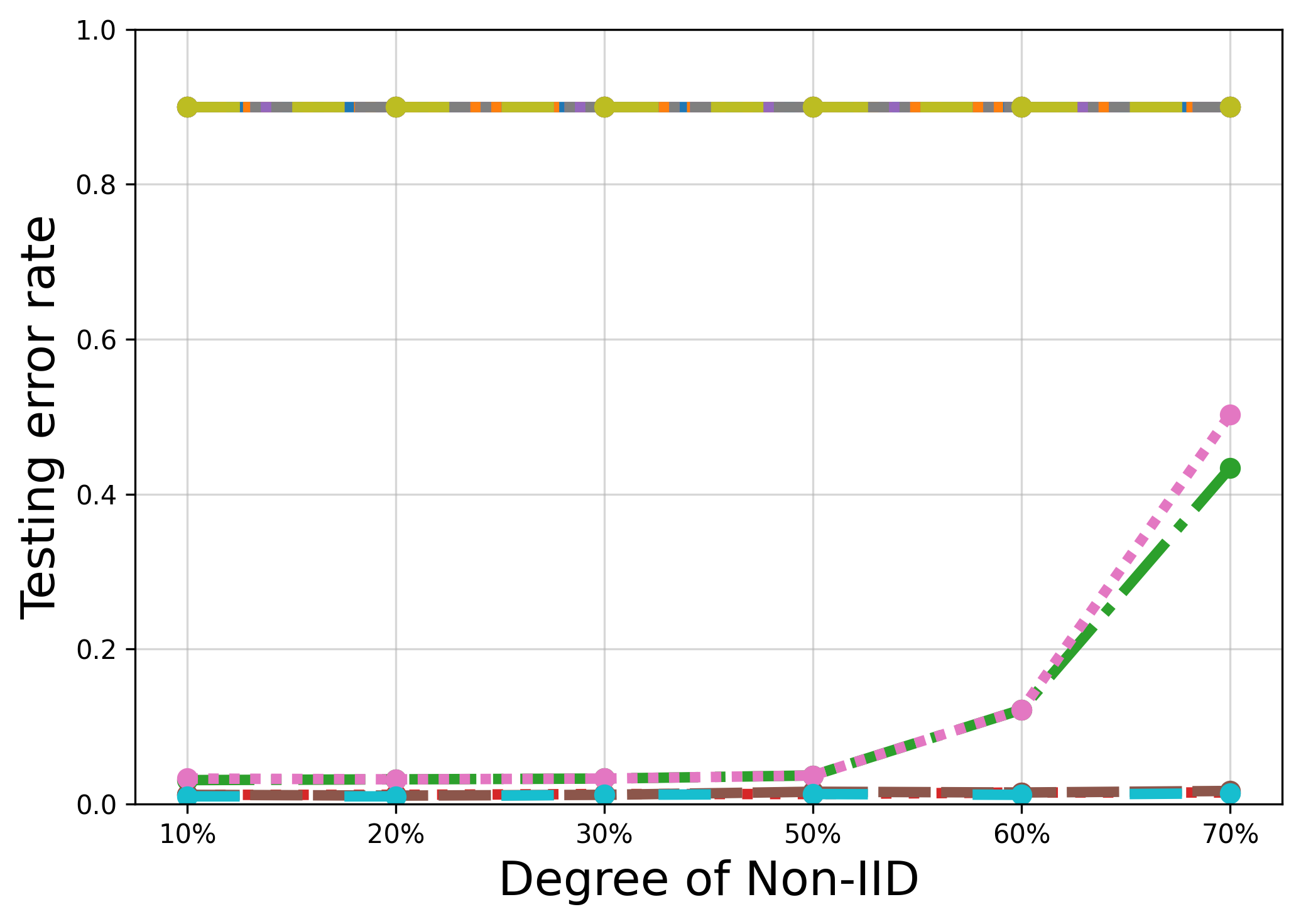}
		\small\captionof*{figure}{(g) Scaling attack}
	\end{minipage}
	\hfill
	\begin{minipage}[b]{0.24\textwidth}
		\centering
		\includegraphics[width=\textwidth]{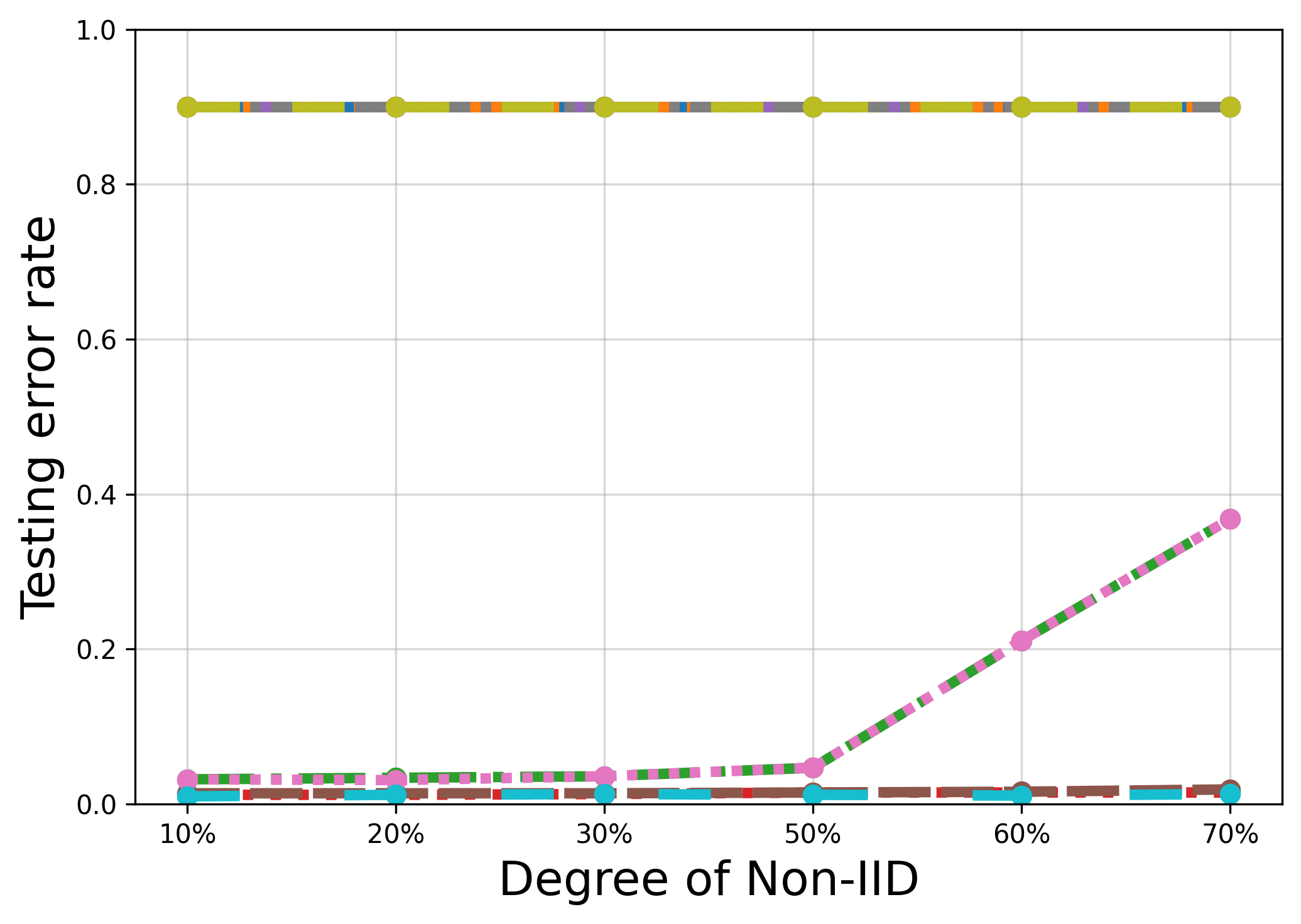}
		\small\captionof*{figure}{(h) Sybil attack}
	\end{minipage}
	
	\caption{Impact of degree of non-iid on MNIST.}
	\label{fig:experimental_results_NIID}
\end{figure*}

\begin{figure*}[htbp]
	\centering
	
	\includegraphics[width=1\textwidth]{Images/Experimental_Results/legend_only.png}
	
	\vspace{0.2cm} 
	
	\begin{minipage}[b]{0.24\textwidth}
		\centering
		\includegraphics[width=\textwidth]{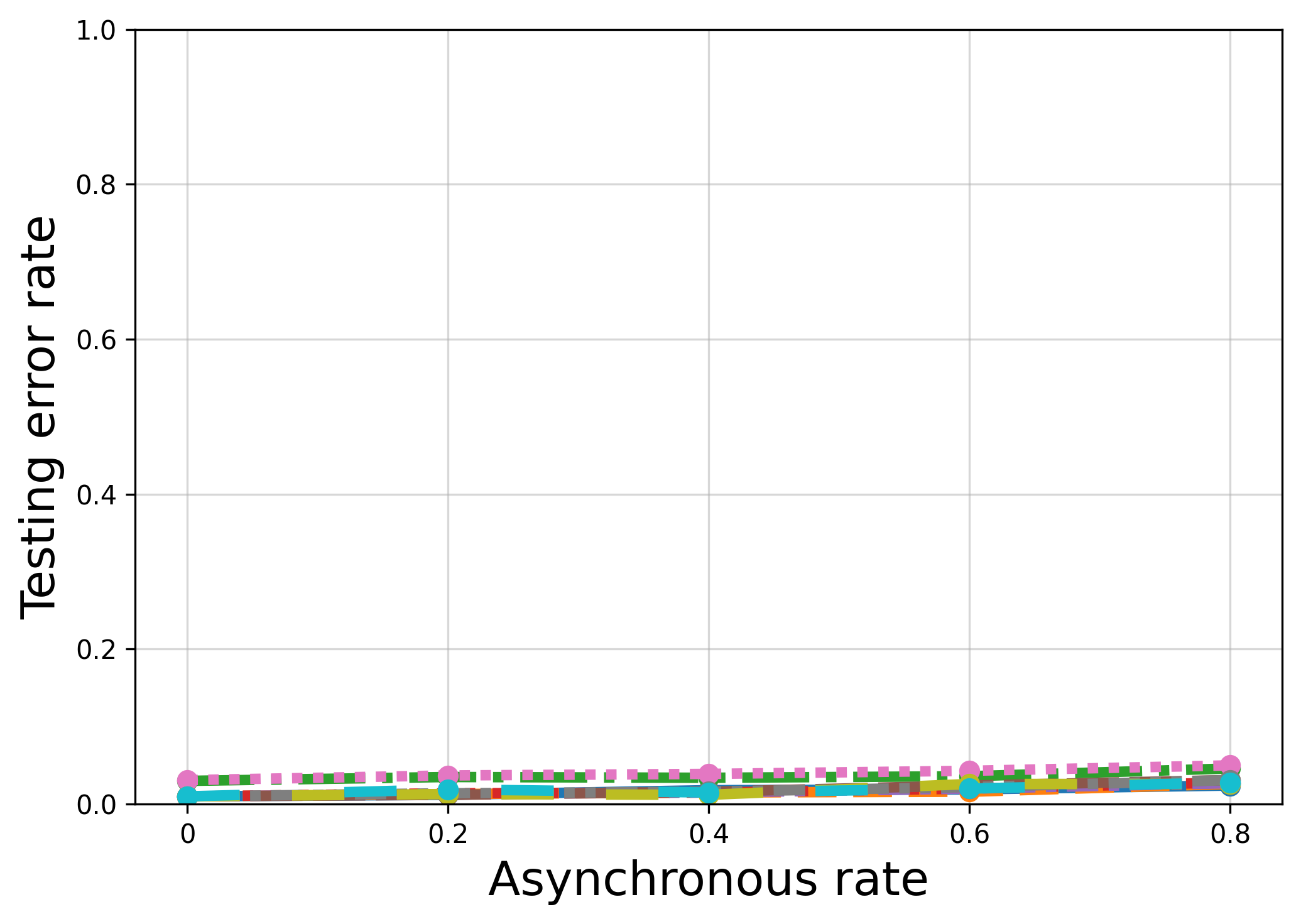}
		\small\captionof*{figure}{(a) No attack} 
	\end{minipage}
	\hfill 
	\begin{minipage}[b]{0.24\textwidth}
		\centering
		\includegraphics[width=\textwidth]{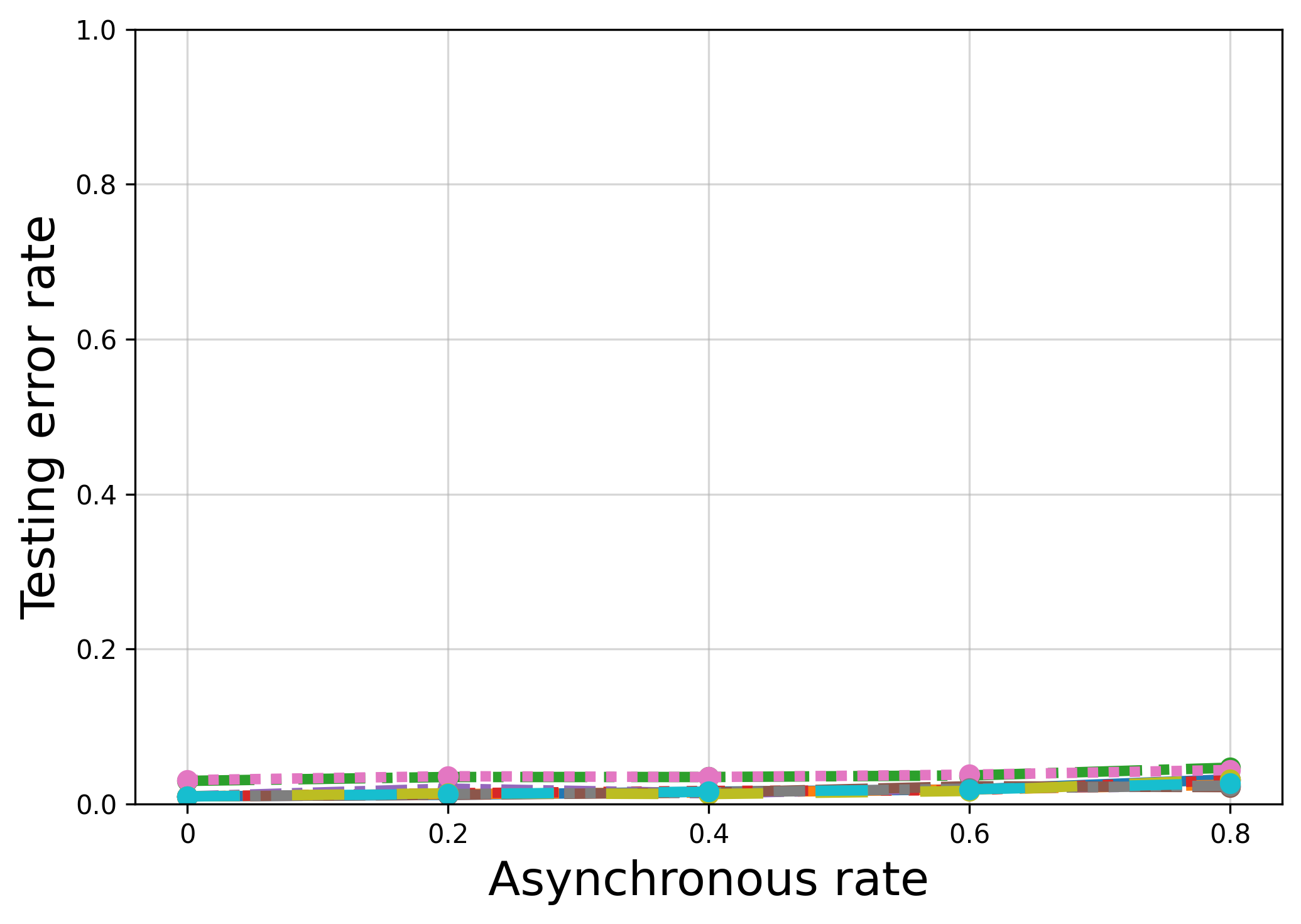}
		\small\captionof*{figure}{(b) LF attack}
	\end{minipage}
	\hfill
	\begin{minipage}[b]{0.24\textwidth}
		\centering
		\includegraphics[width=\textwidth]{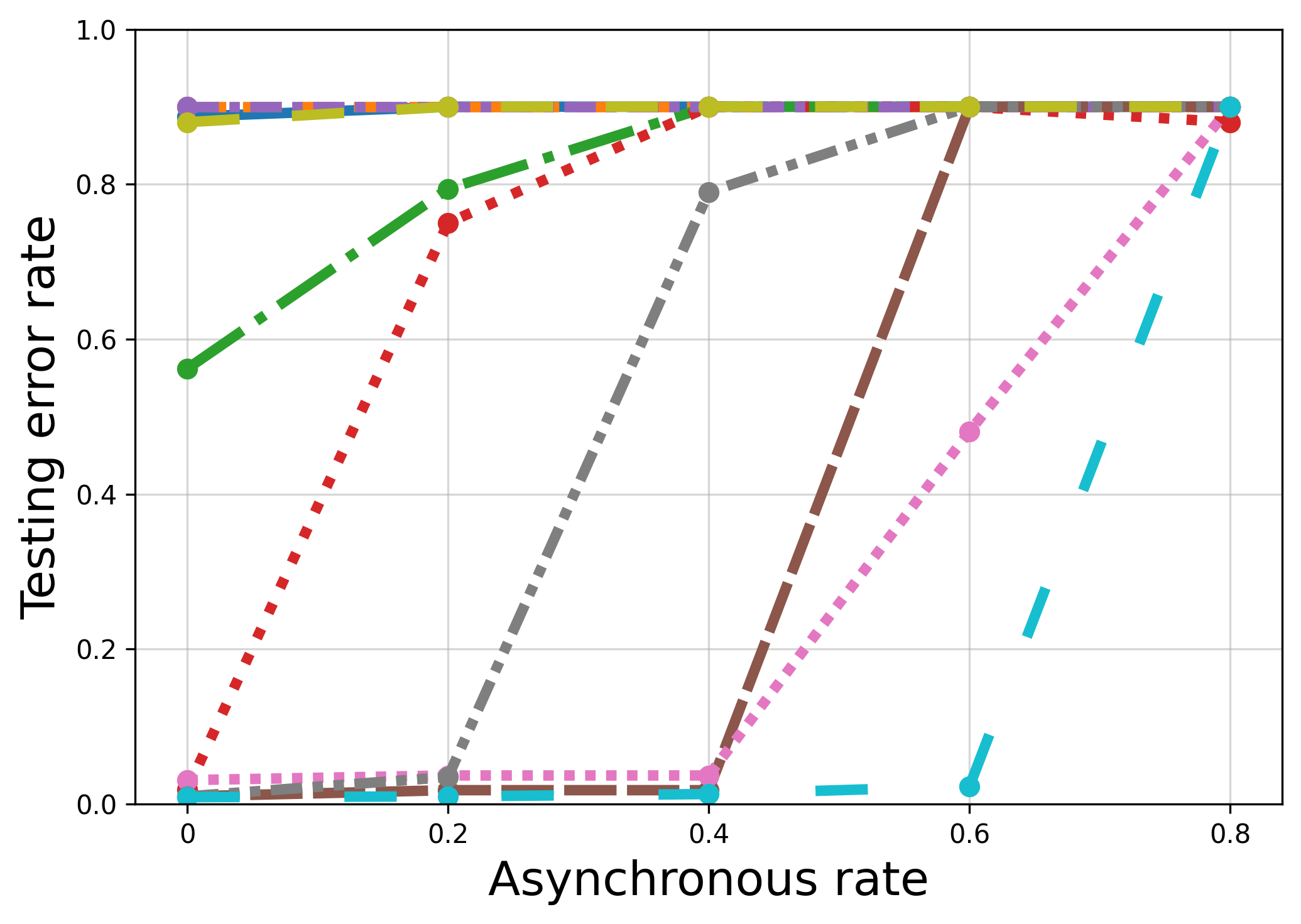}
		\small\captionof*{figure}{(c) Gaussian attack}
	\end{minipage}
	\hfill
	\begin{minipage}[b]{0.24\textwidth}
		\centering
		\includegraphics[width=\textwidth]{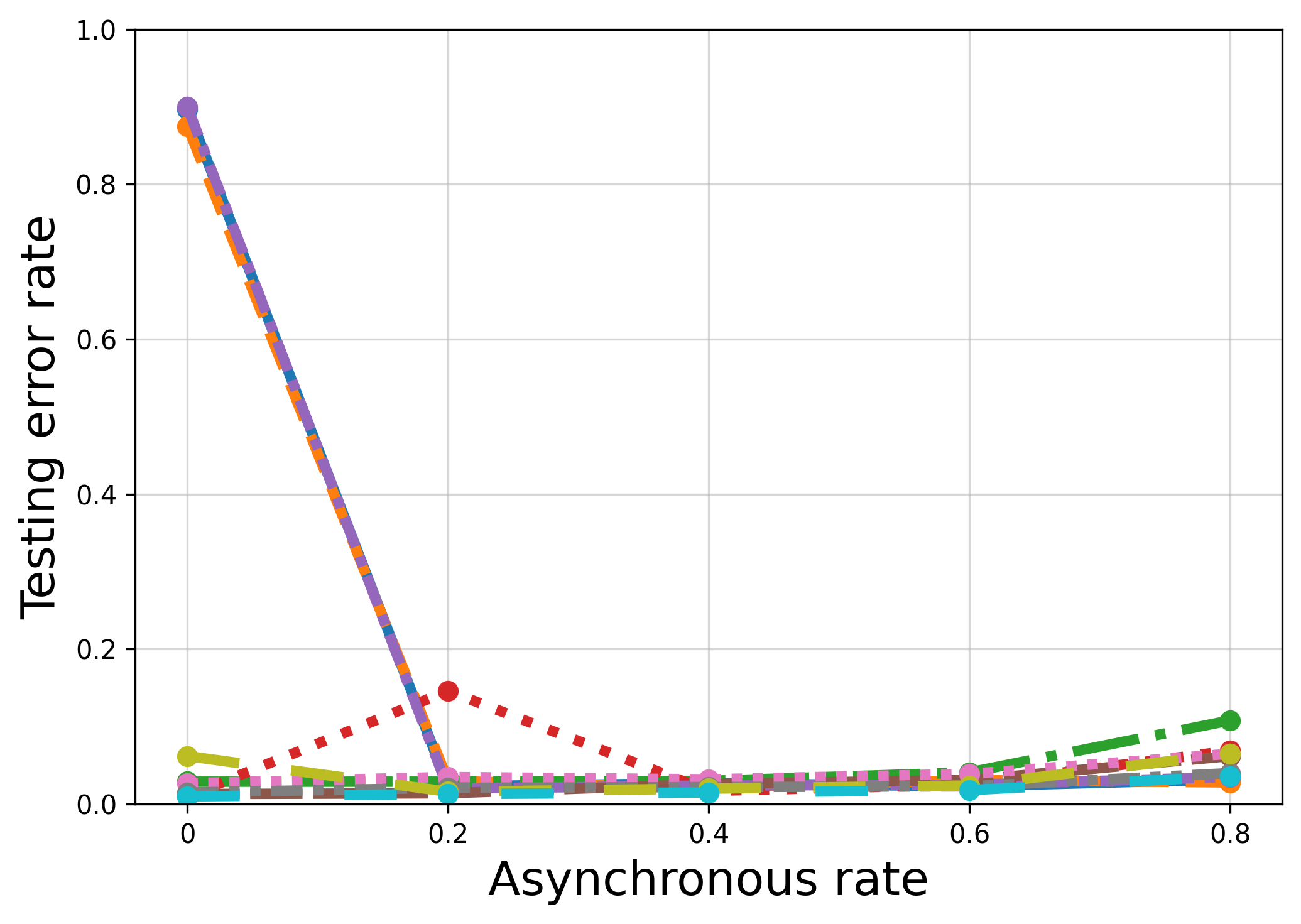}
		\small\captionof*{figure}{(d) Trim attack}
	\end{minipage}
	
	\vspace{0.15cm} 
	
	\begin{minipage}[b]{0.24\textwidth}
		\centering
		\includegraphics[width=\textwidth]{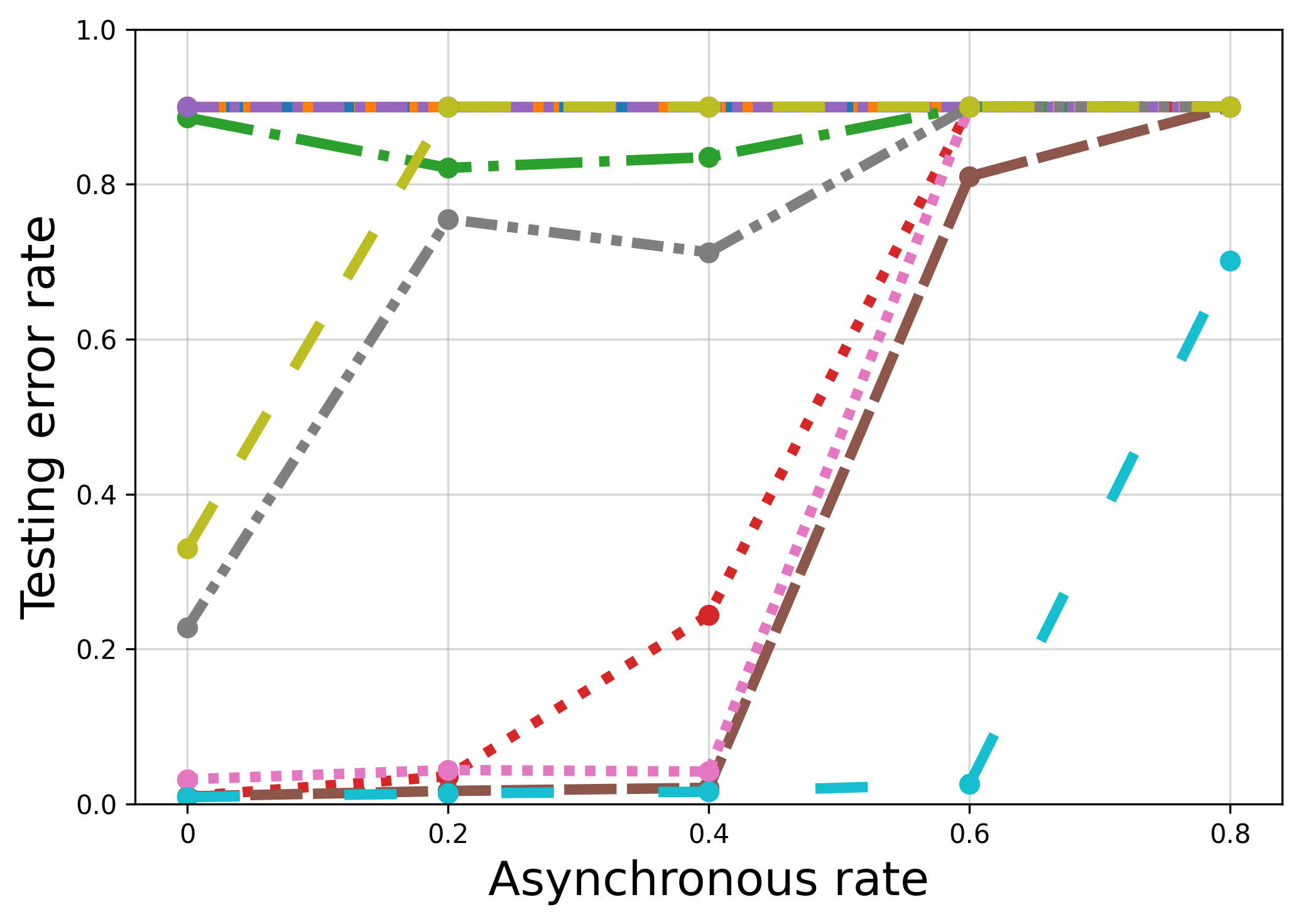}
		\small\captionof*{figure}{(e) Krum attack}
	\end{minipage}
	\hfill
	\begin{minipage}[b]{0.24\textwidth}
		\centering
		\includegraphics[width=\textwidth]{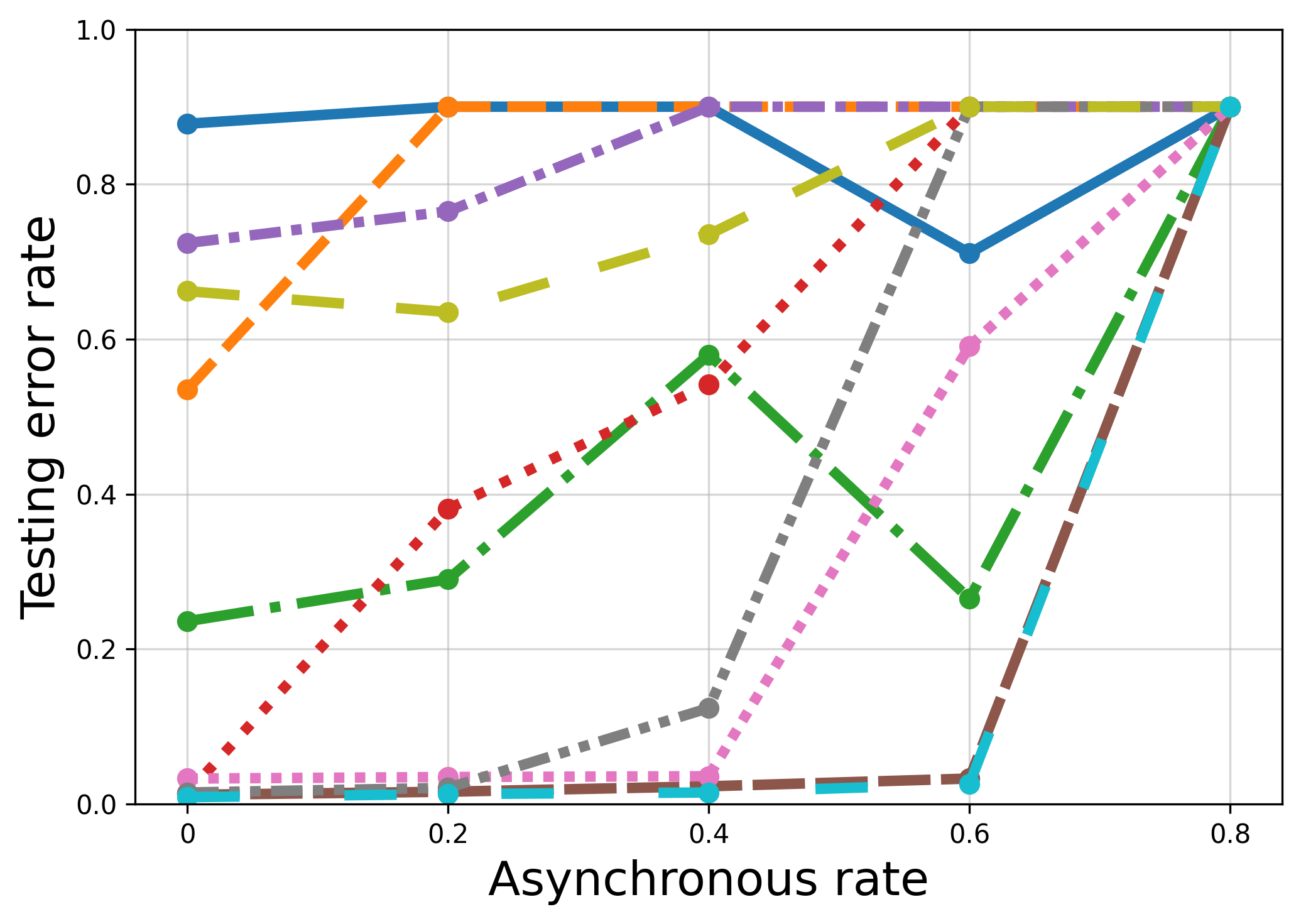}
		\small\captionof*{figure}{(f) Min-Max attack}
	\end{minipage}
	\hfill
	\begin{minipage}[b]{0.24\textwidth}
		\centering
		\includegraphics[width=\textwidth]{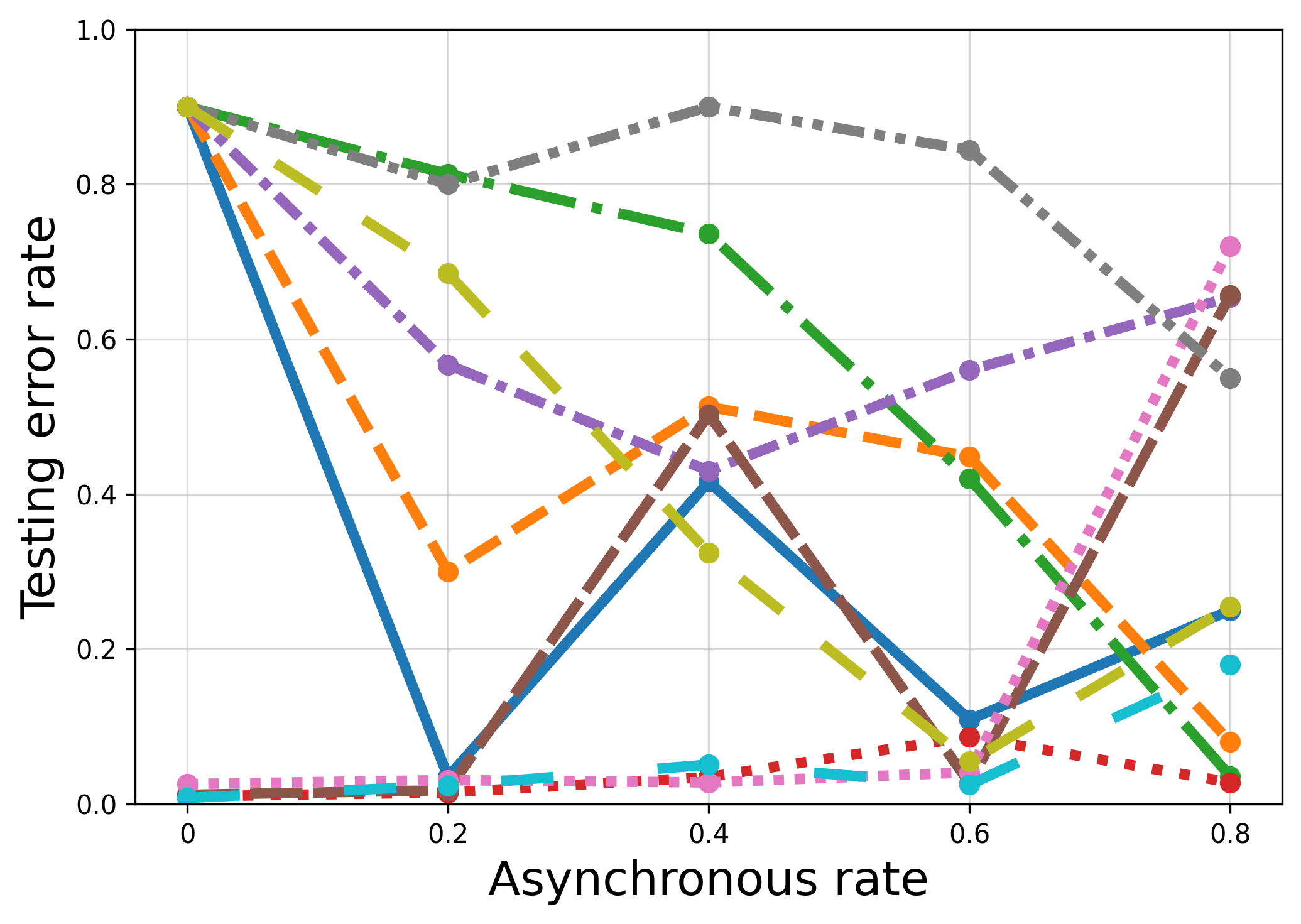}
		\small\captionof*{figure}{(g) Scaling attack}
	\end{minipage}
	\hfill
	\begin{minipage}[b]{0.24\textwidth}
		\centering
		\includegraphics[width=\textwidth]{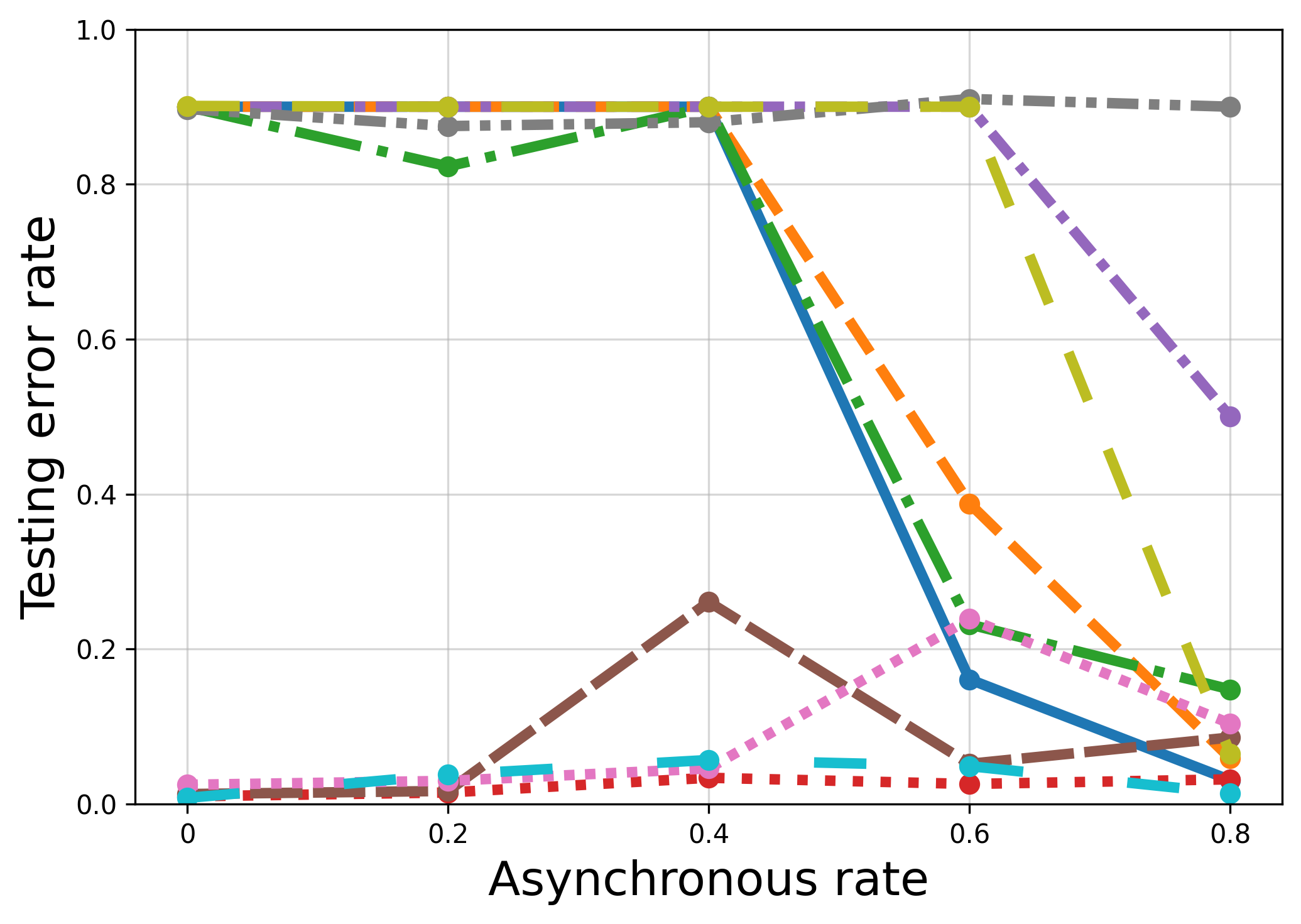}
		\small\captionof*{figure}{(h) Sybil attack}
	\end{minipage}
	
	\caption{Impact of asynchronous rate on MNIST}
	\label{fig:Impact_of_asynchronous_rate}
\end{figure*}

\textbf{Impact of the proportion of malicious clients}:
Fig. \ref{Malicious_num} illustrates the impact of poisoning attacks on the performance of various FL aggregation algorithms using the MNIST dataset, where the proportion of malicious clients among the total number of clients increases from 0\% to 50\%. The computational proportion of malicious clients is denoted as $\frac{m}{n}$, where $m$ represents the number of malicious clients and $n$ denotes the total number of clients. Results show that AdaBFL defends against poisoning attacks when $\frac{m}{n}\le 40\%$. When $ 40 \le \frac{m}{n}\le 50\%$, AdaBFL still resists LF, Gaussian, Trim, Min-Max, and Scaling attacks. Compared to other FL aggregation algorithms, AdaBFL exhibits stronger tolerance for malicious client proportions. For instance, in the Scaling attack, other FL aggregation algorithms tolerate only 20\% malicious clients, whereas AdaBFL tolerates up to 40\%. The above analysis demonstrates that AdaBFL surpasses existing Byzantine-fault-tolerant robust aggregation algorithms in its tolerance for malicious clients.

\textbf{Impact of degree of non-iid:}
Unlike the traditional machine learning assumption of independent and identically distributed (IID) data, FL often deals with non-iid data. When training data exhibits significant differences, attackers can exploit these disparities to create seemingly benign yet malicious models, thereby compromising the global model. Fig. \ref{fig:experimental_results_NIID} explores the effectiveness of various aggregation algorithms under poisoning attacks as the degree of non-independence increases from 10\% to 70\%. Results demonstrate that despite substantial client-side data divergence, our proposed AdaBFL effectively mitigates the impact of data variability and robustly resists various poisoning attacks.

\textbf{Impact of asynchronous rate}: Fig. \ref{fig:Impact_of_asynchronous_rate}  illustrates the performance metrics of AdaBFL and the comparison scheme under varying asynchrony rates $(0-0.8)$ between clients and the central server. Higher asynchrony rates indicate more unstable updates. Results demonstrate that across various attack scenarios, AdaBFL exhibits the strongest resistance to interference and correlation attacks compared to other aggregation algorithms.

\begin{figure*}[htbp]
	\centering
	
	\includegraphics[width=1\textwidth]{Images/Experimental_Results/legend_only.png}
	
	\vspace{0.2cm} 
	
	\begin{minipage}[b]{0.24\textwidth}
		\centering
		\includegraphics[width=\textwidth]{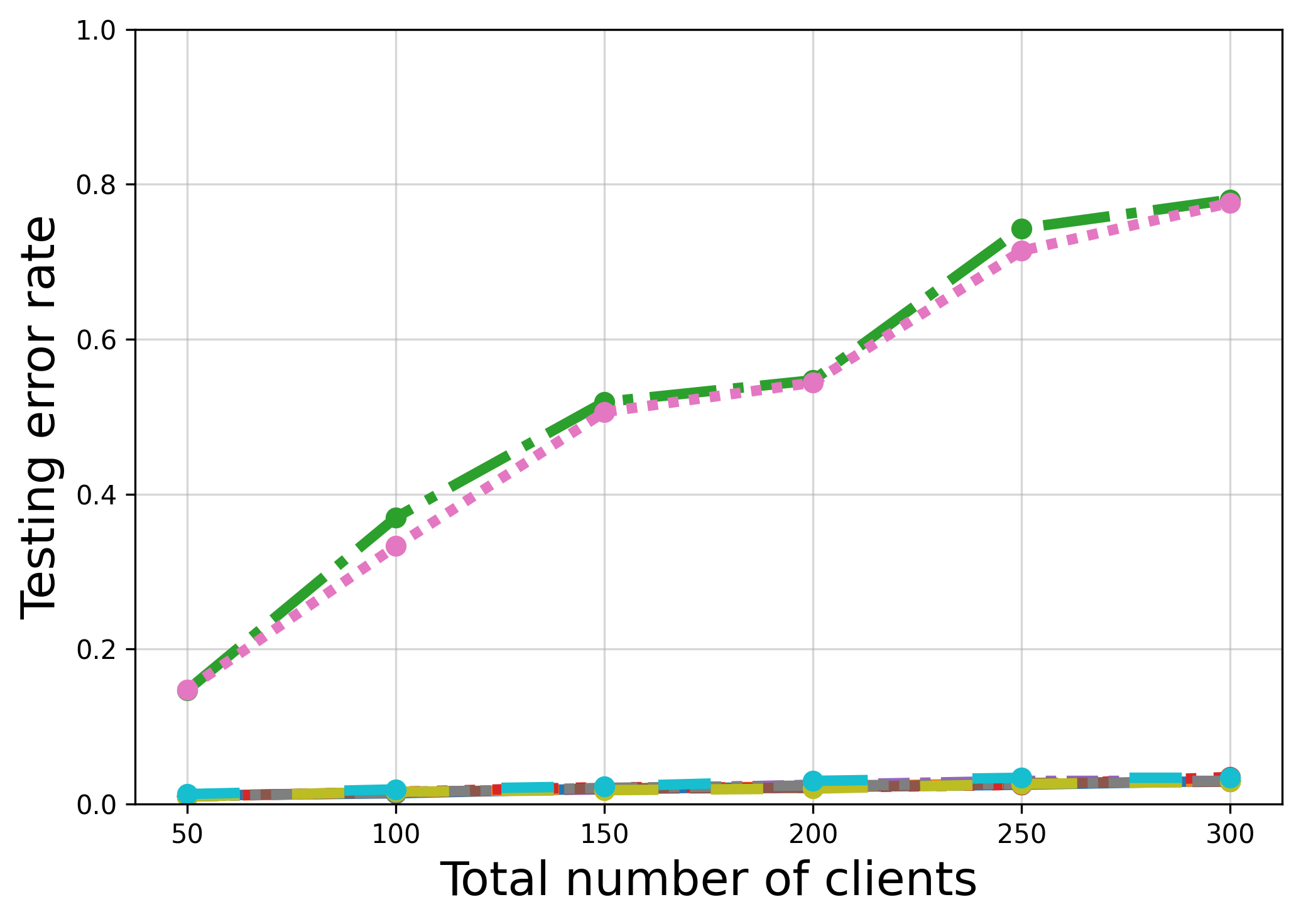}
		\small\captionof*{figure}{(a) No attack} 
	\end{minipage}
	\hfill 
	\begin{minipage}[b]{0.24\textwidth}
		\centering
		\includegraphics[width=\textwidth]{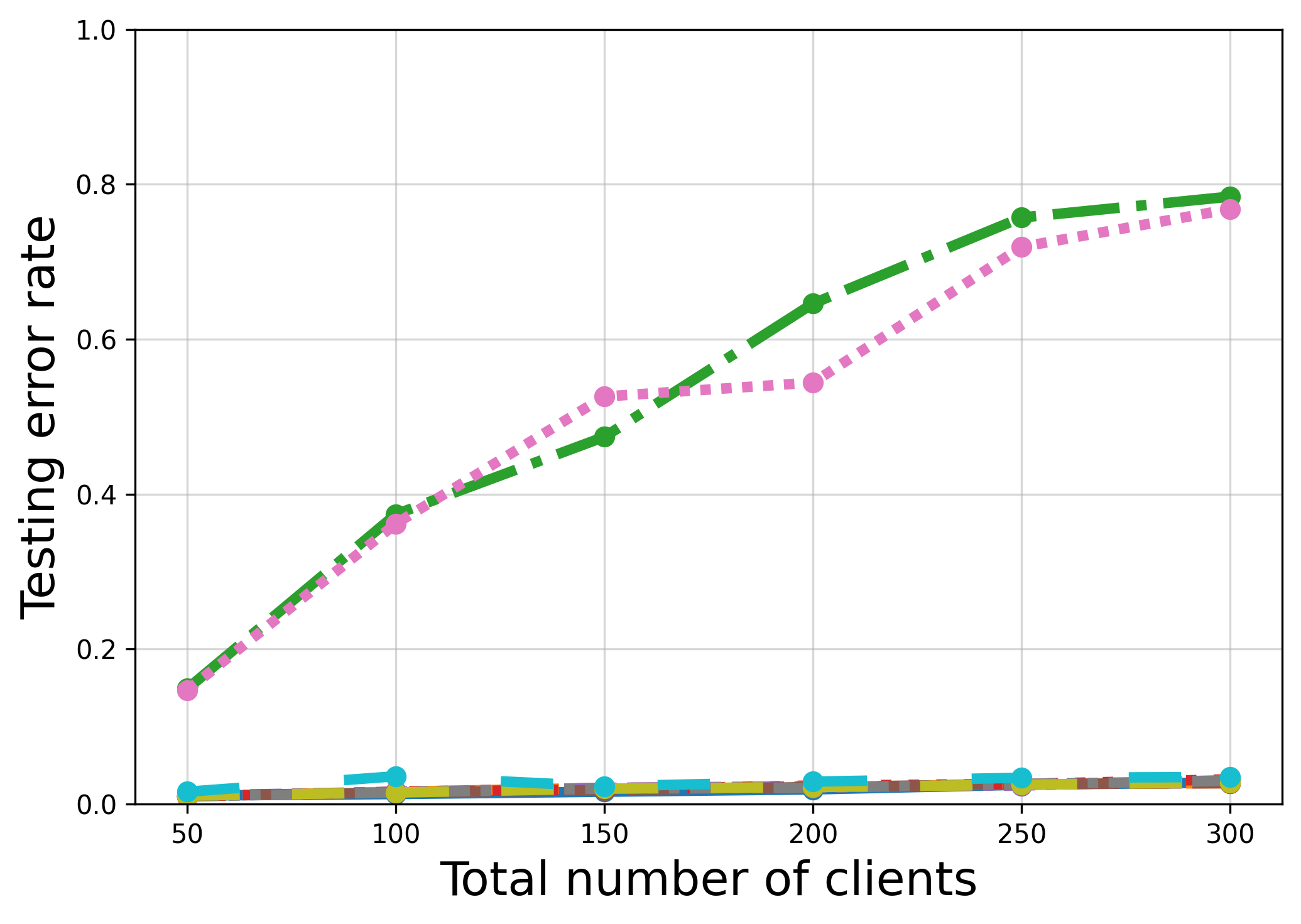}
		\small\captionof*{figure}{(b) LF attack}
	\end{minipage}
	\hfill
	\begin{minipage}[b]{0.24\textwidth}
		\centering
		\includegraphics[width=\textwidth]{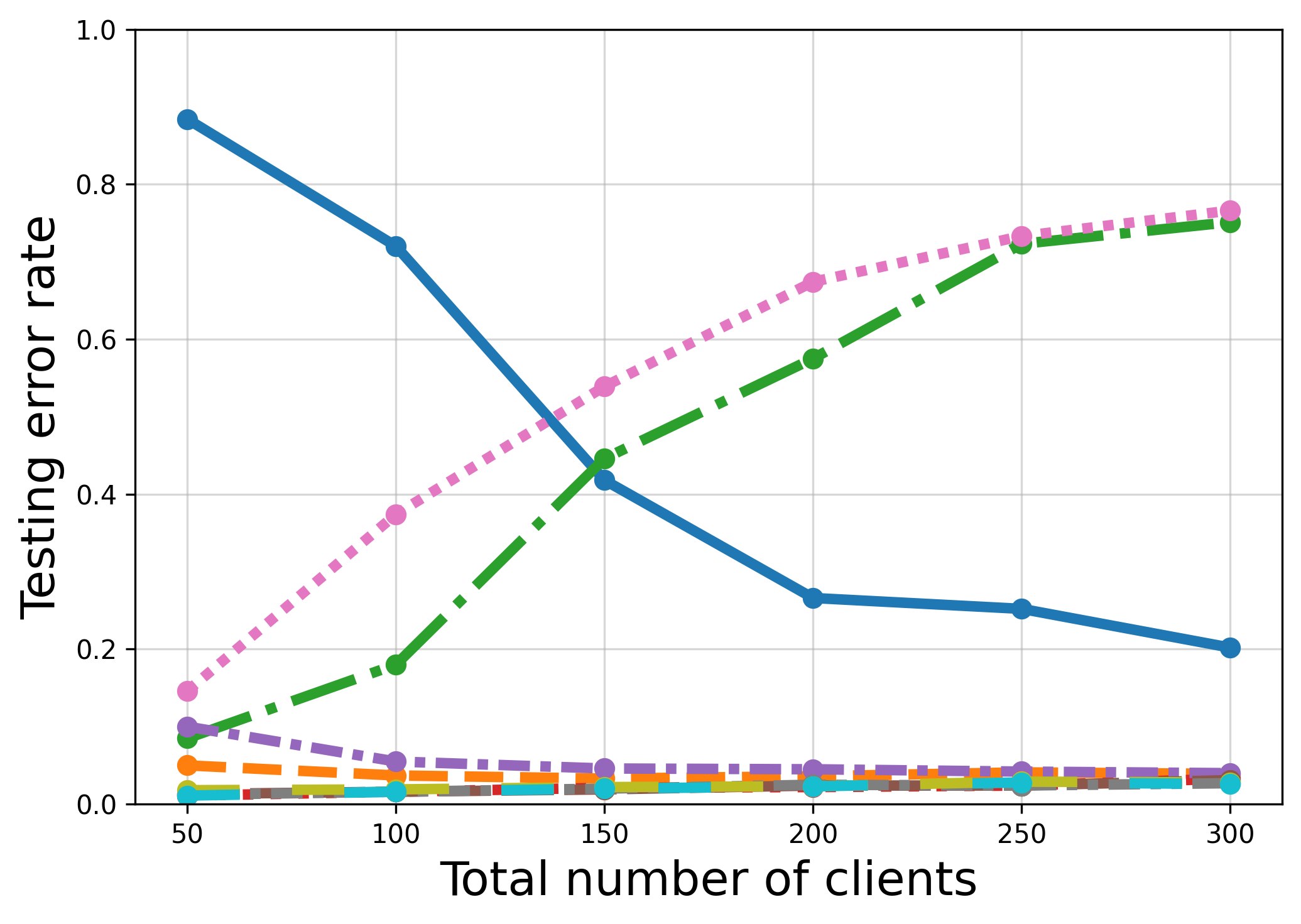}
		\small\captionof*{figure}{(c) Gaussian attack}
	\end{minipage}
	\hfill
	\begin{minipage}[b]{0.24\textwidth}
		\centering
		\includegraphics[width=\textwidth]{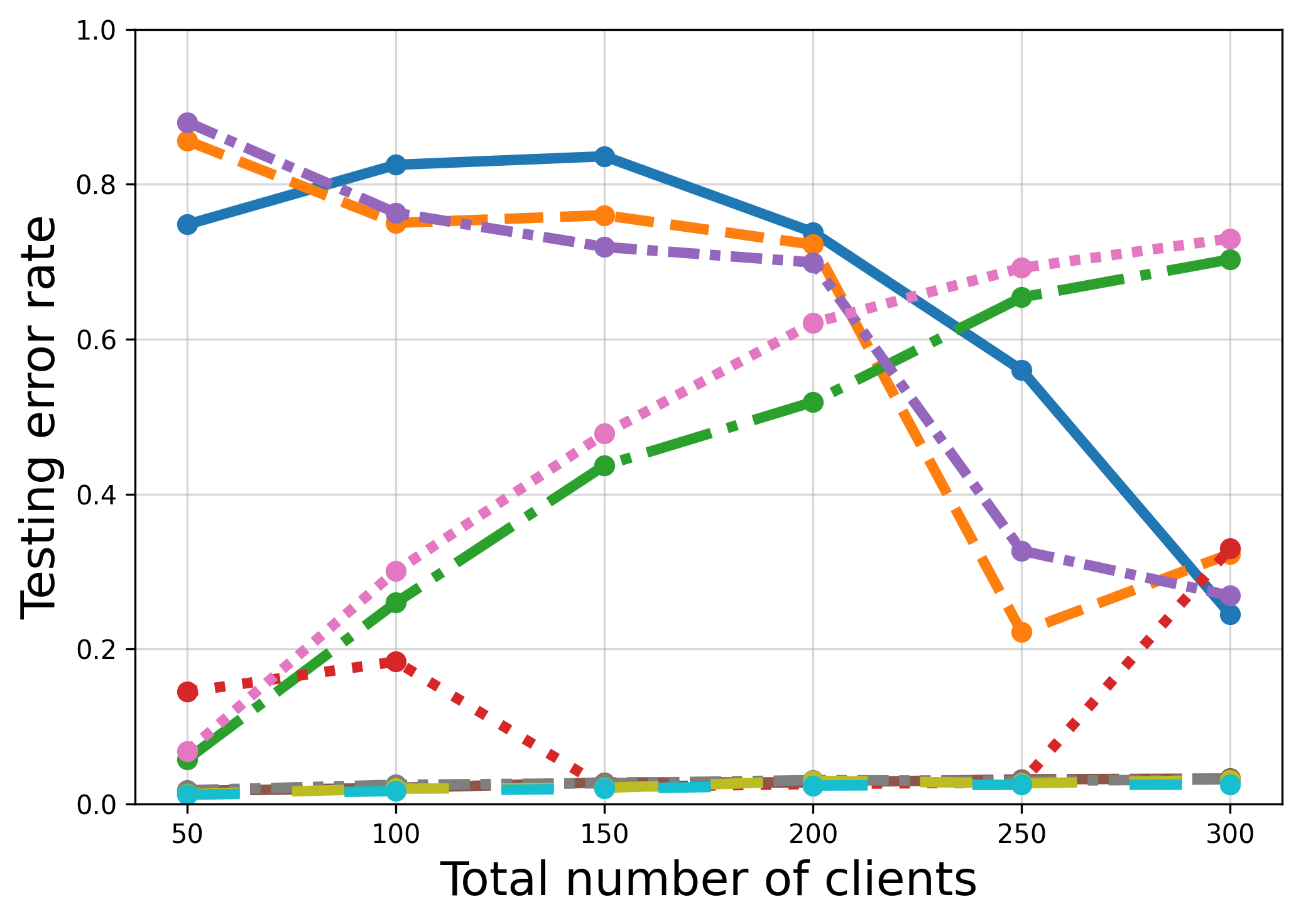}
		\small\captionof*{figure}{(d) Trim attack}
	\end{minipage}
	
	\vspace{0.15cm} 
	
	\begin{minipage}[b]{0.24\textwidth}
		\centering
		\includegraphics[width=\textwidth]{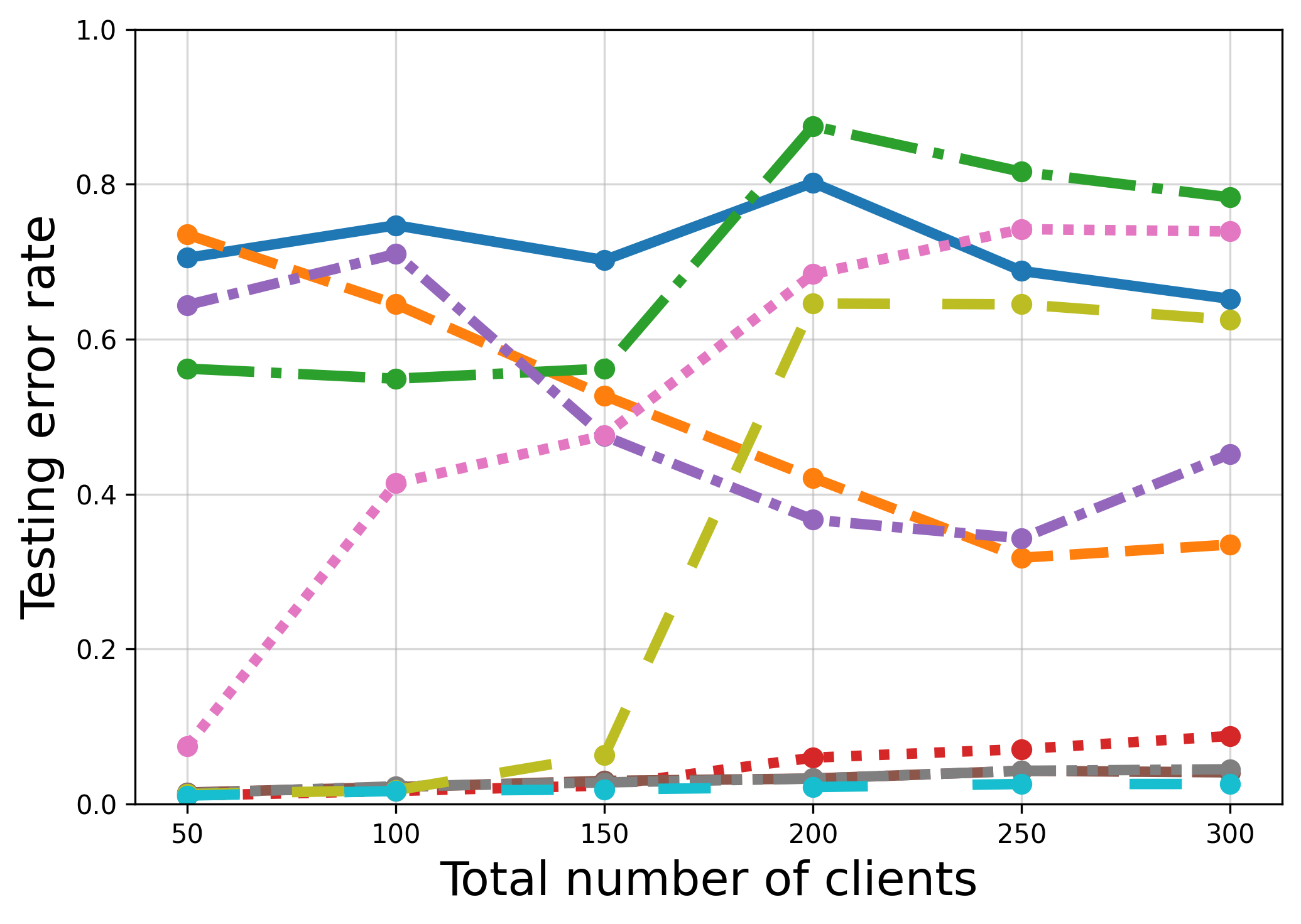}
		\small\captionof*{figure}{(e) Krum attack}
	\end{minipage}
	\hfill
	\begin{minipage}[b]{0.24\textwidth}
		\centering
		\includegraphics[width=\textwidth]{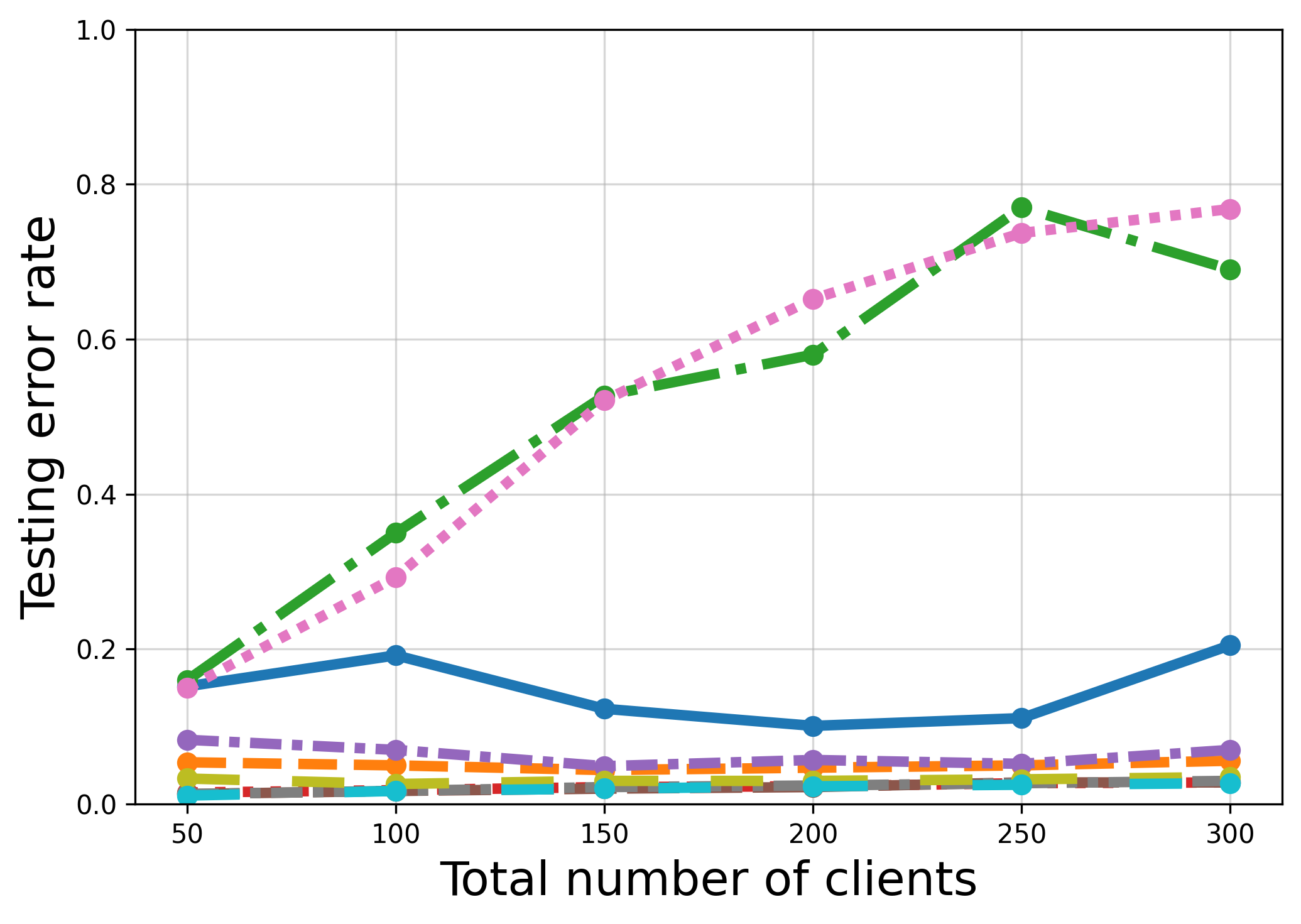}
		\small\captionof*{figure}{(f) Min-Max attack}
	\end{minipage}
	\hfill
	\begin{minipage}[b]{0.24\textwidth}
		\centering
		\includegraphics[width=\textwidth]{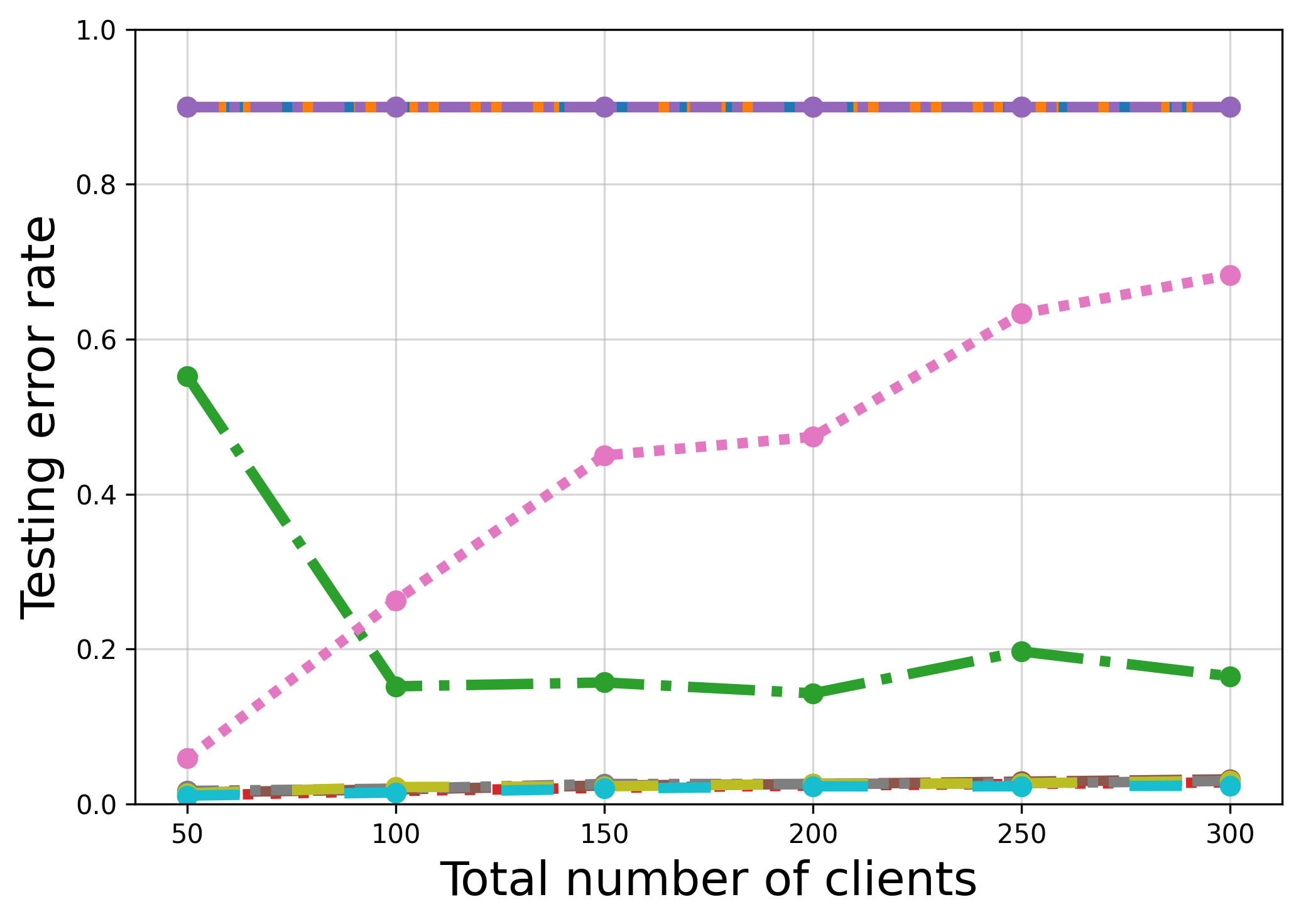}
		\small\captionof*{figure}{(g) Scaling attack}
	\end{minipage}
	\hfill
	\begin{minipage}[b]{0.24\textwidth}
		\centering
		\includegraphics[width=\textwidth]{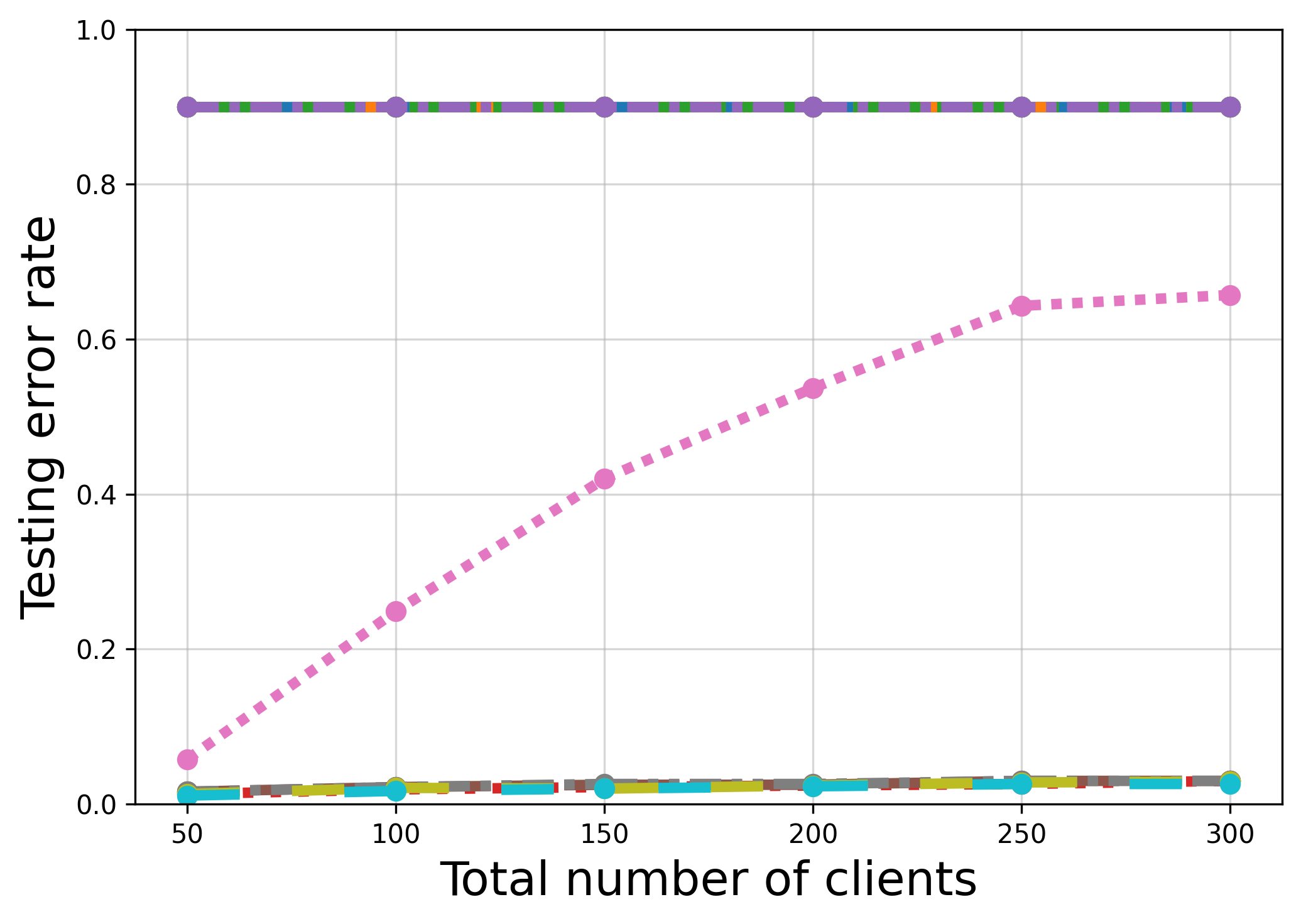}
		\small\captionof*{figure}{(h) Sybil attack}
	\end{minipage}
	
	\caption{Impact of total number of clients on MNIST}
	\label{fig:experimental_results_AII_number}
\end{figure*}

\begin{figure*}[htbp]
	\centering
	
	\includegraphics[width=0.5\textwidth]{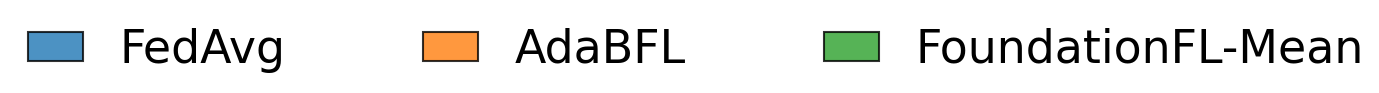}
	
	\vspace{0.1cm} 
	
	\begin{minipage}[b]{0.24\textwidth}
		\centering
		\includegraphics[width=\textwidth]{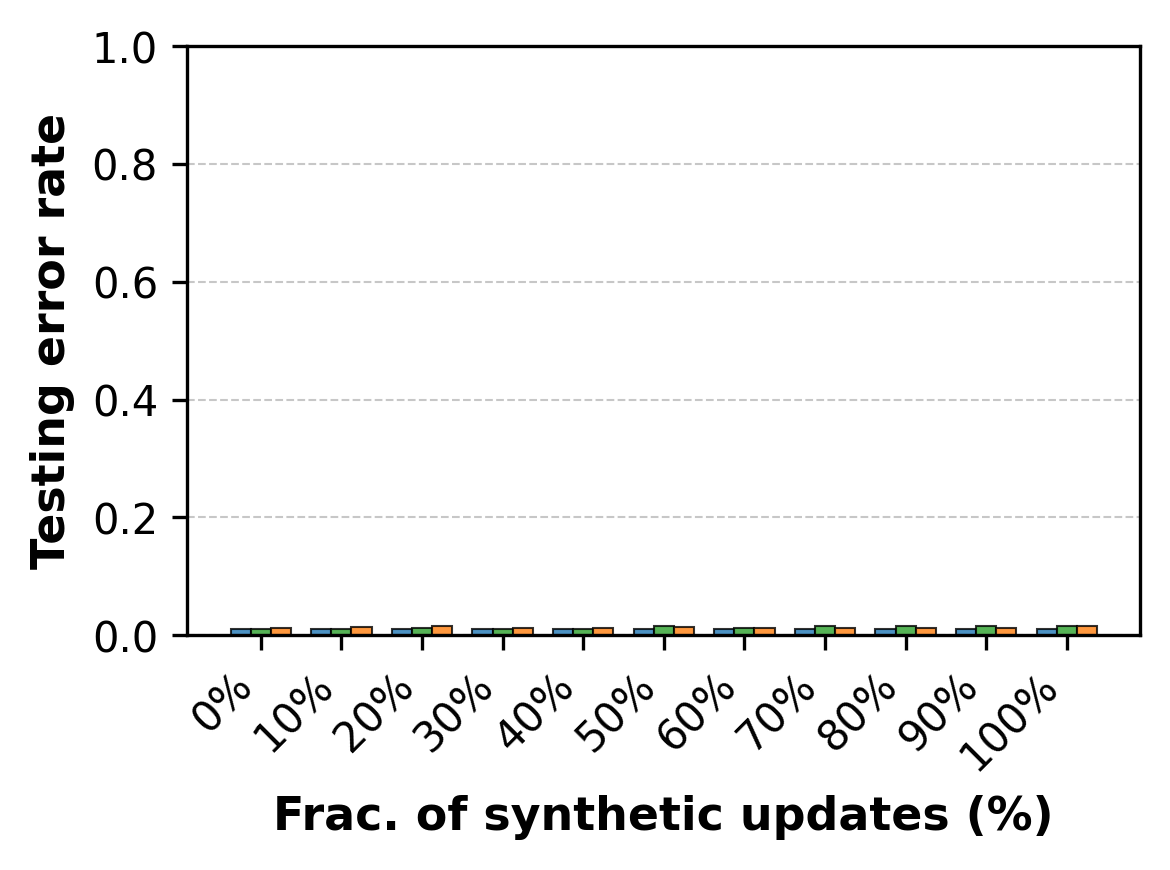}
		\small\captionof*{figure}{(a) No attack} 
	\end{minipage}
	\hfill 
	\begin{minipage}[b]{0.24\textwidth}
		\centering
		\includegraphics[width=\textwidth]{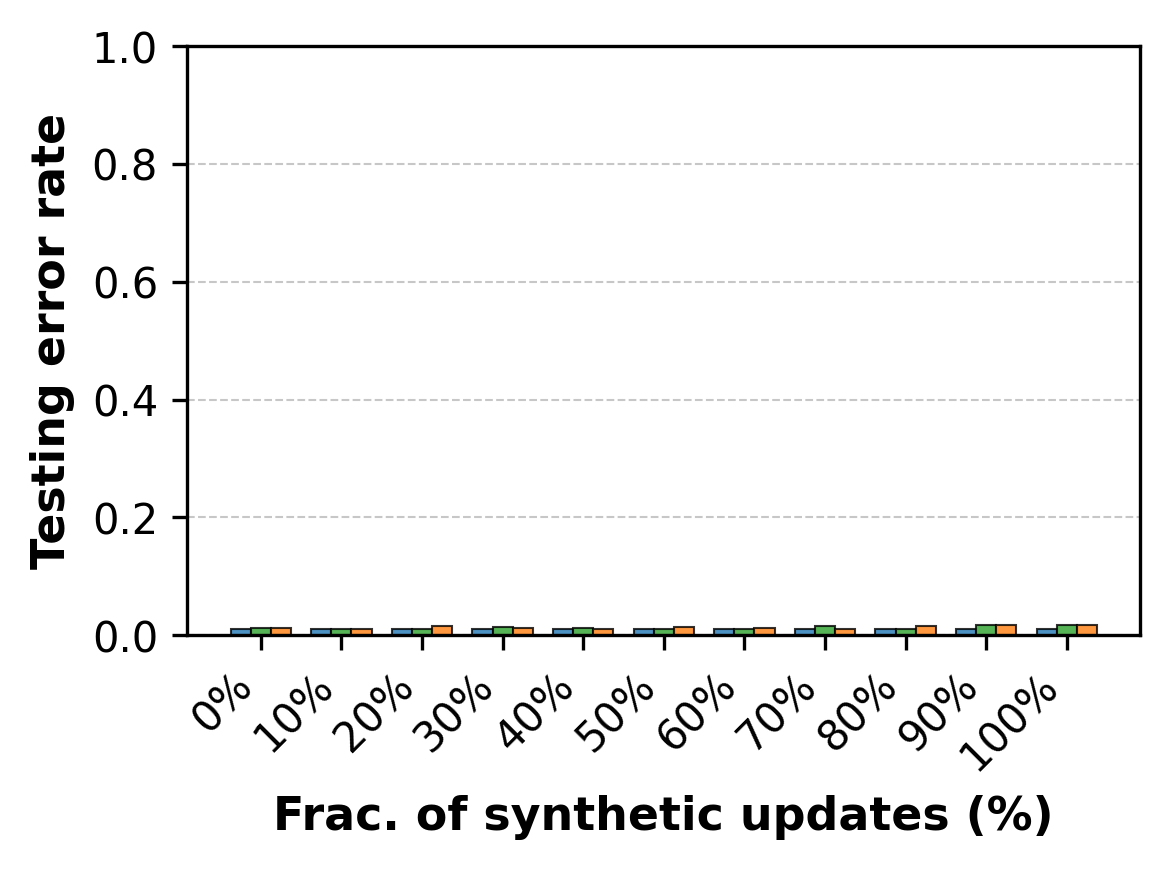}
		\small\captionof*{figure}{(b) LF attack}
	\end{minipage}
	\hfill
	\begin{minipage}[b]{0.24\textwidth}
		\centering
		\includegraphics[width=\textwidth]{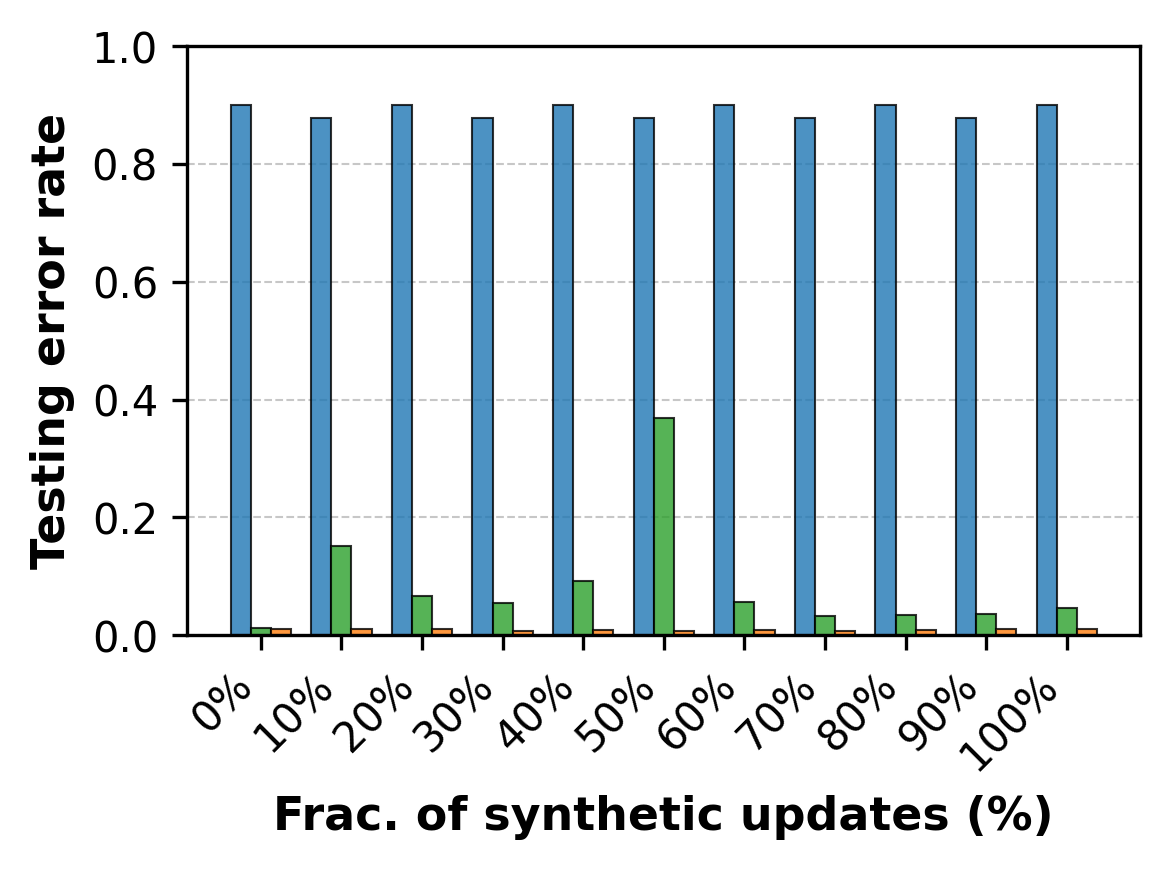}
		\small\captionof*{figure}{(c) Gaussian attack}
	\end{minipage}
	\hfill
	\begin{minipage}[b]{0.24\textwidth}
		\centering
		\includegraphics[width=\textwidth]{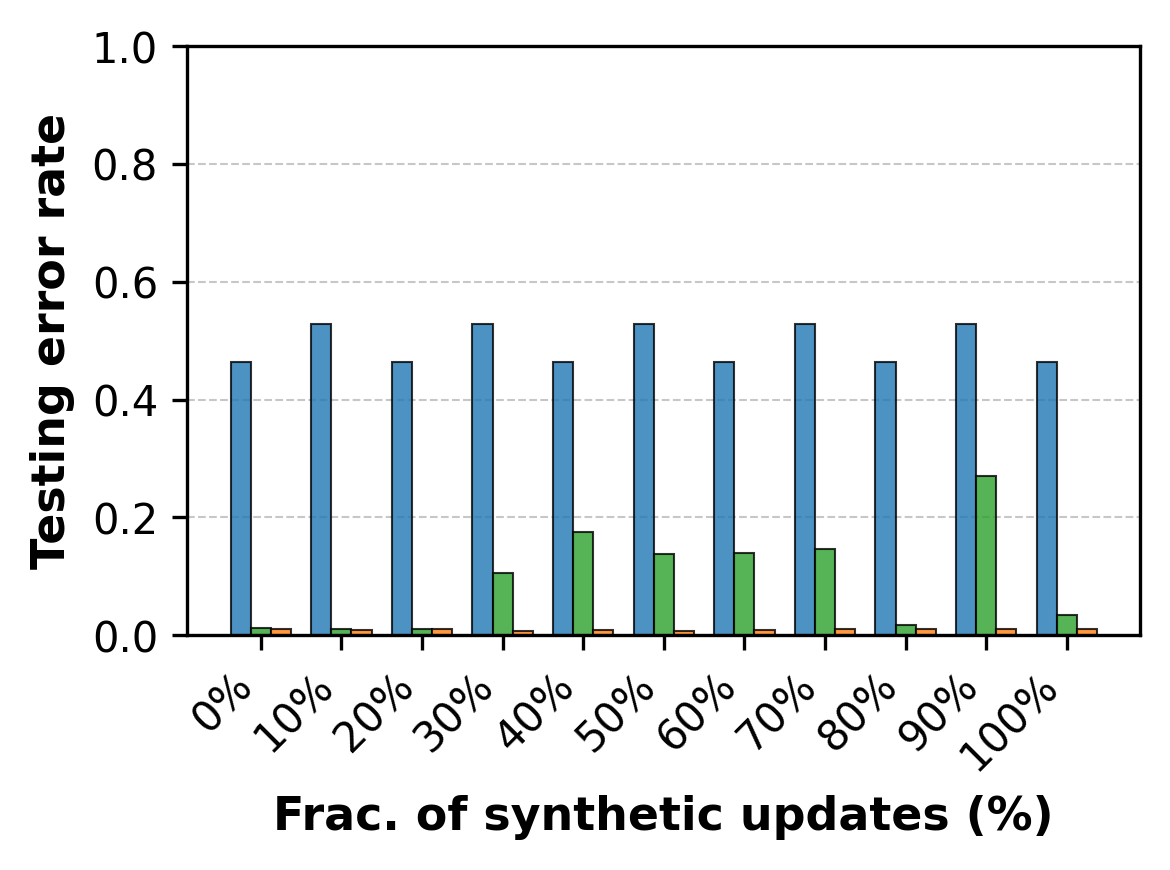}
		\small\captionof*{figure}{(d) Trim attack}
	\end{minipage}
	
	\vspace{0.15cm} 
	
	\begin{minipage}[b]{0.24\textwidth}
		\centering
		\includegraphics[width=\textwidth]{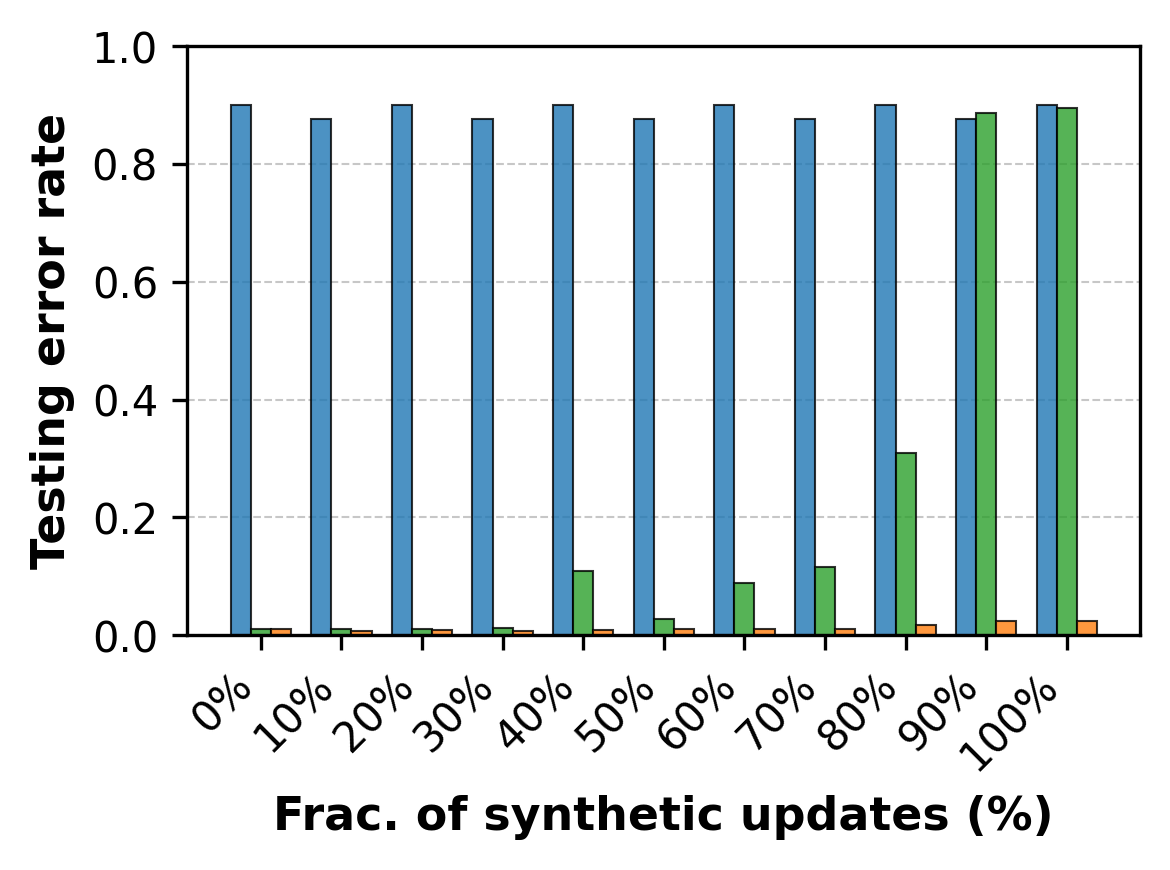}
		\small\captionof*{figure}{(e) Krum attack}
	\end{minipage}
	\hfill
	\begin{minipage}[b]{0.24\textwidth}
		\centering
		\includegraphics[width=\textwidth]{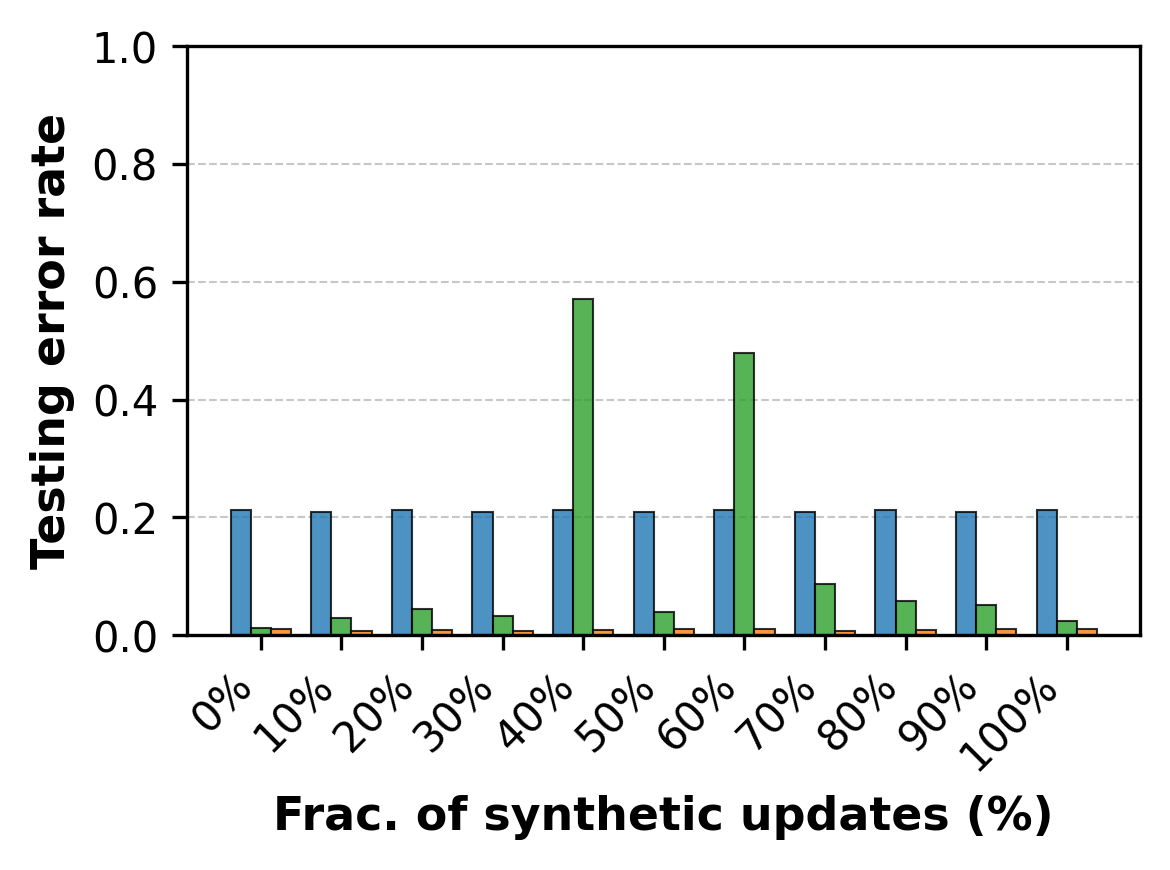}
		\small\captionof*{figure}{(f) Min-Max attack}
	\end{minipage}
	\hfill
	\begin{minipage}[b]{0.24\textwidth}
		\centering
		\includegraphics[width=\textwidth]{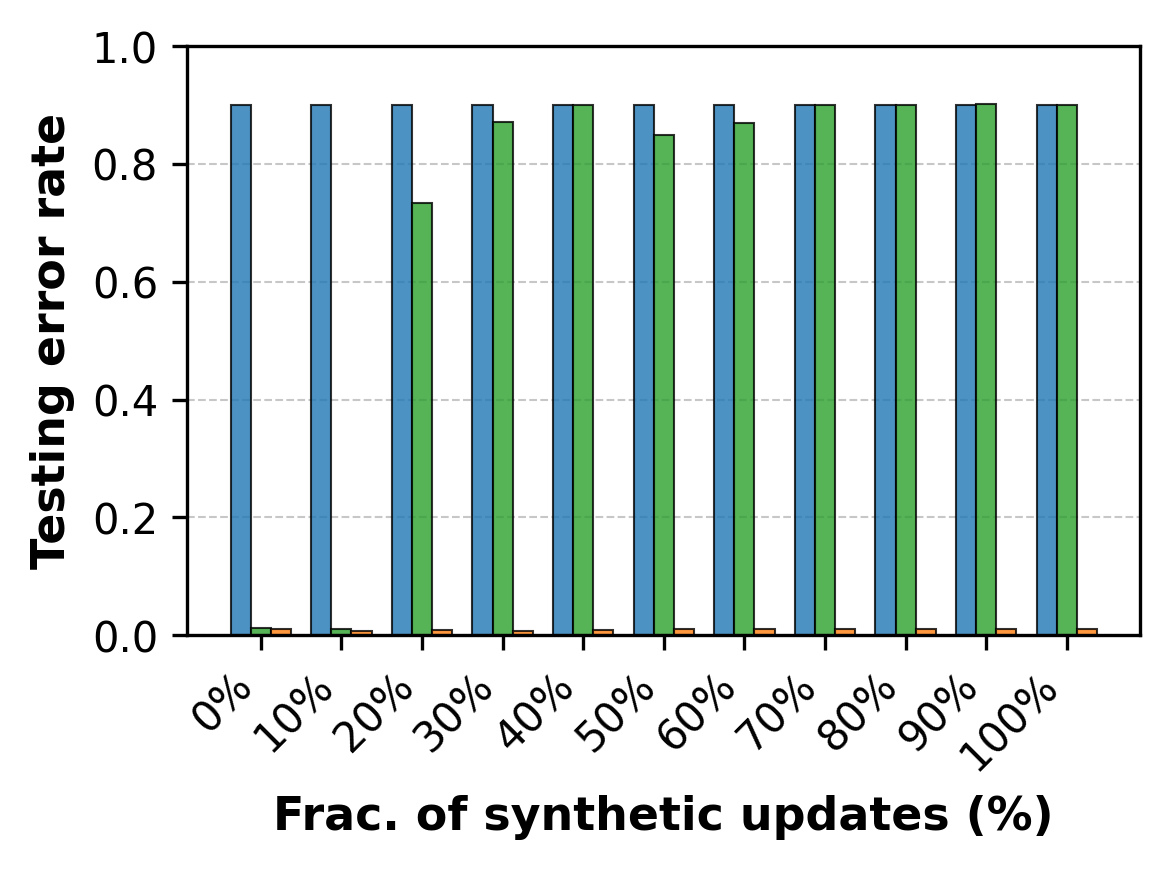}
		\small\captionof*{figure}{(g) Scaling attack}
	\end{minipage}
	\hfill
	\begin{minipage}[b]{0.24\textwidth}
		\centering
		\includegraphics[width=\textwidth]{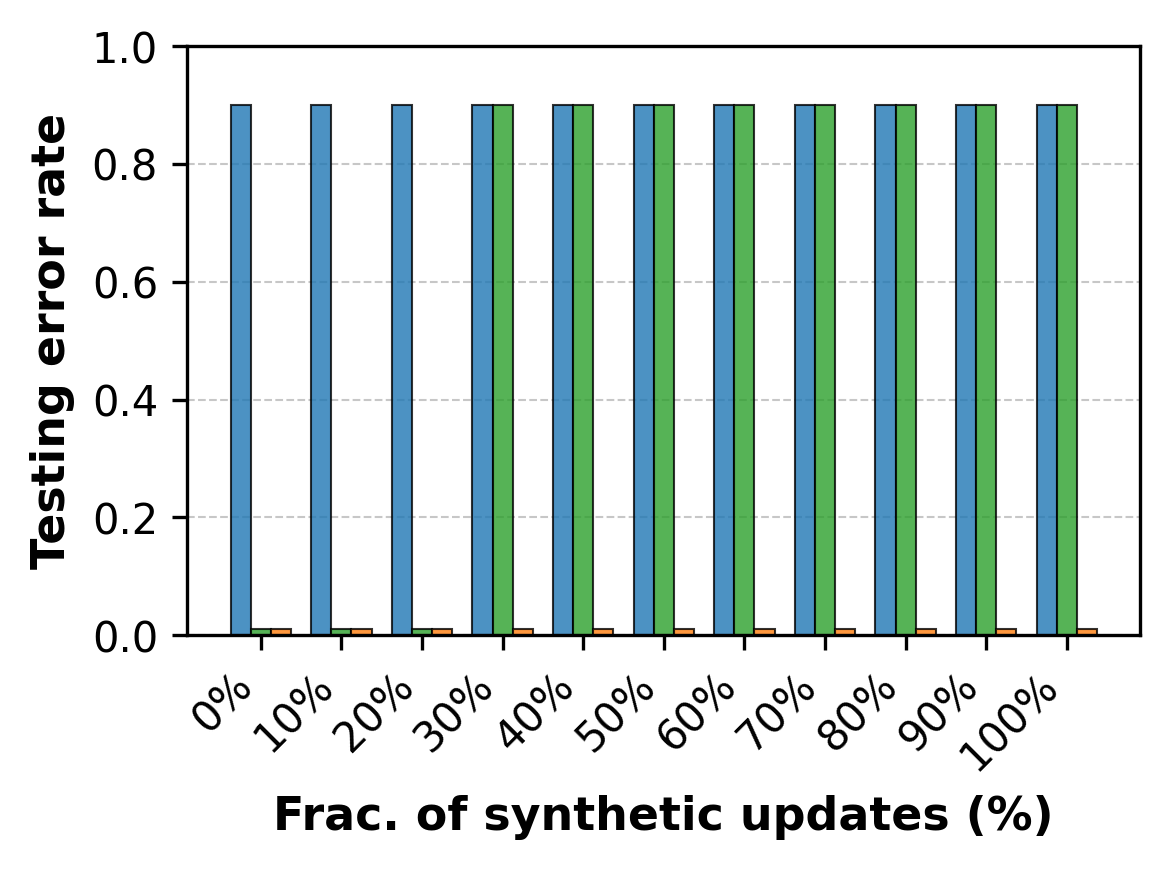}
		\small\captionof*{figure}{(h) Sybil attack}
	\end{minipage}
	
	\caption{Impact of fraction of synthetic updates on MNIST}
	\label{fig:experimental_results_synthetic_updates}
\end{figure*}

\begin{table*}[htbp]
	\centering
	\caption{Modify parameters $\rho_1$, $\rho_2$, $\beta_{\_\min}$, and $\beta_{\_\max}$ to validate AdaBFL-3's aggregated results under various attacks on MNIST.}
	\label{tab:rho}
	
	\centering
	\textbf{(a) ($\rho_1,\rho_2$)}
	
	\vspace{0.2cm}
	\begin{tabular}{lccccccccc}
		\toprule
		($\rho_1,\rho_2$)  & No  & LF & Gaussian & Trim  & Krum  & Min-Max & Scaling  & Sybil  \\
		\midrule
		
		(0.005,0.8) & 0.013 & 0.011 & 0.011 & 0.011 & 0.012 & 0.011 & 0.010 & 0.011 \\
		
		(0.001,0.6)  &   0.017 & 0.015 &	0.010 & 	0.010 & 0.010 &	0.010 & 0.010 & 0.010 \\
		
		(0.05,0.6)  &  0.015 & 	0.018 &	0.010 & 	0.011 & 0.010 & 0.011 & 0.010 & 0.010 \\
		
		(0.01,0.7)  &  0.012 & 	0.011 & 0.009 & 	0.008 & 0.008 & 0.008 & 0.008 & 0.008 \\

		\bottomrule
	\end{tabular}	
	\vspace{0.2cm}
	
	\centering
	\textbf{(b) ( $\beta_{\_\min},\beta_{\_\max}$)}
	
	\vspace{0.2cm}
	\begin{tabular}{lccccccccc}
		\toprule
		($\beta_{\_\min},\beta_{\_\max}$)  & No  & LF & Gaussian  & Trim & Krum  & Min-Max  & Scaling  & Sybil \\
		\midrule
		
		(0.1,0.6) & 0.013 & 0.011 & 0.011 & 0.011 & 0.012 & 0.011 & 0.010 & 0.011 \\
		
		(0.6,0.1)  &  0.015 & 	0.018 & 0.010 & 	0.010 & 0.010 & 0.010 & 0.010 & 0.009 \\
		
		(0.9,0.1)     &   0.014  & 	0.015  & 	0.010  & 	0.010  & 	0.010  & 	0.012  & 	0.010  & 	0.010 \\
		
		(0.8,0.05)  &   0.010 & 0.011 &	0.010 & 	0.011 & 0.010 & 0.011 & 0.011 & 0.010 \\	
		
		\bottomrule
	\end{tabular}	
	
\end{table*}

\textbf{Impact of total number of clients}:
Fig. \ref{fig:experimental_results_AII_number} illustrates the performance of various aggregation algorithms under poisoning attacks on the MNIST dataset as the number of clients increases from 50 to 300 (with malicious clients accounting for 0.3). The results indicate that other aggregation schemes, such as Gas+trim-mean and Gas+median, fail to adapt to changes in client scale, whereas AdaBFL effectively resists poisoning attacks across all client scales.

\textbf{Impact of fraction of synthetic updates:}
Leveraging models provided by benign clients can effectively enhance the performance of the global model. In the Algorithm \ref{alg:dfl} (line 17), we enhance the model's resistance to attacks by generating m new models from the benign models and incorporating these derived models into the aggregation process. Specifically, we evaluated the performance of Foundation-mean and AdaBFL under poisoning attacks as the synthetic ratio increased from 0\% to 100\%, using FedAvg as the baseline. In Fig.\ref{fig:experimental_results_synthetic_updates}, as the synthetic ratio rises, Foundation-mean's anti-attack performance fluctuates. This occurs because too few synthetic updates fail to mitigate the impact of malicious clients, while too many may obscure benign updates within the system. Notably, AdaBFL exhibits less sensitivity to the synthetic ratio, maintaining stable attack resistance (on par with FedAvg under non-attacked conditions).

\textbf{The Impact of parameters}:
Table \ref{tab:rho} shows the performance of AdaBFL-3 (parallel design) when varying the aggregation weight thresholds $\rho_1, \rho_2$ and $\beta_{\_\min}, \beta_{\_\max}$. For simplicity, in this experiment, we set $\beta_{1_\_\min}=\beta_{3_\_\min}=\beta_{\_\min}$ in Algorithm\ref{WUA}. The results demonstrate that AdaBFL exhibits strong stability: under no attack, altering parameters causes only minor fluctuations in the test error rate; under attack, parameter changes do not affect the test error rate.

\textbf{Evaluation of different AbaBFL algorithms:}
Table \ref{tab:Monent} presents the test error rate evaluation of the AbaBFL algorithm under parallel design with internal structure adjustments. Specifically, the following combinations were evaluated: Algorithm\ref{alg:dfl} + Algorithm\ref{WUA} $(Al_1)$, Algorithm\ref{alg:dfl} + Algorithm\ref{WUWT} $(Al_2)$, and Algorithm\ref{alg:AdaM} + Algorithm\ref{WUA} $(Al_3)$. At a malicious client ratio of 0.3, where momentum-based AbaBFL ($Al_3$) demonstrates optimal resistance.

\begin{table*}[htbp]
	\centering
	\caption{Evaluation of different AbaBFL algorithms on MNIST. Algorithm \ref{alg:dfl} + Algorithm \ref{WUA} $(Al_1)$, Algorithm \ref{alg:dfl} + Algorithm \ref{WUWT} $(Al_2)$, and Algorithm \ref{alg:AdaM} + Algorithm \ref{WUA} $(Al_3)$.}
	\label{tab:Monent}
	\centering
	\vspace{0.2cm}
	\begin{tabular}{lccccccccc}
		\toprule
		Algorithm  & No  & LF  & Gaussian  & Trim  & Krum  & Min-Max  & Scaling & Sybil \\
		\midrule
		
		$Al_1$ & 0.013 & 0.011 & 0.011 & 0.011 & 0.012 & 0.011 & 0.010 & 0.011 \\
		
		$Al_2$  &  0.016 & 0.010 & 0.009 & 	0.010 & 	0.011 & 0.009 & 0.010 & 0.010 \\
		
		$Al_3$  &  0.014 & 	0.009 & 0.009 & 0.010 & 0.009 &	0.008 & 0.010 & 0.009 \\
		
		\bottomrule
	\end{tabular}	
	
\end{table*}

\section{Conclusion}

\label{Conclusion}

In the paper, we proposed an adaptive Byzantine-robust federated learning method (AdaBFL) by constructing a novel three-layer defense architecture. Specifically, our three-layer defense architecture includes the filter malicious, parameter clipping, and derivative model, which adaptively adjusts the weights of each defense layer to effectively deal with complex attack scenarios. Moreover, we provide convergence analysis of our AdaBFL method, which exhibits stable convergence properties of our AdaBFL method. Extensive experiments across diverse scenarios validate the effectiveness of our proposed AdaBFL method. 

 \bibliographystyle{splncs04}
 \bibliography{main.bib}
\section{Appendix}
\label{appendix}

\subsection{Detailed Proof of Theorem~\ref{th1}}
In this section, we provide a detailed proof of Theorem~\ref{th1}.

\begin{proof}
	
	Here $g^{i}_{t}=\frac{1}{b}\sum_{j=1}^b\nabla f^i(\theta_{t-1};\xi^i_{t-1,j}) $ denotes the gradient of $i$-th client in training round $t$. 
	
	As shown in Algorithm \ref{WUA}, the aggregation formula is:
	\begin{align}
		\theta_{t} &  \gets \beta_1 \cdot \text{mean}(\{\theta_t^i\}_{i\in \left [ N \right ]  }) + \beta_2 \cdot \widetilde{\theta_{t}} + \beta_3 \cdot \overline{\theta_{t}}   =  \beta_1 \cdot \frac{1}{N}  {\textstyle \sum_{i\in [N] }^{} \theta_t^i }  + \beta_2 \cdot \widetilde{\theta_{t}} + \beta_3 \cdot \overline{\theta_{t}},
	\end{align}
	where
	\begin{equation}
		\label{Eq2}
		\begin{aligned}
			\frac{1}{ N } {\textstyle \sum_{i\in [N] }^{} \theta_t^i } = \theta_{t-1}-\eta \cdot \frac{1}{N}  {\textstyle \sum_{i\in [N] }^{} g^{i}_{t} }.
		\end{aligned}
	\end{equation}
	Therefore, we can obtain 
	\begin{align}
		\theta_{t} \leftarrow\left[\theta_{t-1}-\eta \cdot \frac{1}{N} \sum_{i \in [N]} g^{i}_{t}\right]+\beta_{2} \cdot \widetilde{\theta_{t}}+\beta_{3} \cdot \overline{\theta_{t}}.
	\end{align}
	Let $G_{t} = \frac{1}{N} \sum_{i \in [N]} g^{i}_{t}$.
	Since $\widetilde{\theta_{t}}$ and $\overline{\theta_{t}}$ cannot be linearly represented by $\theta_{t-1}$, we can define a useful aggregation error $e_{t}$ as follows:
	\begin{align}
			e_{t}=\theta_{t} &-  \beta_1 \cdot \left ( \theta_{t-1} - \eta G_{t}   \right ) -  \beta_2 \cdot \left ( \theta_{t-1} - \eta G_{t}   \right ) -\beta_3 \cdot \left ( \theta_{t-1} - \eta G_{t}   \right )  \nonumber
	\end{align}
	\begin{align}
			&=\theta_{t}-\left ( \beta_1+ \beta_2 +\beta_3  \right ) \cdot \left ( \theta_{t-1} - \eta G_{t}  \right )  = \theta_{t}- \left ( \theta_{t-1} - \eta G_{t}  \right ).
	\end{align}
	Thus, we have 
	\begin{align}
		\theta_{t} = \left ( \theta_{t-1} - \eta G_{t}  \right ) + e_{t}.
	\end{align}
	
By using Assumption \ref{ass2}, we have 
	\begin{align}
		\mathbb{E}\left [ g_{t}^{i} \right ]  = \nabla f^{i}(\theta _{t-1}),
	\end{align}
	\begin{align}
		& \mathbb{E}\left [\left \|  g_{t}^{i}  - \nabla f^{i}(\theta _{t-1})\right \|^2  \right ]
		= \mathbb{E}\left [\left \|  \frac{1}{b} \sum_{j=1}^{b} \nabla f^i(\theta_{t-1};\xi^i_{t-1,j})  - \nabla f^{i}(\theta _{t-1})\right \|^2  \right ]  \nonumber \\
		& = \frac{1}{b^2}\mathbb{E}\left [\left \|  \sum_{j=1}^{b} \nabla f^i(\theta_{t-1};\xi^i_{t-1,j})  - \nabla f^{i}(\theta _{t-1})\right \|^2  \right ] 
		= \frac{\sigma^2}{b}.
	\end{align}
 Then we can get
	\begin{align}
		\mathbb{E} \left [ G_{t} \right ]  
		= \frac{1}{N} \sum_{i \in [N]} \mathbb{E} \left [ g^{i}_{t}\right ] 
		= \frac{1}{N} \sum_{i \in [N]} \nabla f^{i}(\theta _{t-1})
		= \nabla F(\theta _{t-1})
	\end{align}
	\begin{align}
		\label{align001}
		\mathbb{E} \left [ \left \|  G_{t} -\nabla F(\theta _{t-1})  \right \|^2  \right ] \le \left ( \frac{\sigma ^2}{bN } + \tau^2 \right ). 
	\end{align}
We can rewrite Eq. (\ref{align001}) as follows:
	\begin{align}
		\label{align002}
		\left \|  G_{t} -\nabla F(\theta _{t-1})  \right \|^2  
		&= \left \| \frac{1}{N} \sum_{i \in  [N]   }^{} \left (  g_t^i -\nabla F(\theta _{t-1}) \right )  \right \|^2
		= \left \| \frac{1}{N } \sum_{i \in  [N]   }^{} \left (  g_t^i - \nabla f^{i}(\theta _{t-1}) + \nabla f^{i}(\theta _{t-1})-\nabla F(\theta _{t-1}) \right )  \right \|^2  \nonumber \\
		&= \left \| \frac{1}{N  } \sum_{i \in  [N]}^{} \left (  g_t^i - \nabla f^{i}(\theta _{t-1})  \right ) \right \|^2   + \left \| \frac{1}{N } \sum_{i \in  [N] }^{} \left ( \nabla f^{i}(\theta _{t-1})-\nabla F(\theta _{t-1}) \right )   \right \|^2.
	\end{align}
Then by calculating the expected value of Eq. (\ref{align002}) and using Assumption \ref{ass2} \ref{ass3}, we can get
	\begin{align}
		\mathbb{E} \left [ \left \|  G_{t} -\nabla F(\theta _{t-1})  \right \|^2 \right ]   
		& = \mathbb{E} \left [ \left \| \frac{1}{N  } \sum_{i \in  [N]}^{} \left (  g_t^i - \nabla f^{i}(\theta _{t-1})  \right ) \right \|^2 \right ]    + \mathbb{E} \left [ \left \| \frac{1}{N } \sum_{i \in  [N]   }^{} \left ( \nabla f^{i}(\theta _{t-1})-\nabla F(\theta _{t-1}) \right )   \right \|^2  \right ] \nonumber  \\
		& \le  \frac{1}{N^2 } \cdot  \frac{N \sigma ^2}{b} + \tau ^2 = \left ( \frac{\sigma ^2}{bN} + \tau^2 \right ).
	\end{align}
	
	By using Assumption \ref{ass1}, we have 
	\begin{align}
		\label{align003}
			F(\theta_{t}) \le F(\theta_{t-1})  -\eta \langle \nabla F(\theta _{t-1}),  G_{t} \rangle + \langle \nabla F(\theta _{t-1}),  e_{t} \rangle + \frac{L}{2} \| -\eta G_{t} +e_{t}\|^2,
	\end{align}
	Then take the expectation of Eq. \ref{align003}, we can get 
	\begin{align}
		\label{align004}
		\mathbb{E} \left [ F(\theta_{t})  \right ] \le \mathbb{E}\left [ F(\theta_{t-1}) \right ]   -\eta  \left \| \nabla F(\theta _{t-1}) \right \|^2  + \underset{(a)}{\underbrace{\mathbb{E}\left [ \langle \nabla F(\theta _{t-1}),  e_{t} \rangle \right ] } }  + \underset{(b)}{\underbrace{\frac{L}{2} \mathbb{E}\left [ \| -\eta G_{t} +e_{t}\|^2 \right ] } },
	\end{align}
	where the term (a) is equivalent to
	\begin{align}
		\mathbb{E}\left [ \langle \nabla F(\theta _{t-1}),  e_{t} \rangle \right ] 
		\le 
		\mathbb{E}\left [  \left \| \nabla F(\theta _{t-1}) \right \| \cdot \left \|  e_{t} \right \|   \right ]. 
	\end{align}
	From \cite{yin2018byzantine}, we know that $\left \| e_{t}  \right \|$ is bounded, and set $\left \| e_t \right \| \le K \eta^2 $, where $K=\sqrt{\sigma + \tau} $. Then we can obtain 
	\begin{align}
		\label{align007}
		\mathbb{E}\left [ \langle \nabla F(\theta _{t-1}),  e_{t} \rangle \right ] 
		\le K\eta^2 \left \| \nabla F(\theta _{t-1}) \right \|.  
	\end{align}
	In Eq.\ref{align004}, the term (b) is equivalent to
	\begin{align}
	\label{align005}
	\frac{L}{2} \mathbb{E}\left [ \| -\eta G_{t} +e_{t}\|^2 \right ] 
	\le L\eta ^2 \left ( \frac{2\sigma ^2}{bN}  + \tau^2 \right )  +2L\eta ^2 \left \| \nabla F(\theta _{t-1}) \right \|^2 +LK^2\eta ^4  .
	\end{align}
	We can rewrite Eq. (\ref{align005}) is as follows:
	\begin{align}
		\label{align006}
		\| -\eta G_{t} +e_{t}\|^2 
		\le 2 \left \| -\eta G_t \right \|^2 + 2 \left \| e_t \right \| ^2  = 2\eta^2 \left \|  G_t \right \|^2 + 2 K^2\eta ^4 ,
	\end{align}
	where Eq. (\ref{align006}), use $(a+b)^2 \le 2a^2 +2b^2$ to scale and $\left \|  G_t \right \|^2$  is equivalent to
	\begin{align}
		\left \|  G_t \right \|^2  = \left \| G_t - \nabla F(\theta _{t-1}) + \nabla F(\theta _{t-1}) \right \|^2
		& \le 2  \left \| G_t - \nabla F(\theta _{t-1}) \right \|^2 +2 \left \| \nabla F(\theta _{t-1}) \right \|^2 \nonumber  \\
		&=2\left ( \frac{\sigma ^2}{bN} +\tau ^2 \right ) +2 \left \| \nabla F(\theta _{t-1}) \right \|^2.
	\end{align}
	 Then we can get
	\begin{align}
		\frac{L}{2} \mathbb{E}\left [ \| -\eta G_{t} +e_{t}\|^2 \right ]
		&\le \frac{1}{2} \mathbb{E}\left [ 2\eta^2 \left \| G_{t} \right \|^2  +2K^2\eta^4 \right ]
		= \frac{1}{2} \mathbb{E}\left [ 2\eta^2 \left ( \left ( \frac{2 \sigma ^2}{bN} +\tau ^2 \right ) +2 \left \| \nabla F(\theta _{t-1}) \right \|^2 \right )  +2K^2\eta^4 \right ] \nonumber \\
		&= L\eta ^2 \left ( \frac{2\sigma ^2}{bN }  + \tau^2 \right )  +2L\eta ^2 \left \| \nabla F(\theta _{t-1}) \right \|^2 +LK^2\eta ^4  .
	\end{align}
	Let $D = \frac{2\sigma ^2}{b\left | N _t \right | } +\tau ^2$, and substitute (\ref{align006}) and (\ref{align007}) into (\ref{align004}) to obtain
	\begin{align}
		\label{align009}
		\mathbb{E} \left [ F(\theta_{t})  \right ] \le \mathbb{E}\left [ F(\theta_{t-1}) \right ]   -\eta  \left \| \nabla F(\theta _{t-1}) \right \|^2
		+\underset{(c)}{\underbrace{ K\eta^2 \left \| \nabla F(\theta _{t-1}) \right \| }}  
		+ L\eta ^2 D  +2L\eta ^2 \left \| \nabla F(\theta _{t-1}) \right \|^2 +LK^2\eta ^4.
	\end{align}
	By applying the inequality $ab \le \frac{\alpha }{2} a^2 + \frac{1}{2\alpha } b^2$ to the term $ (c)$ in the above inequality~(\ref{align009}), where $a=\sqrt{\eta } \left \| \nabla F(\theta_{t-1}  ) \right \|$ and $K\eta^{\frac{3}{2} } $, we can get
	\begin{align}
		K\eta^2 \left \| \nabla F(\theta _{t-1}) \right \|
		\le \frac{\alpha }{2} \eta \left \| \nabla F(\theta _{t-1}) \right \|^2 + \frac{1}{2\alpha }K^2\eta ^3.
	\end{align}
	The let $\alpha =\frac{1}{2} $
	\begin{align}
		\label{align008}
		K\eta^2 \left \| \nabla F(\theta _{t-1}) \right \| \le \frac{\eta}{4}  \left \| \nabla F(\theta _{t-1}) \right \|^2 + K^2\eta ^3.
	\end{align}
	By substituting the inequality  (\ref{align008}) into (\ref{align009}), we can obtain
	\begin{align}
		\mathbb{E} \left [ F(\theta_{t})  \right ] &\le \mathbb{E}\left [ F(\theta_{t-1}) \right ]   -\eta  \left \| \nabla F(\theta _{t-1}) \right \|^2
		+ \frac{\eta}{4}  \left \| \nabla F(\theta _{t-1}) \right \|^2 + K^2\eta ^3
		+ L\eta ^2 D  +2L\eta ^2 \left \| \nabla F(\theta _{t-1}) \right \|^2 +LK^2\eta ^4 \nonumber \\
		& = \mathbb{E}\left [ F(\theta_{t-1}) \right ]  
		+ \left ( \frac{-3\eta}{4} + 2L\eta ^2  \right )  \left \| \nabla F(\theta _{t-1}) \right \|^2 + K^2\eta ^3
		+ L\eta ^2 D  +LK^2\eta ^4.
	\end{align}
	Let $\eta \le \frac{1}{8L}$, we can obtain
	\begin{align}
		\mathbb{E} \left [ F(\theta_{t})  \right ] \le \mathbb{E}\left [ F(\theta_{t-1}) \right ]  
		- \frac{\eta}{2} \left \| \nabla F(\theta _{t-1}) \right \|^2 + K^2\eta ^3
		+ L\eta ^2 D  +LK^2\eta ^4.
	\end{align}
	Telescoping over $t = 1,..., T$ and using Assumption \ref{ass4}, one has that:
	\begin{align}
		\frac{1}{T} \sum_{t=0}^{T-1} \mathbb{E}\left[\left\|\nabla F\left(\theta_{t}\right)\right\|^{2}\right] & \leq \frac{2\left ( F\left(\theta_{0}\right)-F^* \right ) }{\eta T} + 2K^2\eta^2 \left ( 1-L\eta  \right ) + 2L\eta D \nonumber \\
		& \leq \frac{2\left ( F\left(\theta_{0}\right)-F^* \right ) }{\eta T} + \frac{7K^2\eta^2}{4} + 2L\eta D.
	\end{align}
\end{proof}

\subsection{Additional Experiments}
 \begin{figure}[htbp]
	\centering
	\includegraphics[width=0.75\textwidth,height=0.4\textwidth]{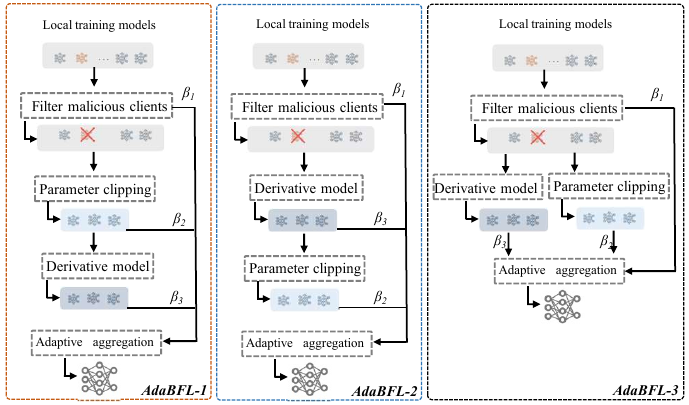}
	\caption {The configuration of the three-layer defense mechanism. AdaBFL-1 and AdaBFL-2 are designed in a series configuration, while AdaBFL-3 employs a parallel design. Both series and parallel configurations utilize the update model of benign clients as input.}
	\label{Combination Method}
\end{figure}

\begin{table*}[htbp]
	\centering
	\caption{Test results for three different configurations (series, parallel) under various attacks on MNIST.}
	\label{tab:AdaBFL1-3}
	
	\vspace{0.2cm}
	\begin{tabular}{lcccccccc}
		\toprule
		Malicious Ratio & Method \quad & LF  & Gaussian  & Trim  & Krum  & Min-Max  & Scaling  & Sybil  \\
		\midrule
		
		\multirow{3}{*}{0\%} & AdaBFL-1 \quad &  0.014 &  0.010 &	\underline{0.009} &	0.013 &	0.013 &	\underline{0.009}	 & \underline{0.009} \\
		& AdaBFL-2  \quad    &  0.017 &  0.013 &	0.020 &	0.014 &	0.016 &	0.019	 & 0.012 \\
		& AdaBFL-3  \quad    &  0.015 &  0.013 &	0.021 &	0.013 &	0.013 &	0.012	 & 0.012 \\
		
		\midrule
		\multirow{3}{*}{10\%} & AdaBFL-1 \quad &  \underline{0.010} & \underline{0.009} & 0.010 & \underline{0.009} & 0.010 & 	0.010 & 0.010 \\
		& AdaBFL-2 \quad &  0.018 &	0.010 & 0.010 &	0.010 & 0.010 & 0.010 & 0.011 \\
		& AdaBFL-3 \quad  &  0.017 &  0.011 &	0.011 &	0.011 &	0.011 &	0.011	 & 0.011 \\
		
		\midrule
		\multirow{3}{*}{20\%} & AdaBFL-1 \quad &  0.011 &	0.011 & 0.010 & \underline{0.009} & 0.010 & \underline{0.009} & 0.010 \\
		& AdaBFL-2 \quad &  \underline{0.010} &	0.010 & \underline{0.009} &	\underline{0.009} & \underline{0.009} & 0.010 & \underline{0.009} \\
		& AdaBFL-3 \quad  &  0.013  & 	0.011  & 	0.011  &	0.011  & 0.010 & 	0.010 & 	0.011 \\
		\midrule
		\multirow{3}{*}{30\%} & AdaBFL-1 \quad &  \underline{0.010} &	0.010 & 0.010 & 0.010 & 0.010 & 	\underline{0.009} & 0.010 \\
		& AdaBFL-2 \quad  &  0.014 & 0.010 & \underline{0.009} &	0.010 &	0.010 &	0.010 & 0.012 \\
		& AdaBFL-3 \quad & 0.011 & 0.011 &	0.011 &	0.012 &	0.011 &	0.010 &	0.011 \\
		
		\midrule
		\multirow{3}{*}{40\%} & AdaBFL-1 \quad &  \underline{0.012}  & 0.010  & 0.010  & 0.010  & 0.900  & 0.764  & 0.900 \\
		& AdaBFL-2 \quad &  0.020 &	0.010 &	0.010 & 	0.010 &	0.900 &	0.017 & 0.900 \\
		& AdaBFL-3 \quad  &  \underline{0.012} & 	0.010 & 0.028 & 	0.013 & 0.012 & 0.011 & 0.783 \\
		
		\bottomrule
	\end{tabular}
	
\end{table*}

As shown in Fig.~\ref{Combination Method}, we designed combinations of this three-layer defensive mechanism under different aggregation scenarios (series and parallel). Among them, AdaBFL-1 and AdaBFL-2 adopt a serial design, while AdaBFL-3 adopts a parallel design. In AdaBFL-1, the server first performs malicious client filtering, then takes the models of benign clients obtained from filtering as input to Parameter Clipping, and the output of Parameter Clipping serves as input to the Derivative Model. In AdaBFL-2, the server first performs malicious client filtering, then takes the models of benign clients obtained from filtering as input to the Derivative Model, and the output of the Derivative Model serves as input to Parameter Clipping. In AdaBFL-3, the server first performs malicious client filtering, then takes the models of benign clients obtained from filtering as input to both Parameter Clipping and the Derivative Model.

Defense evaluation for series or parallel designs:
In Fig.\ref{Combination Method}, we designed serial and parallel AdaBFL architectures. In Table \ref{tab:AdaBFL1-3}, this experiment evaluated the anti-poisoning attack performance of serial and parallel designs as the proportion of malicious clients increased from 0\% to 40\%. Results show that under typical malicious ratios (below 50\%), both serial and parallel AdaBFL effectively resisted poisoning attacks. Overall, the serial design exhibits a lower test error rate than the parallel design.

\begin{algorithm}[H]
	\caption{AdaBFL with Momentum}
	\label{alg:AdaM}
	\begin{algorithmic}[1]
		\footnotesize
		\State \textbf{Input} Total number of clients \( N \); number of clients participating in training per epoch \( n \); filtered benign clients \( \Lambda \); local training data \( D^{i}, i \in N \); number of global training rounds \( T \); learning rate \( \eta \); weights \( \beta_1,\beta_2,\beta_3 \); momentum factor \(\alpha\);
		\State \textbf{Output} Global model \( \theta_T \).
		
		\State Initialize global model \( \theta^0 \).
		\State Initialize \(\rho_1 \), \(\rho_2 \) \Comment{Dynamic thresholds}
		\State \textit{// \textcolor{blue}{Global model synchronization} }
		\State Server broadcasts \( \theta_0 \) to all clients.
		\State \textit{// \textcolor{blue}{Local model updating} }
		\For{\( t =  1, \dots, T-1 \)}
		
		\For{each client \( i \in [n] \) in parallel}
		\State Drawn $b$ i.i.d. samples $\mathcal{B}=\{\xi_{t-1,j}^i\}_{j=1}^b$ from data distribution $\mathcal{D}^i$;
		\State $g^{i}_{t}=\frac{1}{b}\sum_{j=1}^b\nabla f^i(\theta_{t-1};\xi^i_{t-1,j})$; 
		\State \( \theta_{t}^{i} = \theta_{t-1} - \eta g_t^{i} \).
		\EndFor
		\State Clients send \(\theta_t^{i}\) to server.
		
		\State \textit{// \textcolor{blue}{ Adaptive aggregation}}
		\For{each client \( i \in [n] \)}
		\State \textit{// \textcolor{blue}{(1) Filter malicious clients}}
		\State {The server executes Eq.(\ref{EqP1})}
		
		\State \textit{// \textcolor{blue}{(2) Parameter clipping}}
		\State {The server executes Eq.(\ref{EqP2}) and Eq.(\ref{EqP2-2})}
		
		\State \textit{// \textcolor{blue}{(3) Derivative model
		}}
		\State {The server executes Eq.(\ref{EqP3}-\ref{EqP6}))}
		\EndFor  
		
		\State \(\rho_1 \gets \alpha \cdot \rho_1 + (1 - \alpha) \cdot p^1_{t}\)  
		\State \(\rho_2 \gets \alpha \cdot \rho_2 + (1 - \alpha) \cdot p^2_{t}\)  
		\State {The server executes Weight Update Algorithm (\ref{WUA})}  
		
		\EndFor  
	\end{algorithmic}
\end{algorithm}

\begin{table*}[htbp]
	\centering
	\caption{Validation results of different aggregation methods on the MNIST, Petimage, Shakespeare and HAR datasets. The proportion of malicious clients is 0.3.}
	\label{AdaBFL_effective2}
	
	\centering
	\vspace{0.2cm} 
	
	\textbf{(a) Fashion-MNIST dataset}
	
	\vspace{0.2cm}
	\begin{tabular}{lccccccccc}
		\toprule
		Method & No  & LF  & Gaussian  & Trim  & Krum  & Min-Max  & Scaling  & Sybil \\
		\midrule
		
		FedAvg     & \textbf{0.090} & 0.091 & 0.703 & 0.900 & 0.856 & 0.900 & 0.900 & 0.900 \\
		
		Trim-mean     & 0.093 & 0.091 & 0.251 & 0.842 & 0.900 & 0.212 & 0.900  & 0.900 \\
		
		GAS+Trim-mean     & 0.157 & 0.154 & 0.310 & 0.151 & 0.607 & 0.417 & 0.711  & 0.900 \\
		
		FoundationFL+Median     & \textbf{0.090} & \textbf{0.090} & 0.096 & 0.104 & \textbf{0.091} & 0.121 & 0.102  & 0.102 \\
		
		Gau+Trim-mean     & \textbf{0.090} & 0.092 & 0.547 & 0.857 & 0.900 & 0.872 & 0.900 & 0.900 \\
		
		Median     & 0.093 & 0.092 & \textbf{0.090} & 0.120 & 0.096 & 0.101 & 0.124  & 0.116 \\
		
		GAS+Median     & 0.161 & 0.155 & 0.157 & 0.150 & 0.157 & 0.162 & 0.149  & 0.151 \\
		
		Gau+Median     & 0.091 & 0.091 & 0.095 & 0.128 & 0.211 & 0.111 & 0.822  & 0.900 \\
		
		FoundationFL+Mean     & 0.091 & 0.094 & 0.243 & 0.114 & 0.130 & 0.262 & 0.900  & 0.900 \\
		
		\rowcolor{lightblue}
		AdaBFL(ours)    & \textbf{0.090} & \textbf{0.090} & 0.092 & \textbf{0.091} & 0.092 & \textbf{0.092} & \textbf{0.092}  & \textbf{0.094} \\
		\bottomrule
	\end{tabular}
	
	\vspace{0.2cm} 
	
	\centering
	\textbf{(b) Petimage dataset}
	
	\vspace{0.2cm}
	\begin{tabular}{lccccccccc}
		\toprule
		Method & No  & LF  & Gaussian  & Trim  & Krum  & Min-Max  & Scaling & Sybil  \\
		\midrule
		
		FedAvg     & 0.175 & 0.316 & 0.478 & 0.482 & 0.500 & 0.401 & 0.500  & 0.500 \\
		
		Trim-mean     & 0.168 & 0.287 & 0.500 & 0.450 & 0.500 & 0.330 & 0.500  & 0.500 \\
		
		GAS+Trim-mean     & 0.328 & 0.500 & 0.427 & \textbf{0.110} & 0.316 & 0.500 & 0.347  & 0.500 \\
		
		FoundationFL+Median     & 0.184 & 0.326 & 0.187 & 0.197 & 0.180 & 0.327 & 0.305  & 0.305 \\
		
		Gau+Trim-mean     & 0.182 & 0.254 & 0.500 & 0.500 & 0.500 & 0.500 & 0.500  & 0.500 \\
		
		Median     & 0.175 & 0.357 & 0.180 & 0.198 & 0.218 & 0.238 & 0.209  & 0.217 \\
		
		GAS+Median     & 0.326 & 0.500 & 0.318 & 0.319 & 0.387 & 0.309 & 0.316  & 0.314 \\
		
		Gau+Median     & 0.173 & 0.318 & 0.185 & 0.202 & 0.253 & 0.239 & 0.367  & 0.500 \\
		
		FoundationFL+Mean     & \textbf{0.168} & 0.334 & 0.500 & 0.203 & 0.332 & 0.481 & 0.500  & 0.500 \\
		
		\rowcolor{lightblue}
		AdaBFL(ours)    & 0.224 & \textbf{0.272} & \textbf{0.177} & 0.176 & \textbf{0.175} & \textbf{0.169} & \textbf{0.178}  & \textbf{0.176} \\
		\bottomrule
	\end{tabular}
	
	\vspace{0.2cm} 
	
	\centering
	\textbf{(c) Shakespeare dataset}
	
	\vspace{0.2cm}
	\begin{tabular}{lccccccccc}
		\toprule
		Method & No  & LF  & Gaussian & Trim  & Krum  & Min-Max & Scaling  & Sybil  \\
		\midrule
		
		FedAvg     & 3.813 & 3.830 & 4.188 & 5.282 & 3.873 & 8.511 & - & - \\
		
		Trim-mean     & \textbf{3.733} &3.747 & 4.021 & 4.131 & 3.793 & 5.529 & -  & - \\
		
		GAS+Trim-mean     & 12.03 & 10.73 & 14.69 & 8.670 & 9.717 & 16.75 & 4.181  & - \\
		
		FoundationFL+Median     & 3.832 & 4.945 & 3.916 & 3.913 & 3.859 & 3.898 & 3.882  & 3.853 \\
		
		Gau+Trim-mean     & 3.832 & 4.941 & 3.916 & 3.913 & 3.857 & 3.898 & 3.882  & 3.853 \\
		
		Median     & 3.809 & 3.756 & 3.763 & 3.814 & 3.838 & 3.912 & 4.033 & 3.948 \\
		
		GAS+Median     & 10.34 & 12.41 & 9.633 & 10.50 & 10.40 & 10.09 & 11.78 & 9.068 \\
		
		Gau+Median     & 3.820 & 3.868 & 3.821 & 4.000 & 3.887 & 4.162 & - & - \\
		
		FoundationFL+Mean     & 3.785 & 4.814 & 4.243 & 3.823 & \textbf{3.784} & 5.951 & 3.909 & 15.94 \\
		
		\rowcolor{lightblue}
		AdaBFL(ours)    & 3.902 & \textbf{3.721} & \textbf{3.823} & \textbf{3.755} & 3.965 & \textbf{3.790} & \textbf{3.81} & \textbf{3.861} \\
		\bottomrule
	\end{tabular}
	
	\centering
	\vspace{0.2cm}
	\textbf{(d) HAR dataset}
	
	\vspace{0.2cm}
	\begin{tabular}{lccccccccc}
		\toprule
		Method & No  & LF  & Gaussian  & Trim  & Krum  & Min-Max  & Scaling  & Sybil \\
		\midrule
		
		FedAvg     & 0.047 & 0.246 & 0.837 & 0.090 & 0.632 & 0.373 & 0.831  & 0.831 \\
		
		Trim-mean     & 0.047 & 0.053 & 0.900 & 0.900 & 0.900 & 0.900 & 0.900  & 0.900 \\
		
		GAS+Trim-mean     & 0.373 & 0.443 & 0.703 & 0.073 & 0.067 & 0.335 & 0.831  & 0.831 \\
		
		FoundationFL+Median     & \textbf{0.045} & 0.047 & 0.059 & 0.052 & \textbf{0.044} & 0.087 & 0.061  & \textbf{0.053} \\
		
		Gau+Trim-mean     & \textbf{0.045} & 0.050 & 0.708 & 0.129 & 0.545 & 0.431 & 0.831  & 0.831 \\
		
		Median     & \textbf{0.045} & \textbf{0.044} & 0.055 & 0.050 & 0.050 & 0.063 & \textbf{0.050} & 0.062 \\
		
		GAS+Median     & 0.343 & 0.323 & 0.313 & 0.285 & 0.293 & 0.461 & 0.315  & 0.262 \\
		
		Gau+Median     & \textbf{0.045} & 0.053 & 0.064 & 0.057 & 0.317 & 0.091 & 0.817  & 0.831 \\
		
		FoundationFL+Mean     & 0.047 & 0.051 & 0.772 & \textbf{0.048} & 0.063 & 0.236 & 0.107  & 0.831 \\
		
		\rowcolor{lightblue}
		AdaBFL(ours)    & \textbf{0.045} & 0.055 & \textbf{0.053} & 0.053 & 0.047 & \textbf{0.052} & 0.055  & 0.055 \\
		\bottomrule
	\end{tabular}
	
\end{table*}

\begin{table*}[htbp]
	\centering
	\caption{Results for Median and AdaBFL across various datasets when malicious clients exceed 50\% of the total.}
	\label{tab_Malicious_1}

	\centering
	\vspace{1mm}
	\textbf{(a) MNIST dataset}
	
	\begin{tabular}{lcccccccc}
		\toprule
		Malicious Ratio & Method & LF  & Gaussian  & Trim  & Krum  & Min-Max & Scaling & Sybil  \\
		\midrule
		
		\multirow{2}{*}{55\%} & Median  &  0.010 &  0.042 &	0.900 &	0.900 &	0.518 &	0.900	 & 0.900 \\
		& AdaBFL    & \cellcolor{lightblue}0.010 & \cellcolor{lightblue}0.011 & \cellcolor{lightblue}0.010 & \cellcolor{lightblue}0.900 & \cellcolor{lightblue}0.010 & \cellcolor{lightblue}0.900 & \cellcolor{lightblue}0.900 \\
		\midrule
		\multirow{2}{*}{60\%} & Median     &  0.010 & 0.060 & 0.895 &	0.900 & 	0.895 &	0.900 & 0.900 \\
		& AdaBFL    & \cellcolor{lightblue}0.010 & \cellcolor{lightblue}0.011 & \cellcolor{lightblue}0.094 & \cellcolor{lightblue}0.900 & \cellcolor{lightblue}0.010 & \cellcolor{lightblue}0.900  & \cellcolor{lightblue}0.900 \\
		\midrule
		\multirow{2}{*}{65\%} & Median  &   0.010 & 0.886 & 0.886 & 0.900 & 0.900 & 0.900 & 0.900 \\
		& AdaBFL    & \cellcolor{lightblue}0.010 & \cellcolor{lightblue}0.023 & \cellcolor{lightblue}0.830 & \cellcolor{lightblue}0.900 & \cellcolor{lightblue}0.011 & \cellcolor{lightblue}0.900  & \cellcolor{lightblue}0.900 \\
		\bottomrule
	\end{tabular}
	
	\centering
	\vspace{1mm}
	\textbf{(b) Fashion-MNIST dataset}
	
	\begin{tabular}{lcccccccc}
		\toprule
		Malicious Ratio & Method & LF  & Gaussian & Trim  & Krum  & Min-Max  & Scaling  & Sybil  \\
		\midrule
		
		\multirow{2}{*}{55\%} & Median  &   0.093 & 0.108 & 0.850 &	0.900 & 0.261 &	0.900 & 0.900 \\
		& AdaBFL    & \cellcolor{lightblue}0.091 & \cellcolor{lightblue}0.097 & \cellcolor{lightblue}0.153 & \cellcolor{lightblue}0.900 & \cellcolor{lightblue}0.094 & \cellcolor{lightblue}0.900 & \cellcolor{lightblue}0.900 \\
		\midrule
		\multirow{2}{*}{60\%} & Median   &  0.096 &	0.194 &	0.855 & 0.900 & 0.900 & 0.900 & 0.900 \\
		& AdaBFL    & \cellcolor{lightblue}0.091 & \cellcolor{lightblue}0.098 & \cellcolor{lightblue}0.702 & \cellcolor{lightblue}0.900 & \cellcolor{lightblue}0.096 & \cellcolor{lightblue}0.900  & \cellcolor{lightblue}0.900 \\
		\midrule
		\multirow{2}{*}{65\%} & Median   &  0.098 & 0.264 & 0.900 & 0.900 & 0.900 &	0.900 & 0.900 \\
		& AdaBFL    & \cellcolor{lightblue}0.092 & \cellcolor{lightblue}0.098 & \cellcolor{lightblue}0.723 & \cellcolor{lightblue}0.900 & \cellcolor{lightblue}0.091 & \cellcolor{lightblue}0.900  & \cellcolor{lightblue}0.900 \\
		\bottomrule
	\end{tabular}

	\centering
	\vspace{1mm}
	\textbf{(c) HAR dataset}
	
	\begin{tabular}{lcccccccc}
		\toprule
		Malicious Ratio & Method & LF  & Gaussian  & Trim  & Krum  & Min-Max  & Scaling  & Sybil  \\
		\midrule
		
		\multirow{2}{*}{55\%} & Median  &  0.054 &	0.053 & 0.109 & 0.837 & 0.236 & 0.831 & 0.831 \\
		& AdaBFL    & \cellcolor{lightblue}0.066 & \cellcolor{lightblue}0.048& \cellcolor{lightblue}0.093 & \cellcolor{lightblue}0.772 & \cellcolor{lightblue}0.051 & \cellcolor{lightblue}0.831 & \cellcolor{lightblue}0.831 \\
		\midrule
		\multirow{2}{*}{60\%} & Median   &  0.093 & 0.056 & 0.115 & 0.856 & 0.288 &	0.831 & 0.831 \\
		& AdaBFL    & \cellcolor{lightblue}0.071 & \cellcolor{lightblue}0.059 & \cellcolor{lightblue}0.082 & \cellcolor{lightblue}0.675 & \cellcolor{lightblue}0.054 & \cellcolor{lightblue}0.831  & \cellcolor{lightblue}0.831 \\
		\midrule
		\multirow{2}{*}{65\%} & Median  &  0.100 & 0.116 & 0.117 & 0.857 & 0.597 & 0.831 & 0.831 \\
		& AdaBFL    & \cellcolor{lightblue}0.084 & \cellcolor{lightblue}0.062 & \cellcolor{lightblue}0.114 & \cellcolor{lightblue}0.676 & \cellcolor{lightblue}0.060 & \cellcolor{lightblue}0.831  & \cellcolor{lightblue}0.831 \\
		\bottomrule
	\end{tabular}
	
	\centering
	\vspace{1mm}
	\textbf{(d) Cifar-10 dataset}

	\begin{tabular}{lcccccccc}
		\toprule
		Malicious Ratio & Method & LF  & Gaussian  & Trim  & Krum  & Min-Max  & Scaling  & Sybil \\
		\midrule
		
		\multirow{2}{*}{55\%} & Median  &   0.284 & 0.677 & 0.900 & 0.900 & 0.900 & 0.900 & 0.900 \\
		& AdaBFL    & \cellcolor{lightblue}0.263 & \cellcolor{lightblue}0.294 & \cellcolor{lightblue}0.400 & \cellcolor{lightblue}0.900 & \cellcolor{lightblue}0.298 & \cellcolor{lightblue}0.900 & \cellcolor{lightblue}0.900 \\
		\midrule
		\multirow{2}{*}{60\%} & Median  &  0.287 & 0.835 & 	0.900 & 0.900 & 0.900 & 0.900 & 0.900 \\
		& AdaBFL    & \cellcolor{lightblue}0.272 & \cellcolor{lightblue}0.306 & \cellcolor{lightblue}0.661 & \cellcolor{lightblue}0.900 & \cellcolor{lightblue}0.308 & \cellcolor{lightblue}0.900  & \cellcolor{lightblue}0.900 \\
		\midrule
		\multirow{2}{*}{65\%} & Median   &  0.296 & 0.900 & 0.900 & 0.900 & 0.900 & 0.900 & 0.900 \\
		& AdaBL    & \cellcolor{lightblue}0.281 & \cellcolor{lightblue}0.319 & \cellcolor{lightblue}0.826 & \cellcolor{lightblue}0.900 & \cellcolor{lightblue}0.324 & \cellcolor{lightblue}0.900  & \cellcolor{lightblue}0.900 \\
		\bottomrule
	\end{tabular}
	
	\centering
	\vspace{1mm}
	\textbf{(e) Petimage dataset}

	\begin{tabular}{lcccccccc}
		\toprule
		Malicious Ratio & Method & LF  & Gaussian  & Trim  & Krum  & Min-Max  & Scaling  & Sybil \\
		\midrule
		
		\multirow{2}{*}{55\%} & Median   &  0.500	 & 0.200 & 	0.347 & 0.500 & 	0.500 &	0.500 & 0.500 \\
		& AdaBFL    & \cellcolor{lightblue}0.371 & \cellcolor{lightblue}0.194 & \cellcolor{lightblue}0.201 & \cellcolor{lightblue}0.500 & \cellcolor{lightblue}0.198 & \cellcolor{lightblue}0.193 & \cellcolor{lightblue}0.500 \\
		\midrule
		\multirow{2}{*}{60\%} & Median  &  0.500 & 	0.270 & 0.500 &	0.500 & 0.500 & 0.500 & 0.500 \\
		& AdaBFL    & \cellcolor{lightblue}0.500 & \cellcolor{lightblue}0.206 & \cellcolor{lightblue}0.196 & \cellcolor{lightblue}0.500 & \cellcolor{lightblue}0.203 & \cellcolor{lightblue}0.500 & \cellcolor{lightblue}0.500 \\
		\midrule
		\multirow{2}{*}{65\%} & Median &  0.500 & 0.350 & 0.500 & 0.500 & 0.500 &	0.500 & 0.500\\
		& AdaBFL    & \cellcolor{lightblue}0.500 & \cellcolor{lightblue}0.236 & \cellcolor{lightblue}0.221 & \cellcolor{lightblue}0.500 & \cellcolor{lightblue}0.213 & \cellcolor{lightblue}0.500  & \cellcolor{lightblue}0.500 \\
		\bottomrule
	\end{tabular}

	\centering
	\vspace{1mm}
	\textbf{(f) Shakespeare dataset}

	\begin{tabular}{lcccccccc}
		\toprule
		Malicious Ratio & Method & LF & Gaussian  & Trim  & Krum  & Min-Max & Scaling  & Sybil \\
		\midrule
		
		\multirow{2}{*}{55\%} & Median  & 10.904 & 	4.007 & - & 3.883 & 5.714 & - & 	- \\
		& AdaBFL    & \cellcolor{lightblue}9.091 & \cellcolor{lightblue}3.856 & \cellcolor{lightblue}6.508 & \cellcolor{lightblue}4.094 & \cellcolor{lightblue}3.866 & \cellcolor{lightblue}-  & \cellcolor{lightblue}- \\
		\midrule
		\multirow{2}{*}{60\%} & Median   &   16.72 & 	4.129 & - & 6.500 &	5.475 & - & - \\
		& AdaBFL    & \cellcolor{lightblue}14.22 & \cellcolor{lightblue}3.987 & \cellcolor{lightblue}8.64 & \cellcolor{lightblue}5.026 & \cellcolor{lightblue}3.878 & \cellcolor{lightblue}-  & \cellcolor{lightblue}- \\
		\midrule
		\multirow{2}{*}{65\%} & Median  & 18.33 & 	4.133 & - & 6.640 & 8.180 & - & 	- \\
		& AdaBFL    & \cellcolor{lightblue}16.27 & \cellcolor{lightblue}3.916 & \cellcolor{lightblue}12.47 & \cellcolor{lightblue}5.436 & \cellcolor{lightblue}6.564 & \cellcolor{lightblue}-  & \cellcolor{lightblue}- \\
		\bottomrule
	\end{tabular}
\end{table*}

\end{document}